%% file: main.tex
\PassOptionsToPackage{numbers, compress}{natbib}
\documentclass{article}

\usepackage[preprint]{neurips_2026}

\usepackage[utf8]{inputenc}
\usepackage[T1]{fontenc}
\usepackage{algorithm}
\usepackage{algorithmic}
\usepackage{hyperref}
\usepackage{url}
\usepackage{booktabs}
\usepackage{amsfonts}
\usepackage{amsmath}
\usepackage{amssymb}
\usepackage{nicefrac}
\usepackage{microtype}
\usepackage{xcolor}
\usepackage{graphicx}
\usepackage{multirow}
\usepackage{tikz}
\usetikzlibrary{arrows.meta}
\usepackage{docmute}

\definecolor{p1fill}{HTML}{FFF3E0}
\definecolor{p1stroke}{HTML}{FF9800}
\definecolor{p2fill}{HTML}{E0F2F1}
\definecolor{p2stroke}{HTML}{009688}
\definecolor{latentfill}{HTML}{BBDEFB}
\definecolor{latentstroke}{HTML}{1976D2}
\definecolor{actorfill}{HTML}{E1BEE7}
\definecolor{actorstroke}{HTML}{8E24AA}
\definecolor{groupfill}{HTML}{ECEFF1}
\definecolor{groupstroke}{HTML}{90A4AE}
\definecolor{startfill}{HTML}{E8F5E9}
\definecolor{startstroke}{HTML}{2E7D32}
\definecolor{decodedfill}{HTML}{F5F5F5}

\newcommand{\R}{\mathbb{R}}

\DeclareMathOperator*{\argmin}{arg\,min}
\DeclareMathOperator{\KL}{KL}
\DeclareMathOperator{\sg}{sg}
\newcommand{\given}{\,|\,}
\newcommand{\reach}{\nu}          
\newcommand{\lreach}{\hat{\nu}}   

\newtheorem{definition}{Definition}
\newtheorem{proof}{Proof}
\newtheorem{theorem}{Theorem}
\newtheorem{corollary}{Corollary}
\newtheorem{lemma}{Lemma}
\newtheorem{remark}{Remark}

\title{NashDreamer: Model-Based Reinforcement Learning for Zero-Sum Imperfect-Information Games}

\author{%
  \hfill Tom\'{a}\v{s} Hole\v{c}ek \hfill Viliam Lis\'{y} \hfill ~\\
  Artificial Intelligence Center, Department of Computer Science\\
  Faculty of Electrical Engineering, Czech Technical University in Prague\\
  \texttt{holecto6@fel.cvut.cz, viliam.lisy@agents.fel.cvut.cz}\\
}

\begin{document}

\maketitle

\begin{abstract}
Model-based reinforcement learning (MBRL) has achieved remarkable results in single-agent domains, yet its extension to competitive imperfect information games (IIGs) remains underexplored. In multi-agent settings, opponent-induced non-stationarity complicates the learning process, and decentralized model learning faces severe identifiability barriers, which we argue make centralized model learning a mathematical necessity. Building on this analysis, we propose NashDreamer, a principled MBRL framework for two-player zero-sum IIGs. It introduces a centralized Multi-Agent Recurrent State-Space Model (MARSSM) that decouples environment dynamics from the effect of players' strategies on their individual observations. 
NashDreamer is designed to use arbitrary policy gradient algorithms and inherits their convergence guarantees towards Nash equilibria under an idealized model.
Empirical evaluations across five benchmark games demonstrate that NashDreamer substantially improves sample efficiency over model-free baselines early in the training.
Finally, we theoretically analyze the architecture's optimization landscape, identifying the vulnerability of the Dreamer family of algorithms to posterior collapse in stochastic environments. We leave it as an open challenge.
\end{abstract}

\section{Introduction}
\label{sec:intro}

Model-based reinforcement learning (MBRL) has emerged as a powerful paradigm for sample-efficient learning. MuZero~\citep{muzero} and its extensions demonstrate that learned latent models can serve as effective planning abstractions, while the Dreamer series~\citep{dreamerv1,dreamerv2,dreamerv3, dreamerv4} shows that world models can generate imagination trajectories that reduce sample requirements. These approaches have matched or surpassed the state of the art across diverse single-agent domains, including robotics~\citep{daydreamer}, autonomous driving~\citep{driving1,driving2}, and nuclear fusion control~\citep{tokamak}.

Complex imperfect information adversarial domains currently force a trade-off between fidelity and tractability. Automated cyber defense is studied on emulators made fast by simplifying the attacker, the defender, or the network itself~\citep{nasimemu}. As another example, dogfight simulators, such as the one in~\citet{dreamerzerosum}, hand each aircraft a complete frame rather than only what lies within its own sensor range. Sample-efficient generative MBRL promises a path towards recovering speed from a small budget of interactions with the unsimplified environment. To lay the foundation, we investigate the Dreamer-style paradigm for multi-player imperfect information games (IIGs), leveraging generative world models to synthesize large datasets for sample-based policy optimization.

Adapting MBRL to IIGs, however, exposes fundamental identifiability barriers. In a standard decentralized setting, a learning agent cannot distinguish between the objective stochastic dynamics of the environment and the shifting policy of its opponent~\citep{amirchang,oliehoek}. We argue that this ``policy-physics entanglement'' dictates that Centralized Training with Decentralized Execution (CTDE) is not merely an architectural convenience, but a mathematical necessity to learn a sound, stationary world model.

We further provide the following contributions: (1) We propose \emph{NashDreamer}, a principled CTDE framework for two-player zero-sum IIGs. It introduces a centralized Multi-Agent Recurrent State-Space Model (MARSSM), which can be combined with game-theoretically sound policy-gradient methods, to facilitate convergence towards the Nash equilibrium. (2) We show that DreamerV3 ``KL-balancing'' can cause posterior collapse, highlighting the challenge of mitigating latent space non-stationarity while maintaining a theoretically sound framework. (3) We provide thorough empirical evaluations showing that NashDreamer improves sample efficiency compared to model-free baselines across various domains.

NashDreamer is not meant to replace well-tuned model-free policy-gradient methods, which reach excellent policies given abundant simulator access~\citep{pgconvergence}. It is a \emph{sample-efficiency layer} agnostic to the optimizer running inside imagination. The appropriate comparison is thus a policy-gradient method $X$ against NashDreamer[$X$] at an equal budget of \emph{real} environment interactions. This is the regime of our motivating applications, such as automated cyber defense or expensive simulations, where a real interaction, rather than a gradient step, is the dominant cost.

\input{main_body}

\begin{ack}
We thank the anonymous reviewers, whose thorough and constructive feedback prompted several additional experiments and analyses that substantially improved both the empirical evaluation and the presentation of this work.

This research is supported by the Czech Science Foundation (GA25-18353S). Computational resources were provided by the e-INFRA CZ project (ID:90254), supported by the Ministry of Education, Youth and Sports of the Czech Republic. 
\end{ack}

\bibliographystyle{plainnat}
\bibliography{references}


\input{appendices}


\newpage

\end{document}

%% file: main_body.tex
\section{Background}
\label{sec:background}

\label{sec:posg}

We formalize our target domains as two-player zero-sum (2p0s) simultaneous-move games.
We use a custom variant of Partially Observable Stochastic Games (POSGs) inspired by~\citep{fosg}.
We denote $\mathcal{N} = \{1, 2\}$ the player set with $c$ for the chance player; $\mathcal{S}$ the state space with initial state $s_0$; $\mathcal{A} = \mathcal{A}_1 \times \mathcal{A}_2$ the joint action space, whose elements we write as $a = (a_1, a_2)$; $\mathbb{O} = \mathbb{O}_1 \times \mathbb{O}_2$ is the joint observation space with per-player observations. $T: \mathcal{S} \times \mathcal{A} \to \Delta\mathcal{S}$ is the (partial) transition function; $O: \mathcal{S} \to \mathbb{O}$ is the (deterministic) observation function\footnote{A deterministic observation function is no restriction:\citet[fn.~3]{fosg} note that the variant with $O: \mathcal{S} \to \Delta\mathbb{O}$ is equivalent to the deterministic one, since stochastic observations can be emulated by adding a hidden variable to the world state. Recording the drawn observation in that variable additionally makes $O$ a function of the state alone, which is the form we use.}; $R: \mathcal{S} \times \mathcal{A} \to \R$ is the zero-sum reward function ($R_1 = -R_2$); $P: \mathcal{S} \to 2^{\mathcal{N}}$ identifies the acting players at each state; and $L_{i}: \mathcal{S} \to 2^{\mathcal{A_{{\it i}}}}$ specifies legal actions for player ${\it i}$. States where $P(s) = \emptyset$ and the transition function is defined are \emph{chance nodes}\footnote{Chance is the only source of stochasticity in our environments, so the transition function is deterministic everywhere except at chance nodes. The explicit chance player is a formal convenience, and the chance node can equivalently be absorbed by turning the previous transition stochastic. This is how the world model treats chance during learning.}; states with undefined transitions are \emph{terminal} and other nodes are called decision nodes, where we assume $P(s) = \{1,2\}$.
Turn-based games, such as Leduc Poker, are embedded by assigning a \texttt{noop} action to the inactive player (Appendix~\ref{app:environments}).

A \emph{history} $h = (s_0, a^0, \dots, s_t)$ is the full sequence of states and joint actions from the initial state, the ground-truth object no player observes directly; $s(h)$ is its last state, through which we lift the state-indexed functions, writing $T(h,a) := T(s(h),a)$ and likewise for $O$ and $L_i$. Player $i$ observes only their \emph{action-observation history} (AOH) $\tau_i(h) = (o_i^0, a_i^0, \dots, o_i^t)$, and we assume perfect recall. The \emph{information set} (infoset) of player $i$ at $h$ is $I_i(h) = \{h' : \tau_i(h') = \tau_i(h)\}$, with $\mathcal{I}_i$ the set of all infosets of player $i$. A policy $\pi_i: \mathcal{I}_i \to \Delta(\mathcal{A}_i)$ conditions on $\tau_i$ alone and never on $h$, legal-action masks entering only as an implementation-level restriction of its support. We use $\pi$ for optimized policies and reserve $\mu$ for the joint \emph{sampling} policy that collects data. 
We assume \emph{collective observability}: the joint action-observation history identifies the true underlying history, i.e., $I_1(h) \cap I_2(h) = \{h\}$.\footnote{We note that the players identify $h$ \emph{jointly}, not individually, so this is not a perfect-information assumption.} It is a technical convenience rather than a restriction, since a chance outcome eventually observed by some player can be deferred to the point where it is first observed~\citep{thompson}. It is also the assumption that makes public belief states well defined~\citep{rebel}, and all our benchmarks satisfy it.

A \emph{Nash equilibrium} (NE)~\citep{nasheq} is a joint policy $\pi^*$ where no player can improve their utility by unilateral deviation. We measure distance to NE using \emph{NashConv}~\citep{openspiel}: $\text{NashConv}(\pi) = u(\text{BR}(\pi_2), \pi_2) - u(\pi_1, \text{BR}(\pi_1))$, where $\text{BR}(\pi_i)$ is the best response of the player $i$'s opponent.

\subsection{DreamerV3}
\label{sec:dreamer}

DreamerV3~\citep{dreamerv3} learns a latent world model, the Recurrent State-Space Model (RSSM), for acting in POMDPs. The model state $m_t = [\hat{h}_t, z_t]$ concatenates a deterministic recurrent state $\hat{h}_t$ produced by a sequence model $\hat{h}_t = \text{seq}(m_{t-1}, a_{t-1})$ and a stochastic categorical state $z_t$ sampled from either the posterior encoder $q^{\theta}_t = \text{enc}(\hat{h}_t, o_t)$ (during training with real observations) or the prior dynamics predictor $p^{\theta}_t = \text{dyn}(\hat{h}_t)$ (during imagination). The RSSM is trained end-to-end by optimizing:
\vspace{-1mm}
\begin{equation}
    \mathcal{L}^{\text{model}} = \sum_t^T \beta^{\text{enc}} \mathcal{L}^{\text{enc}}_t + \beta^{\text{dyn}} \mathcal{L}^{\text{dyn}}_t + \beta^{\text{pred}} \mathcal{L}^{\text{pred}}_t.
    \label{eq:dreamer_loss}
\vspace{-2mm}
\end{equation}
The prediction loss $\mathcal{L}^{\text{pred}}_t$ trains decoders to reconstruct rewards, continuation flags, and observations from the model state. The encoder and dynamics losses $\mathcal{L}^{\text{enc}}_t = \text{KL}(q^{\theta}_t \parallel \text{sg}(p^{\theta}_t))$ and $\mathcal{L}^{\text{dyn}}_t = \text{KL}(\text{sg}(q^{\theta}_t) \parallel p^{\theta}_t)$ are both KL divergences between the posterior and prior ($\text{sg}$ denotes stop-gradient). These asymmetric losses perform ``KL balancing'': $\mathcal{L}^{\text{dyn}}_t$ trains the dynamics to match the encoder, while $\mathcal{L}^{\text{enc}}_t$ serves as an encoder regularizer to remain predictable by the dynamics.

From trained model states, Dreamer generates \emph{imagination trajectories} using the dynamics predictor, effectively amplifying the training data. For policy optimization, the actor uses REINFORCE with an entropy bonus, and the critic uses TD($\lambda$).

\subsection{Game-Theoretic Policy Optimization}
\label{sec:rnad}

Policy-gradient methods developed for single-agent problems are often not sound in imperfect-information games: the online-learning dynamics they descend from are known to cycle rather than converge to an NE, in matrix games~\citep{forelcycle} and in sequential IIGs~\citep{rnadtheory} alike. Regularizing toward a reference policy has proven the successful remedy, recovering last-iterate convergence.

Regularized Nash Dynamics (RNaD)~\citep{rnadstratego} regularizes in \emph{reward} space, subtracting a KL penalty against a reference policy $\pi^{\text{reg}}$ from the payoff, which preserves the zero-sum structure. Periodically resetting $\pi^{\text{reg}} \leftarrow \pi$ then solves a sequence of regularized games whose solutions converge to an NE of the original one. Magnetic Mirror Descent (MMD)~\citep{mmd} instead regularizes in \emph{policy} space, pairing a KL term toward a fixed \emph{magnet} policy with a trust-region term on consecutive iterates. RNaD provably converges to an NE, whereas MMD has convergence guarantees only to a quantal response equilibrium, so its asymptotic exploitability is bounded away from zero.  Appendix~\ref{app:optimizers} details both algorithms, including the off-policy simultaneous-move variant of RNaD~\citep{lamir} that we use.

\section{World Models in Games: The Necessity of Centralized Training}
\label{sec:wm_games}
\label{sec:ctde}

Learning a world model without observing the opponent's actions faces the following fundamental barriers preventing the use of contemporary decentralized approaches in adversarial environments (see Appendix~\ref{app:contemporary}).

\textbf{Action identifiability.} For a given state $s$ and ego action $a_1$, different opponent actions $a_2 \neq a_2'$ must produce distinct next state transition distributions: $T(s, (a_1, a_2)) \neq T(s, (a_1, a_2'))$. When violated, \emph{action aliasing} makes the data insufficient to disambiguate the causes of state changes~\citep{amirchang,ljung}. While this aliasing is benign for a purely forward-predictive model (as the resulting state is identical), it poses a fundamental identifiability barrier for any decentralized model attempting to reverse-engineer the objective environment dynamics independent of the opponent's policy.

\textbf{Policy-physics entanglement.} Without observing $a_2$, a learner cannot distinguish true environment dynamics $T$ from modified dynamics $T'$ where the opponent's strategy artificially mimics the effects of $T$~\citep{amirchang}. A decentralized model is therefore a \emph{joint-system model} rather than an \emph{environment model}. It captures how the world works \emph{against that specific opponent}, failing to capture the objective environment dynamics~\citep{oliehoek}.

\textbf{Adversarial exploitation.} In strictly competitive environments, a strong opponent will actively exploit this identifiability gap, deliberately selecting actions that maximize the mismatch between true physical reality and the ego-agent's corrupted internal representation.

Centralized training bypasses these barriers entirely. By conditioning the world model on the joint action $(a_1, a_2)$ during training, the objective environment physics is perfectly decoupled from the opponent's shifting strategy. We intentionally structure the world model so that it can be used with only the ego-agent's partial observations at test time. 

Centralization also amplifies the usual benefits of a world model. Disentangling the stochasticity of the opponent's policy from the environment dynamics reduces the number of real samples needed to learn the model: in Rock-Paper-Scissors, a tabular world model would require only 9 transitions, after which policy training can proceed entirely in imagination. A world model further decouples data collection from policy optimization, so agents may collect with highly exploratory policies while training purely on-policy inside imagination, avoiding high-variance off-policy corrections in IIGs.

As a small scale demonstration of the decentralized model instability, we performed an experiment in Perturbed Rock-Paper-Scissors, where rewards for all outcomes in which at least one player played "Paper" are multiplied by two. We compare a centralized model observing both players' actions with a decentralized model, observing only the ego player's action from each seat. While the centralized model learns the payoffs perfectly, the error of the decentralized model \emph{increases} over training compared to random initialization, since the opponent action remains a latent variable (Appendix~\ref{app:decentralized}).
\section{NashDreamer}
\label{sec:method}
\label{sec:marssm}

As established in Section~\ref{sec:ctde}, applying decentralized world models to multi-agent environments inevitably results in policy-physics entanglement. To resolve this, we introduce the centralized Multi-Agent Recurrent State-Space Model (MARSSM) to learn the objective environment dynamics, while a separate recurrent network handles information aggregation for each player in a decentralized manner. Crucially, to preserve the causal chain of information aggregation, we unroll complete trajectories starting at the initial state $t=0$ for the world model training. We detail possible mechanisms allowing training on shorter, mid-trajectory chunks and their respective pitfalls in Appendix~\ref{app:replay}.

\begin{definition}
A \textbf{centralized MARSSM} is a tuple $(\text{seq}, \text{enc}, \text{dyn}, \text{iset})$, where:
$\hat{h}_t = \rm{seq}(m_{t-1}, (a_1^{t-1}, a_2^{t-1}))$ is the sequence model, receiving the \emph{joint} action; $q^{\theta}_t = \rm{enc}(\hat{h}_t, (o_1^t, o_2^t))$ is the posterior encoder, inferring the stochastic state $z_t \sim q^{\theta}_t$ from the \emph{joint} current observations; $p^{\theta}_t = \rm{dyn}(\hat{h}_t)$ is the prior dynamics predictor, predicting the stochastic state $z_t \sim p^{\theta}_t$ without access to the current observations; $\hat{I}_i^t = \rm{iset}(\hat{I}_i^{t-1}, a_i^{t-1}, o_i^t)$ is the infoset model for player $i$, recurrently building the latent infoset embedding from their own AOH $\tau_i$ (Section~\ref{sec:posg}).
\end{definition}

The central model state $m_t = [\hat{h}_t, z_t]$ aggregates the private observations of both players and handles the environment dynamics. The infoset model embeds each player's AOH $\tau_i$ into a fixed-size vector $\hat{I}_i^t$, serving as our latent infoset representation.

\subsection{World Model Objective}
\label{sec:worldmodel_training}

Our predictor loss extends DreamerV3 with legal action prediction and per-player observation reconstruction. Given the centralized model state $m_t = [\hat{h}_t, z_t]$, the predictors decode the reward $R_1^t$ for Player~1, a termination flag $d_t$, and, for each player $i$, the legal action mask $\ell_i^t$ and observation $o_i^t$:
\begin{equation}
    \mathcal{L}^{\text{mpred}}_t = \mathcal{L}^{\text{rew}}_t(R_1^t \mid m_t) + \mathcal{L}^{\text{BCE}}_t(d_t \mid m_t) + \sum_{i \in \{1,2\}} \left[ \mathcal{L}^{\text{BCE}}_t(\ell_i^t \mid m_t) + \mathcal{L}^{\text{MSE}}_t(o_i^t \mid m_t) \right],
    \label{eq:pred_loss}
\end{equation}
where $\mathcal{L}^{\text{rew}}$ follows the DreamerV3 categorical parameterization, BCE denotes binary cross-entropy (per element for the legal action mask), and MSE is mean squared error. 

The latent infoset embeddings $\hat{I}_i^t$ are trained by two further losses. The \emph{unifier} loss requires the joint embeddings $(\hat{I}_1^t, \hat{I}_2^t)$ to reconstruct the centralized model state $m_t$, which is well defined due to collective observability (Section~\ref{sec:posg}); the \emph{individual} loss requires each embedding to reconstruct its own player's observation $o_i^t$ and previous one-hot action $a_i^{t-1}$, so that it captures local context:
\begin{align}
    \mathcal{L}^{\text{iset}}_t &= \underbrace{\mathcal{L}^{\text{MSE}}_t(\sg(\hat{h}_t) \mid \hat{I}_1^t, \hat{I}_2^t) + \mathcal{L}^{\text{CE}}_t(\sg(z_t) \mid \hat{I}_1^t, \hat{I}_2^t)}_{\text{Unifier Loss}} \quad + \sum_{i \in \{1,2\}} \underbrace{\left[ \mathcal{L}^{\text{MSE}}_t(o_i^t \mid \hat{I}_i^t) + \mathcal{L}^{\text{CE}}_t(a_i^{t-1} \mid \hat{I}_i^t) \right]}_{\text{Individual Loss}}
    \label{eq:iset_loss}
\end{align}
where CE is categorical cross-entropy.

The total model loss combines the standard DreamerV3 KL-balancing losses $\mathcal{L}^{\text{enc}}_t$ and $\mathcal{L}^{\text{dyn}}_t$ (Section~\ref{sec:dreamer}) with our prediction and infoset losses:
\begin{equation}
    \mathcal{L}^{\text{model}} = \sum_t^T \beta^{\text{enc}} \mathcal{L}^{\text{enc}}_t + \beta^{\text{dyn}} \mathcal{L}^{\text{dyn}}_t + \beta^{\text{pred}} \mathcal{L}^{\text{mpred}}_t + \beta^{\text{iset}} \mathcal{L}^{\text{iset}}_t,
    \label{eq:total_loss}
\end{equation}
where we set $\beta^{\text{dyn}} = \beta^{\text{pred}} = \beta^{\text{iset}} = 1$ and $\beta^{\text{enc}} = 0.1$, following the DreamerV3 convention, so that the posterior is shaped primarily by the prediction targets and only mildly pulled toward the prior. We deliberately do not tune these coefficients per environment, with the exception of Leduc, where we showcase failure modes of the KL balancing loss in stochastic environments (Appendix \ref{app:leducerror}). Figure~\ref{fig:ndwm} visualizes two steps of world-model training.

\subsection{Imagination}
\label{sec:imagination}

Imagination supplies the synthetic data on which the actor-critic is trained. As in DreamerV3, we start one imagined trajectory for every non-chance node of trajectories collected from the real environment, terminal nodes included. Imperfect Information Goofspiel 5 (all our Goofspiel variants are imperfect information; for brevity we omit writing it out explicitly from now on), for instance, yields five imagined trajectories per real one. The imagined volume per gradient step therefore scales with the length of the game rather than with the batch size alone. We deviate from DreamerV3 in \emph{where} a rollout begins: instead of branching from the encountered node, every rollout restarts from the root state $t=0$ and runs for the maximum trajectory length of the underlying game (Appendix~\ref{app:replay}).

Starting from a model state $m_0$ obtained during world model training, each imagined step reconstructs the per-player legal actions, the reward and the termination flag from $m_t$, samples each actor's action from its policy conditioned on that player's latent infoset and legal actions\footnote{We explicitly mask out the policy output with the legal actions, rather than forcing the network to learn which actions are illegal, which is inefficient.}, advances the sequence model on the joint action, samples the next stochastic state from the prior to form $m_{t+1}$ and updates each player's latent infoset from the observations reconstructed at $m_{t+1}$ (see Appendix~\ref{app:imagination_step}).

That last reconstruction is a key difference from standard Dreamer, where the decoder is merely a training signal. Here the infoset model consumes observations, so the decoder becomes an essential generative component of the imagination pipeline. We acknowledge the risks introduced by explicit reconstruction of high-dimensional observations and discuss their mitigation in Appendix \ref{app:compounding_errors}.

Following the CTDE paradigm, the \emph{decentralized} actor for player $i$ is conditioned on $\hat{I}_i^t$ alone, ensuring it cannot access the opponent's private observations. The \emph{centralized} critic, used exclusively during training, operates on the joint representation $(\hat{I}_1^t, \hat{I}_2^t)$, grounding its value estimates in the full game history. All MARSSM parameters are frozen during the imagination phase. Figure~\ref{fig:ndimag} visualizes two imagined steps; Appendix~\ref{app:imagination_step} states the step in full.

\begin{figure}[t]
    \centering
    \resizebox{0.95\linewidth}{!}{\input{diagrams/nash_dreamer_wm}}
    \caption{Two steps of NashDreamer world-model training on a real trajectory. The centralized RSSM is unrolled on the joint action and joint observation, the infoset model separately for each player. All stochastic states are posterior. The KL losses are omitted; $a_d$ is the initial dummy action.}
    \label{fig:ndwm}
\end{figure}

\begin{figure}[t]
    \centering
    \resizebox{0.8\linewidth}{!}{\input{diagrams/nash_dreamer_imagination}}
    \caption{Two imagined steps. Imagination is grounded in a posterior model state taken from world-model training and proceeds with prior stochastic states. The observation decoder is required during imagination, since the infoset model consumes observations.}
    \label{fig:ndimag}
\end{figure}

\subsection{Actor-Critic}
\label{sec:actor_critic}

We choose not to use the DreamerV3 standard REINFORCE, as its fragility towards hyperparameter tuning is unsuitable for MBRL (Appendix~\ref{app:reinforce_vs_rnad}). Among the game-theoretically sound alternatives, such as MMD~\citep{mmd}, we suggest RNaD because: (i) its last-iterate convergence guarantee is what our model-convergence result composes with (Corollary~\ref{cor:ne_convergence}); (ii) it reportedly needs less domain-specific tuning, which matters here because the game it optimizes against is a learned, non-stationary object; and (iii) its off-policy simultaneous-move variant~\citep{lamir} is a natural fit for the replay/exploratory model learning. The choice is nevertheless modular as we demonstrate by also testing with MMD.
We train the actor-critic networks concurrently on both real and imagined trajectories.

While the MARSSM imagination pipeline permits purely on-policy optimization, we deliberately employ the off-policy formulation in our empirical setup to ensure exact parity with the (necessarily) off-policy model-free RNaD with replay buffer and to permit exploratory policies for model learning.

\subsection{Model Convergence}
\label{sec:theory_ideal}

To understand how NashDreamer captures environment dynamics, we analyzed the global minimum of the world model objective (all formal definitions and proofs are deferred to Appendix~\ref{app:theory}). We define an \emph{infoset-isomorphic} model ($\mathcal{M}^*$) as one that establishes a one-to-one mapping between latent states and true histories, and clusters latent states into latent infosets if and only if corresponding true histories belong to the same true infoset. 

Our first result is that this structure-preserving model is strictly better, in expected loss, than models with two structural deviations: a ``branch duplication'', which maps one history to several hallucinated latent branches with a distribution over them, and a model that violates the information partitioning of the game\footnote{Either by aliasing distinct true infosets or by subdividing a single true infoset into multiple latent ones.}. The same result fixes the value of the optimum, $\mathbb{E}_\mu[\mathcal{L}^{\rm{model}}(\mathcal{M}^*)] = (\beta^{\text{enc}} + \beta^{\text{dyn}})\,\mathbb{E}_{h,a \sim \mu}[H(T(h,a))]$, so the optimal loss of the ``perfect'' model scales linearly with environment stochasticity (Theorem~\ref{thm:isomorphic_optimal}). Because $\mathcal{M}^*$ captures the POSG structure, which unrolls into an EFG~\citep{fosg}, and because branch duplication is indistinguishable from it from the actor-critic's view, RNaD's convergence result composes with it: an infoset-isomorphic model held fixed, together with exactly realized RNaD dynamics, yields convergence to a Nash equilibrium of the underlying game (Corollary~\ref{cor:ne_convergence}).
We note that this is an idealized convergence assumption, since the practical algorithm approximates every one of its assumptions: neural function approximation, sampled trajectories, off-policy corrections, and a \emph{learned} model updated concurrently with the policy.

However, our second result exposes a profound vulnerability of the two-way KL-balancing objective to posterior collapse in a stochastic environment. Specifically, this can happen in environments with high stochasticity and small magnitude changes for the predictor, turning disambiguation of the chance outcomes into a \emph{globally} suboptimal solution. We note that this theorem also holds for a single-agent environment and the standard DreamerV3 architecture.

\begin{theorem}[Vulnerability to Posterior Collapse]
\label{thm:aliasedcounterexample}
Let $\mathcal{M}^{\text{collapse}}$ be a model whose posterior equals its prior, i.e., $\text{KL}(\text{enc}(\hat{h}_t, o_t) \parallel \text{dyn}(\hat{h}_t)) = 0$ for all reachable $\hat{h}_t$, and $o_t$ where $p(o_t | \hat{h}_t) > 0$. There exists an environment, sampling policy $\mu$, and loss coefficients, such that $\mathbb{E}_\mu[\mathcal{L}(\mathcal{M}^{\text{collapse}})] < \mathbb{E}_\mu[\mathcal{L}(\mathcal{M}^*)].$
\end{theorem}

The dependence of the isomorphism loss bound on the (typically) non-stationary $\mu$ renders static coefficient tuning unsound. While a natural mitigation technique is to employ a  ``one-way'' KL loss (applying a stop-gradient to the posterior),  removing the gradient pull towards $\mathcal{M}^{collapse}$, unfortunately, removing this posterior regularization allows for representation drift (Section~\ref{sec:dreamer}). Consequently, Dreamer-style MBRL faces a trade-off: the regularizer serving to mitigate representation non-stationarity and unpredictability can drive the optimization towards a degenerate solution.

Finally, this artifact does not conflict with DreamerV3's empirical success on stochastic domains such as Minecraft, Atari and BSuite. What makes collapse harmful is not stochasticity as such, but chance that is discrete, many-outcome, and strategically decisive. Minecraft's procedural generation is highly stochastic, yet subsequent frames are largely predictable from those already seen, the initial-only randomness is closer to epistemic uncertainty than aleatoric, and pixel-level error is strategically benign where collapse does occur. The randomness injected into Atari and BSuite is genuinely aleatoric, but either few-outcome (sticky actions) and thus unlikely to form the degenerate collapsed optimum, or noise-like. Chance in games is the opposite on every count, which is why the same artifact becomes damaging, as we demonstrate in Leduc.


\section{Empirical Evaluation}
\label{sec:experiments}

To avoid imagination in a highly inaccurate model, a world model warm-up precedes imagination: $1000$ gradient steps for Goofspiel 5 and Leduc Poker, and $2000$ for the larger environments: stochastic Goofspiel 13 (only written as Goofspiel 13 from now on), Battleship $5\times5$ and Phantom Tic-Tac-Toe. Appendix~\ref{app:warmup} ablates this choice.  
 Goofspiel and Battleship are natively \emph{simultaneous-move}; Leduc Poker and Phantom Tic-Tac-Toe are turn-based, embedded by the dummy-action construction of Section~\ref{sec:posg}. The world model has no turn-based mode. We refer the reader to Appendix \ref{app:environments} for detailed environment rules, \ref{app:hyperparameters} for extended hyperparameters, and  \ref{app:pcm_results} for comparison with DreamerV3 in a single-agent domain. We parametrized the latent categorical space as $1 \times 3$ for deterministic games, $2 \times 6$ for Leduc and $1 \times 13$ for Goofspiel 13.
 All the experiments were carried out with 10 distinct training seeds. Our code is publicly available at GitHub\footnote{\url{https://github.com/aicenter/nash_dreamer}}.

\textbf{Evaluation protocol.} All reported NashConv values are computed \emph{exactly} against the true game by best-response dynamic programming over the full game tree. For Phantom Tic-Tac-Toe, we use the exp-a-spiel library \citep{pgconvergence}. Appendix~\ref{app:nashconv} details the handling of chance nodes and simultaneous moves.

\textbf{Baselines and positioning.} We compare against model-free RNaD (with and without replay) and against MMD~\citep{mmd}, the strongest-performing model-free policy-gradient method in the evaluation of~\citet{pgconvergence}. MMD serves both as a baseline and, run inside imagination as NashDreamer[MMD], as the demonstration that the policy optimizer is interchangeable (Section~\ref{sec:actor_critic}). Its hyperparameters are the generic setting of~\citet{pgconvergence} and are held fixed across all domains, detailed in Appendix~\ref{app:mmd_baseline}. We do not compare directly against published multi-agent Dreamer systems (Section \ref{sec:related}), because each is structurally inapplicable to 2p0s IIGs: cooperative communication-based methods lose their central mechanism in a zero-sum game; opponent-embedding models define no fixed joint policy. Appendix~\ref{app:contemporary} argues each case in detail. However, we do perform the comparison against a decentralized model and show its unidentifiability in Appendix \ref{app:decentralized}.

\subsection{Sample Efficiency: NashDreamer vs. Model-Free Baselines}
\label{sec:sample_efficiency}
We compare against the identical model-free versions of our policy optimization algorithms in terms of real environment steps seen. To compensate for the additional steps seen in imagination, we also evaluate against RNaD augmented with a replay buffer.

Our environments directly provide infoset representation, used as an input for our model-free baseline. To ensure exact input parity for NashDreamer, we use these in place of observations. We validate the capabilities of NashDreamer to learn from observation-only feedback via an ablation in Goofspiel 5 (Appendix \ref{app:observations}), where we pass only the information of played card this turn and a win/loss/tie information at each turn, and observe highly similar performance to full infoset representation input.

\textbf{Goofspiel 5.} While both methods achieved comparable asymptotic performance, NashDreamer reaches it utilizing substantially fewer environment interactions (Figure~\ref{fig:nashconv}). 

\textbf{Leduc Poker.} Here the default NashDreamer hyperparameters actually deteriorate against the model free baseline (Figure~\ref{fig:nashconv}). Investigating the learned dynamics reveals a structural error present throughout training: the model cannot properly disambiguate outcomes at the root chance deal, creating uncertainty over the dealt private cards.  A variant with downscaled loss coefficients $\beta^{\text{enc}} = 0.01, \beta^{\text{dyn}}=0.1$ managed to recover performance near replay level but still suffers from an error of dealing two identical private cards, due to the variational assumption that the latent distributions are independent. We tested two more variants: a one-way KL loss, which sets $\beta^{\text{enc}} = 0$ and a flat variant with latent categorical space $1\times 34$ to try to prevent the variational gap. Empirically, the flat variant performs close to the tuned one, whereas the one-way variant has the worst results due to representation drift (Appendix~\ref{app:leducerror}), confirming our trade-off claim from Section~\ref{sec:theory_ideal}. Appendix~\ref{app:diagnostics} reports quantitative world-model diagnostics evidencing this failure.

\textbf{Phantom Tic-Tac-Toe.} Phantom TTT serves as a benchmark of a large, yet tractable game. In this game, both optimizers improve when wrapped in NashDreamer, with NashDreamer[RNaD] attaining model-free RNaD's final exploitability using roughly a quarter of interactions with environment (Figure~\ref{fig:nashconv}). Replay improves substantially on plain RNaD here, as it also does in Leduc, but it stays above NashDreamer[RNaD] throughout training.

\textbf{MMD as a model-free baseline.} On small games, model-free MMD converges faster than RNaD. Wrapping MMD in NashDreamer changes little, as MMD already largely converges within the world-model warm-up phase. Furthermore, in Leduc, the wrong world-model of vanilla hyperparameters causes degradation. In Phantom Tic-Tac-Toe, however, the model already brings visible acceleration.

\begin{figure}[t]
    \centering
    \includegraphics[width=\linewidth]{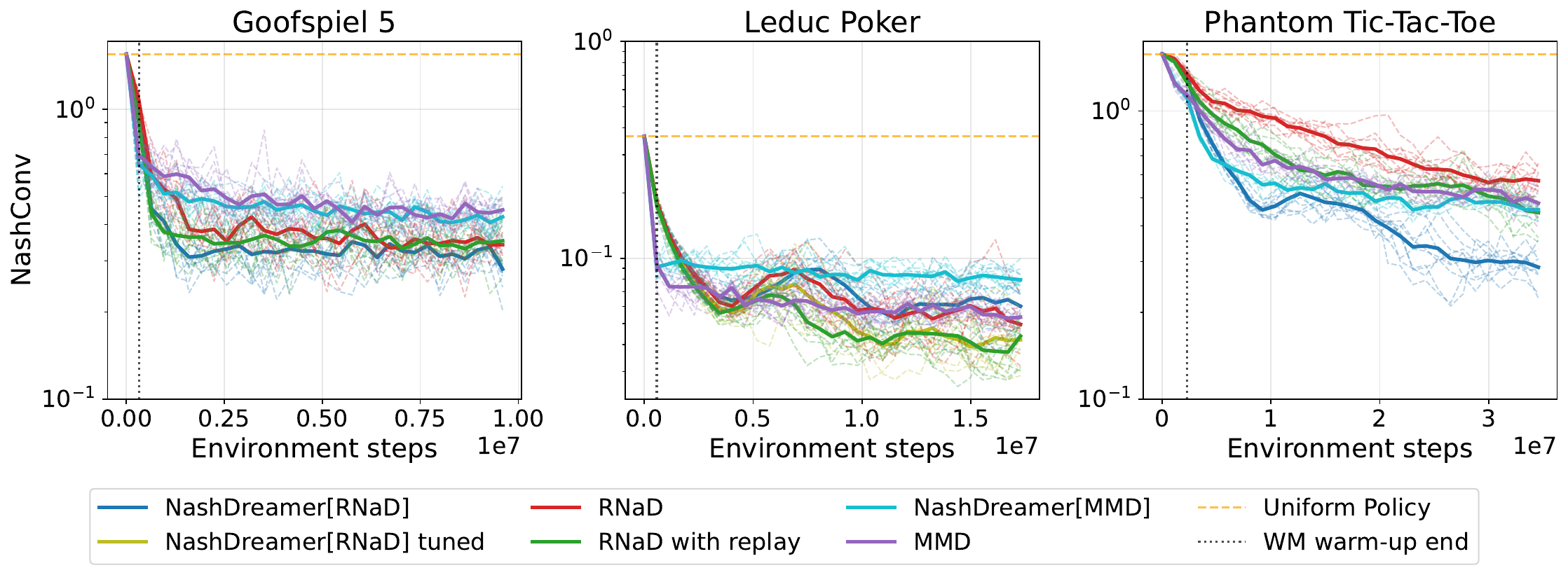}
    \caption{NashConv against real environment interaction on the three tested domains where it is exactly computable. Same hyperparameters across domains, only in Leduc we also show a tuned variant with $\beta^{\text{enc}} = 0.01$ and $\beta^{\text{dyn}} = 0.1$  for NashDreamer[RNaD]. In Goofspiel 5, and with the tuned variant in Leduc, NashDreamer reaches comparable benefits to a replay buffer for RNaD, but the vanilla hyperparameters in Leduc degrade due to a wrong world model. In Phantom Tic-Tac-Toe NashDreamer significantly outperforms the model-free baselines, continuing to improve for RNaD while eventually settling together for MMD. Bold curves are means over the $10$ training seeds, faint curves the individual seeds; the dotted line marks the end of the world-model warm-up.}
    \label{fig:nashconv}
\end{figure}

\subsection{Approximate NashConv on Larger Domains}
\label{sec:approx_expl}

Exact NashConv is intractable in Goofspiel 13 and Battleship, so we approximate it by budgeted best-response search. We freeze a trained checkpoint's policy and, from each seat in turn, train $10$ PPO seeds against it for $3{,}000$ gradient steps, evaluating each over $100{,}000$ games. The largest best-response value obtained by any PPO seed is taken as that seat's estimate, and the two seats are summed. Appendix~\ref{app:nashconv} gives the full protocol.

We have trained each of the algorithms for $30k$ minibatches, with a world-model warm-up of $2k$ minibatches for the NashDreamer variants and we evaluate the approximate NashConv each $10k$ minibatches. Since a fixed search budget need not recover the true best response, it is a \emph{budgeted lower bound} on NashConv: values are comparable across methods evaluated at the same budget, but not against the exact NashConv of Section~\ref{sec:sample_efficiency} and not against results obtained at any other budget. Intervals are $\pm 2\sigma$ over the 10 training seeds, not over PPO seeds.

\begin{table}[h]
\centering
\footnotesize
\setlength{\tabcolsep}{2pt}
\caption{Budgeted approximate NashConv (lower is better) at three checkpoints, as mean $\pm 2\sigma$ over the 10 training seeds. Values are a lower bound on NashConv at a fixed best-response budget and are not comparable with exact NashConv or with other budgets. ND denotes NashDreamer with the given actor-critic inside imagination; best value per domain in bold.}
\label{tab:approx_expl}
\begin{tabular}{lcccccc}
\toprule
 & \multicolumn{3}{c}{Goofspiel (13 cards)} & \multicolumn{3}{c}{Battleship (5$\times$5)} \\
\cmidrule(lr){2-4} \cmidrule(lr){5-7}
Method & 10k & 20k & 30k & 10k & 20k & 30k \\
\midrule
MMD & $0.827 \pm 0.076$ & $0.741 \pm 0.093$ & $0.732 \pm 0.075$ & $1.306 \pm 0.070$ & $1.161 \pm 0.095$ & $1.022 \pm 0.150$ \\
ND[MMD] & $0.853 \pm 0.074$ & $0.791 \pm 0.085$ & $0.763 \pm 0.075$ & $1.209 \pm 0.126$ & $0.969 \pm 0.097$ & $\mathbf{0.817 \pm 0.161}$ \\
RNaD & $0.742 \pm 0.388$ & $0.448 \pm 0.247$ & $0.356 \pm 0.291$ & $1.416 \pm 0.066$ & $1.326 \pm 0.132$ & $1.155 \pm 0.183$ \\
RNaD w/ replay & $0.756 \pm 0.074$ & $0.531 \pm 0.061$ & $0.499 \pm 0.074$ & $1.438 \pm 0.086$ & $1.395 \pm 0.122$ & $1.223 \pm 0.120$ \\
ND[RNaD] & $0.469 \pm 0.138$ & $0.250 \pm 0.145$ & $\mathbf{0.168 \pm 0.139}$ & $1.199 \pm 0.131$ & $1.184 \pm 0.168$ & $1.169 \pm 0.256$ \\
\bottomrule
\end{tabular}
\end{table}

In Goofspiel 13 NashDreamer[RNaD] reaches lower approximate NashConv after $20$k minibatches than model-free RNaD reaches only at $30$k, and it leads at every checkpoint. In Battleship the same holds for the other optimizer, NashDreamer[MMD] at $20$k surpassing MMD at $30$k. Which optimizer is the stronger one is domain-dependent, but our model nevertheless accelerates it further (Table~\ref{tab:approx_expl}). The off-policy replay failed to improve RNaD's sample efficiency. As theorized in Appendix~\ref{app:replay}, we attribute this failure to the variance injected by unbiasing the off-policy data. We also present head-to-head results in Appendix \ref{app:headtohead}, which show that NashDreamer beats the model-free baselines on Battleship, but no difference in head-to-head win rate is perceived on Goofspiel 13. We consider the budgeted approximate NashConv a more reliable metric of policy quality, as head-to-head measures strength against a particular opponent, rather than the robustness of the policy towards exploitation.

\section{Related Work}
\label{sec:related}

\textbf{Model-based RL.} MuZero~\citep{muzero} and Stochastic MuZero~\citep{stochasticmuzero} learn latent models for planning in perfect-information games. The Dreamer series~\citep{dreamerv1,dreamerv2,dreamerv3,dreamerv4} trains world models for imagination-based policy learning in single-agent settings. NashDreamer specifically builds upon the online architecture of DreamerV3, as adapting the offline learning paradigm of DreamerV4 to adversarial multi-agent settings introduces overfitting to a particular opponent set, as well as prohibitive requirements of dataset size. LAMIR~\citep{lamir} extends MuZero-style look-ahead reasoning to IIGs, but focuses on test-time planning rather than pure RL and assumes deterministic games.

\textbf{Multi-agent Dreamer.} MAMBA~\citep{mamba} and DreamerCF~\citep{dreamercf} extend Dreamer to cooperative settings with inter-agent communication. In competitive ones, \citet{dreamerzerosum} condition the world model on ego and opponent embeddings, the latter randomly initialized at test time and so requiring zero-shot generalization, and \citet{dreamerhockey} use single-player modeling with population-based self-play. All use policy-gradient actor-critics without game-theoretic convergence guarantees.

\textbf{Game-theoretic RL.} CFR~\citep{cfr} and its variants are the gold standard for solving IIGs but require explicit game tree access. DREAM~\citep{dream} removes this requirement via Monte Carlo sampling, enabling model-free deep CFR. RNaD/DeepNash~\citep{rnadstratego} achieves human-level Stratego play using model-free RL with regularized reward dynamics. PSRO~\citep{psro} maintains a population of best-response policies. MMD \citep{mmd} achieves superhuman performance in Stratego \citep{ataraxos}. All mentioned model-free approaches are potential alternatives to RNaD as the actor-critic component within NashDreamer's learned world model, and we substantiate this for MMD.

\section{Conclusion}
\label{sec:conclusion}

We introduced NashDreamer, a principled MBRL framework for two-player zero-sum imperfect information games. By pairing a centralized latent world model (MARSSM) with a decentralized, game-theoretic actor-critic (such as RNaD), NashDreamer overcomes the identifiability barriers inherent to multi-agent learning. Empirically, it significantly improves early-to-mid training convergence. 

However, as formalized theoretically and validated in Leduc Poker, environments with high chance stochasticity remain challenging due to the trade-off between stability and collapse. Resolving this tension without inducing degenerate minima is the most critical direction for future work. Additional limitations to overcome include relaxing the $t=0$ unrolling constraint and validating the architecture in much larger environments, closer to the desired real-world deployment.

\section{Broader Impacts}
We envision the primary application of our framework to be a building block for automated cyber defense. We acknowledge the dual-use nature of this research, as the same methodology could be leveraged to train sophisticated adversarial agents. We maintain, however, that the most effective countermeasure is making the algorithm widely accessible, allowing the community to train robust defenders with symmetric capabilities, rather than relying on security through obscurity.

%% file: diagrams/nash_dreamer_wm.tex
\begin{tikzpicture}[
  >=Stealth,
  arrow/.style={-{Stealth[length=2.2mm]}, line width=0.8pt, black},
  arrowdashed/.style={-{Stealth[length=2.2mm]}, line width=0.8pt, black, dash pattern=on 3pt off 2pt},
  arrowloss/.style={-{Stealth[length=2.2mm]}, line width=0.8pt, red, dash pattern=on 3pt off 2pt},
  arrowp1/.style={-{Stealth[length=2.2mm]}, line width=0.8pt, p1stroke},
  arrowp2/.style={-{Stealth[length=2.2mm]}, line width=0.8pt, p2stroke},
]
\draw[rounded corners=0.08cm,fill=white,draw=gray!50,line width=1.0pt] (11.400,-0.312) rectangle (13.640,-2.920);
\node[font=\scriptsize\bfseries,text=black] at (12.520,-0.544) {Legend};
\draw[rounded corners=0.02cm,fill=p1fill,draw=p1stroke,line width=1.0pt] (11.560,-0.736) rectangle (11.784,-0.960);
\node[anchor=west,font=\scriptsize,text=p1stroke] at (11.880,-0.848) {Player 1};
\draw[rounded corners=0.02cm,fill=p2fill,draw=p2stroke,line width=1.0pt] (11.560,-1.088) rectangle (11.784,-1.312);
\node[anchor=west,font=\scriptsize,text=p2stroke] at (11.880,-1.200) {Player 2};
\draw[rounded corners=0.02cm,fill=latentfill,draw=latentstroke,line width=1.0pt] (11.560,-1.440) rectangle (11.784,-1.664);
\node[anchor=west,font=\scriptsize,text=latentstroke] at (11.880,-1.552) {Det. latent};
\draw[arrowloss] (11.672,-1.792) -- (11.896,-1.792);
\node[anchor=west,font=\scriptsize,text=red] at (11.880,-1.856) {Loss};
\node[font=\scriptsize,text=gray!70!black] at (12.520,-2.400) {$\hat{x} = \text{latent}$};
\draw[rounded corners=0.1cm,fill=startfill,draw=startstroke,line width=1.0pt] (0.800,-0.520) rectangle (4.960,-1.400);
\node[font=\footnotesize\bfseries,text=startstroke] at (2.880,-0.736) {Real Data (t=1)};
\node[font=\scriptsize,text=startstroke] at (2.880,-1.152) {$(o^1_{1},\, o^2_{1},\, d_{1},\, r_{1},\, a^1_{d},\, a^2_{d})$};
\draw[rounded corners=0.06cm,fill=p1fill,draw=p1stroke,line width=1.0pt] (2.080,-2.240) rectangle (3.680,-2.880);
\node[font=\footnotesize,text=black] at (2.880,-2.560) {$\hat{\mathcal{I}}^1_{1}$};
\draw[arrowp1] (2.720,-1.408) -- (2.720,-2.240);
\node[anchor=east,font=\footnotesize,text=p1stroke] at (2.560,-1.916) {$a^1_{d},\, o^1_{1}$};
\draw[rounded corners=0.06cm,fill=p2fill,draw=p2stroke,line width=1.0pt] (2.080,-3.680) rectangle (3.680,-4.320);
\node[font=\footnotesize,text=black] at (2.880,-4.000) {$\hat{\mathcal{I}}^2_{1}$};
\draw[arrowp2] (3.040,-1.408) .. controls (4.000,-1.408) and (4.000,-3.680) .. (3.040,-3.680);
\node[anchor=west,font=\footnotesize,text=p2stroke] at (3.680,-3.333) {$a^2_{d},\, o^2_{1}$};
\draw[rounded corners=0.1cm,fill=groupfill,draw=groupstroke,line width=1.0pt,dash pattern=on 3pt off 2pt,fill opacity=0.6,draw opacity=0.6] (1.920,-5.360) rectangle (3.840,-7.600);
\node[font=\tiny\bfseries,text=groupstroke!70!black] at (2.880,-5.552) {Central WM};
\draw[rounded corners=0.06cm,fill=latentfill,draw=latentstroke,line width=1.0pt] (2.240,-5.680) rectangle (3.520,-6.160);
\node[font=\footnotesize,text=black] at (2.880,-5.920) {$\hat{h}_{1}$};
\begin{scope}[shift={(2.880,-7.040)}]
\draw[line width=0.7pt] (-0.281,-0.281) rectangle (0.281,0.281);
\pgfmathsetmacro{\gridBbRa}{int(random(0,3))}
\ifnum\gridBbRa=0 \fill[black] (-0.256,0.128) rectangle (-0.128,0.256); \else \draw[gray!55,line width=0.25pt] (-0.256,0.128) rectangle (-0.128,0.256); \fi
\ifnum\gridBbRa=1 \fill[black] (-0.128,0.128) rectangle (0.000,0.256); \else \draw[gray!55,line width=0.25pt] (-0.128,0.128) rectangle (0.000,0.256); \fi
\ifnum\gridBbRa=2 \fill[black] (0.000,0.128) rectangle (0.128,0.256); \else \draw[gray!55,line width=0.25pt] (0.000,0.128) rectangle (0.128,0.256); \fi
\ifnum\gridBbRa=3 \fill[black] (0.128,0.128) rectangle (0.256,0.256); \else \draw[gray!55,line width=0.25pt] (0.128,0.128) rectangle (0.256,0.256); \fi
\pgfmathsetmacro{\gridBbRb}{int(random(0,3))}
\ifnum\gridBbRb=0 \fill[black] (-0.256,0.000) rectangle (-0.128,0.128); \else \draw[gray!55,line width=0.25pt] (-0.256,0.000) rectangle (-0.128,0.128); \fi
\ifnum\gridBbRb=1 \fill[black] (-0.128,0.000) rectangle (0.000,0.128); \else \draw[gray!55,line width=0.25pt] (-0.128,0.000) rectangle (0.000,0.128); \fi
\ifnum\gridBbRb=2 \fill[black] (0.000,0.000) rectangle (0.128,0.128); \else \draw[gray!55,line width=0.25pt] (0.000,0.000) rectangle (0.128,0.128); \fi
\ifnum\gridBbRb=3 \fill[black] (0.128,0.000) rectangle (0.256,0.128); \else \draw[gray!55,line width=0.25pt] (0.128,0.000) rectangle (0.256,0.128); \fi
\pgfmathsetmacro{\gridBbRc}{int(random(0,3))}
\ifnum\gridBbRc=0 \fill[black] (-0.256,-0.128) rectangle (-0.128,0.000); \else \draw[gray!55,line width=0.25pt] (-0.256,-0.128) rectangle (-0.128,0.000); \fi
\ifnum\gridBbRc=1 \fill[black] (-0.128,-0.128) rectangle (0.000,0.000); \else \draw[gray!55,line width=0.25pt] (-0.128,-0.128) rectangle (0.000,0.000); \fi
\ifnum\gridBbRc=2 \fill[black] (0.000,-0.128) rectangle (0.128,0.000); \else \draw[gray!55,line width=0.25pt] (0.000,-0.128) rectangle (0.128,0.000); \fi
\ifnum\gridBbRc=3 \fill[black] (0.128,-0.128) rectangle (0.256,0.000); \else \draw[gray!55,line width=0.25pt] (0.128,-0.128) rectangle (0.256,0.000); \fi
\pgfmathsetmacro{\gridBbRd}{int(random(0,3))}
\ifnum\gridBbRd=0 \fill[black] (-0.256,-0.256) rectangle (-0.128,-0.128); \else \draw[gray!55,line width=0.25pt] (-0.256,-0.256) rectangle (-0.128,-0.128); \fi
\ifnum\gridBbRd=1 \fill[black] (-0.128,-0.256) rectangle (0.000,-0.128); \else \draw[gray!55,line width=0.25pt] (-0.128,-0.256) rectangle (0.000,-0.128); \fi
\ifnum\gridBbRd=2 \fill[black] (0.000,-0.256) rectangle (0.128,-0.128); \else \draw[gray!55,line width=0.25pt] (0.000,-0.256) rectangle (0.128,-0.128); \fi
\ifnum\gridBbRd=3 \fill[black] (0.128,-0.256) rectangle (0.256,-0.128); \else \draw[gray!55,line width=0.25pt] (0.128,-0.256) rectangle (0.256,-0.128); \fi
\end{scope}
\node[font=\scriptsize,text=black] at (2.880,-7.536) {$\hat{z}_{1}$};
\draw[arrowdashed] (2.880,-6.160) -- (2.880,-6.640);
\node[anchor=west,font=\scriptsize\ttfamily\bfseries,text=gray!70!black] at (2.976,-6.419) {Posterior};
\draw[arrow] (3.360,-1.408) .. controls (5.280,-1.408) and (5.280,-5.680) .. (3.360,-5.680);
\node[anchor=east,font=\footnotesize,text=black] at (4.240,-4.560) {$a^1_{d},\, a^2_{d}$};
\draw[arrow] (3.680,-1.408) .. controls (6.560,-1.408) and (6.560,-6.640) .. (3.680,-6.640);
\node[anchor=west,font=\footnotesize,text=black] at (4.768,-4.560) {$o^1_{1},\, o^2_{1}$};
\draw[arrowloss] (2.560,-2.880) .. controls (1.600,-2.880) and (1.600,-5.920) .. (2.240,-5.920);
\draw[arrowloss] (2.560,-4.320) .. controls (2.020,-4.736) and (1.960,-5.920) .. (2.240,-5.920);
\node[anchor=west,font=\scriptsize\ttfamily\bfseries,text=red] at (2.400,-5.120) {Unifier Loss};
\draw[rounded corners=0.06cm,fill=white,draw=black!80,line width=1.0pt] (1.040,-8.160) rectangle (4.720,-8.864);
\node[font=\footnotesize\bfseries,text=black] at (2.880,-8.400) {Predictions};
\node[font=\scriptsize,text=gray!70!black] at (2.880,-8.704) {$\hat{o}^1_{1},\, \hat{o}^2_{1},\, \hat{r}_{1},\, \hat{d}_{1},\, \hat{\ell}^1_{1},\, \hat{\ell}^2_{1}$};
\draw[arrow] (2.880,-7.600) -- (2.880,-8.160);
\draw[arrowp1] (3.680,-2.560) -- (7.520,-2.560);
\draw[arrowp2] (3.680,-4.000) -- (7.520,-4.000);
\draw[arrow] (3.520,-5.920) -- (7.680,-5.920);
\draw[arrow] (3.360,-7.040) -- (7.680,-6.080);
\draw[rounded corners=0.1cm,fill=startfill,draw=startstroke,line width=1.0pt] (6.240,-0.520) rectangle (10.400,-1.400);
\node[font=\footnotesize\bfseries,text=startstroke] at (8.320,-0.736) {Real Data (t=2)};
\node[font=\scriptsize,text=startstroke] at (8.320,-1.152) {$(o^1_{2},\, o^2_{2},\, d_{2},\, r_{2},\, a^1_{1},\, a^2_{1})$};
\draw[rounded corners=0.06cm,fill=p1fill,draw=p1stroke,line width=1.0pt] (7.520,-2.240) rectangle (9.120,-2.880);
\node[font=\footnotesize,text=black] at (8.320,-2.560) {$\hat{\mathcal{I}}^1_{2}$};
\draw[arrowp1] (8.160,-1.408) -- (8.160,-2.240);
\node[anchor=east,font=\footnotesize,text=p1stroke] at (8.000,-1.916) {$a^1_{1},\, o^1_{2}$};
\draw[rounded corners=0.06cm,fill=p2fill,draw=p2stroke,line width=1.0pt] (7.520,-3.680) rectangle (9.120,-4.320);
\node[font=\footnotesize,text=black] at (8.320,-4.000) {$\hat{\mathcal{I}}^2_{2}$};
\draw[arrowp2] (8.480,-1.408) .. controls (9.440,-1.408) and (9.440,-3.680) .. (8.480,-3.680);
\node[anchor=west,font=\footnotesize,text=p2stroke] at (9.120,-3.333) {$a^2_{1},\, o^2_{2}$};
\draw[rounded corners=0.1cm,fill=groupfill,draw=groupstroke,line width=1.0pt,dash pattern=on 3pt off 2pt,fill opacity=0.6,draw opacity=0.6] (7.360,-5.360) rectangle (9.280,-7.600);
\node[font=\tiny\bfseries,text=groupstroke!70!black] at (8.320,-5.552) {Central WM};
\draw[rounded corners=0.06cm,fill=latentfill,draw=latentstroke,line width=1.0pt] (7.680,-5.680) rectangle (8.960,-6.160);
\node[font=\footnotesize,text=black] at (8.320,-5.920) {$\hat{h}_{2}$};
\begin{scope}[shift={(8.320,-7.040)}]
\draw[line width=0.7pt] (-0.281,-0.281) rectangle (0.281,0.281);
\pgfmathsetmacro{\gridBcRa}{int(random(0,3))}
\ifnum\gridBcRa=0 \fill[black] (-0.256,0.128) rectangle (-0.128,0.256); \else \draw[gray!55,line width=0.25pt] (-0.256,0.128) rectangle (-0.128,0.256); \fi
\ifnum\gridBcRa=1 \fill[black] (-0.128,0.128) rectangle (0.000,0.256); \else \draw[gray!55,line width=0.25pt] (-0.128,0.128) rectangle (0.000,0.256); \fi
\ifnum\gridBcRa=2 \fill[black] (0.000,0.128) rectangle (0.128,0.256); \else \draw[gray!55,line width=0.25pt] (0.000,0.128) rectangle (0.128,0.256); \fi
\ifnum\gridBcRa=3 \fill[black] (0.128,0.128) rectangle (0.256,0.256); \else \draw[gray!55,line width=0.25pt] (0.128,0.128) rectangle (0.256,0.256); \fi
\pgfmathsetmacro{\gridBcRb}{int(random(0,3))}
\ifnum\gridBcRb=0 \fill[black] (-0.256,0.000) rectangle (-0.128,0.128); \else \draw[gray!55,line width=0.25pt] (-0.256,0.000) rectangle (-0.128,0.128); \fi
\ifnum\gridBcRb=1 \fill[black] (-0.128,0.000) rectangle (0.000,0.128); \else \draw[gray!55,line width=0.25pt] (-0.128,0.000) rectangle (0.000,0.128); \fi
\ifnum\gridBcRb=2 \fill[black] (0.000,0.000) rectangle (0.128,0.128); \else \draw[gray!55,line width=0.25pt] (0.000,0.000) rectangle (0.128,0.128); \fi
\ifnum\gridBcRb=3 \fill[black] (0.128,0.000) rectangle (0.256,0.128); \else \draw[gray!55,line width=0.25pt] (0.128,0.000) rectangle (0.256,0.128); \fi
\pgfmathsetmacro{\gridBcRc}{int(random(0,3))}
\ifnum\gridBcRc=0 \fill[black] (-0.256,-0.128) rectangle (-0.128,0.000); \else \draw[gray!55,line width=0.25pt] (-0.256,-0.128) rectangle (-0.128,0.000); \fi
\ifnum\gridBcRc=1 \fill[black] (-0.128,-0.128) rectangle (0.000,0.000); \else \draw[gray!55,line width=0.25pt] (-0.128,-0.128) rectangle (0.000,0.000); \fi
\ifnum\gridBcRc=2 \fill[black] (0.000,-0.128) rectangle (0.128,0.000); \else \draw[gray!55,line width=0.25pt] (0.000,-0.128) rectangle (0.128,0.000); \fi
\ifnum\gridBcRc=3 \fill[black] (0.128,-0.128) rectangle (0.256,0.000); \else \draw[gray!55,line width=0.25pt] (0.128,-0.128) rectangle (0.256,0.000); \fi
\pgfmathsetmacro{\gridBcRd}{int(random(0,3))}
\ifnum\gridBcRd=0 \fill[black] (-0.256,-0.256) rectangle (-0.128,-0.128); \else \draw[gray!55,line width=0.25pt] (-0.256,-0.256) rectangle (-0.128,-0.128); \fi
\ifnum\gridBcRd=1 \fill[black] (-0.128,-0.256) rectangle (0.000,-0.128); \else \draw[gray!55,line width=0.25pt] (-0.128,-0.256) rectangle (0.000,-0.128); \fi
\ifnum\gridBcRd=2 \fill[black] (0.000,-0.256) rectangle (0.128,-0.128); \else \draw[gray!55,line width=0.25pt] (0.000,-0.256) rectangle (0.128,-0.128); \fi
\ifnum\gridBcRd=3 \fill[black] (0.128,-0.256) rectangle (0.256,-0.128); \else \draw[gray!55,line width=0.25pt] (0.128,-0.256) rectangle (0.256,-0.128); \fi
\end{scope}
\node[font=\scriptsize,text=black] at (8.320,-7.536) {$\hat{z}_{2}$};
\draw[arrowdashed] (8.320,-6.160) -- (8.320,-6.640);
\node[anchor=west,font=\scriptsize\ttfamily\bfseries,text=gray!70!black] at (8.416,-6.419) {Posterior};
\draw[arrow] (8.800,-1.408) .. controls (10.720,-1.408) and (10.720,-5.680) .. (8.800,-5.680);
\node[anchor=east,font=\footnotesize,text=black] at (9.680,-4.560) {$a^1_{1},\, a^2_{1}$};
\draw[arrow] (9.120,-1.408) .. controls (12.000,-1.408) and (12.000,-6.640) .. (9.120,-6.640);
\node[anchor=west,font=\footnotesize,text=black] at (10.208,-4.560) {$o^1_{2},\, o^2_{2}$};
\draw[arrowloss] (8.000,-2.880) .. controls (7.040,-2.880) and (7.040,-5.920) .. (7.680,-5.920);
\draw[arrowloss] (8.000,-4.320) .. controls (7.460,-4.736) and (7.400,-5.920) .. (7.680,-5.920);
\node[anchor=west,font=\scriptsize\ttfamily\bfseries,text=red] at (7.840,-5.120) {Unifier Loss};
\draw[rounded corners=0.06cm,fill=white,draw=black!80,line width=1.0pt] (6.480,-8.160) rectangle (10.160,-8.864);
\node[font=\footnotesize\bfseries,text=black] at (8.320,-8.400) {Predictions};
\node[font=\scriptsize,text=gray!70!black] at (8.320,-8.704) {$\hat{o}^1_{2},\, \hat{o}^2_{2},\, \hat{r}_{2},\, \hat{d}_{2},\, \hat{\ell}^1_{2},\, \hat{\ell}^2_{2}$};
\draw[arrow] (8.320,-7.600) -- (8.320,-8.160);
\end{tikzpicture}

%% file: diagrams/nash_dreamer_imagination.tex
\begin{tikzpicture}[
  >=Stealth,
  arrow/.style={-{Stealth[length=2.2mm]}, line width=0.8pt, black},
  arrowdashed/.style={-{Stealth[length=2.2mm]}, line width=0.8pt, black, dash pattern=on 3pt off 2pt},
  arrowactor/.style={-{Stealth[length=2.2mm]}, line width=0.8pt, actorstroke},
  arrowp1/.style={-{Stealth[length=2.2mm]}, line width=0.8pt, p1stroke},
  arrowp2/.style={-{Stealth[length=2.2mm]}, line width=0.8pt, p2stroke},
]
\draw[rounded corners=0.06cm,fill=p1fill,draw=p1stroke,line width=1.0pt] (2.080,-1.600) rectangle (3.680,-2.240);
\node[font=\footnotesize,text=black] at (2.880,-1.920) {$\hat{\mathcal{I}}^1_{1}$};
\draw[rounded corners=0.06cm,fill=p2fill,draw=p2stroke,line width=1.0pt] (2.080,-3.200) rectangle (3.680,-3.840);
\node[font=\footnotesize,text=black] at (2.880,-3.520) {$\hat{\mathcal{I}}^2_{1}$};
\draw[rounded corners=0.1cm,fill=groupfill,draw=groupstroke,line width=1.0pt,dash pattern=on 3pt off 2pt,fill opacity=0.6,draw opacity=0.6] (1.920,-6.800) rectangle (3.840,-9.040);
\node[font=\tiny\bfseries,text=groupstroke!70!black] at (2.880,-6.992) {Central WM};
\draw[rounded corners=0.06cm,fill=latentfill,draw=latentstroke,line width=1.0pt] (2.240,-7.120) rectangle (3.520,-7.600);
\node[font=\footnotesize,text=black] at (2.880,-7.360) {$\hat{h}_{1}$};
\begin{scope}[shift={(2.880,-8.480)}]
\draw[line width=0.7pt] (-0.281,-0.281) rectangle (0.281,0.281);
\pgfmathsetmacro{\gridAbRa}{int(random(0,3))}
\ifnum\gridAbRa=0 \fill[black] (-0.256,0.128) rectangle (-0.128,0.256); \else \draw[gray!55,line width=0.25pt] (-0.256,0.128) rectangle (-0.128,0.256); \fi
\ifnum\gridAbRa=1 \fill[black] (-0.128,0.128) rectangle (0.000,0.256); \else \draw[gray!55,line width=0.25pt] (-0.128,0.128) rectangle (0.000,0.256); \fi
\ifnum\gridAbRa=2 \fill[black] (0.000,0.128) rectangle (0.128,0.256); \else \draw[gray!55,line width=0.25pt] (0.000,0.128) rectangle (0.128,0.256); \fi
\ifnum\gridAbRa=3 \fill[black] (0.128,0.128) rectangle (0.256,0.256); \else \draw[gray!55,line width=0.25pt] (0.128,0.128) rectangle (0.256,0.256); \fi
\pgfmathsetmacro{\gridAbRb}{int(random(0,3))}
\ifnum\gridAbRb=0 \fill[black] (-0.256,0.000) rectangle (-0.128,0.128); \else \draw[gray!55,line width=0.25pt] (-0.256,0.000) rectangle (-0.128,0.128); \fi
\ifnum\gridAbRb=1 \fill[black] (-0.128,0.000) rectangle (0.000,0.128); \else \draw[gray!55,line width=0.25pt] (-0.128,0.000) rectangle (0.000,0.128); \fi
\ifnum\gridAbRb=2 \fill[black] (0.000,0.000) rectangle (0.128,0.128); \else \draw[gray!55,line width=0.25pt] (0.000,0.000) rectangle (0.128,0.128); \fi
\ifnum\gridAbRb=3 \fill[black] (0.128,0.000) rectangle (0.256,0.128); \else \draw[gray!55,line width=0.25pt] (0.128,0.000) rectangle (0.256,0.128); \fi
\pgfmathsetmacro{\gridAbRc}{int(random(0,3))}
\ifnum\gridAbRc=0 \fill[black] (-0.256,-0.128) rectangle (-0.128,0.000); \else \draw[gray!55,line width=0.25pt] (-0.256,-0.128) rectangle (-0.128,0.000); \fi
\ifnum\gridAbRc=1 \fill[black] (-0.128,-0.128) rectangle (0.000,0.000); \else \draw[gray!55,line width=0.25pt] (-0.128,-0.128) rectangle (0.000,0.000); \fi
\ifnum\gridAbRc=2 \fill[black] (0.000,-0.128) rectangle (0.128,0.000); \else \draw[gray!55,line width=0.25pt] (0.000,-0.128) rectangle (0.128,0.000); \fi
\ifnum\gridAbRc=3 \fill[black] (0.128,-0.128) rectangle (0.256,0.000); \else \draw[gray!55,line width=0.25pt] (0.128,-0.128) rectangle (0.256,0.000); \fi
\pgfmathsetmacro{\gridAbRd}{int(random(0,3))}
\ifnum\gridAbRd=0 \fill[black] (-0.256,-0.256) rectangle (-0.128,-0.128); \else \draw[gray!55,line width=0.25pt] (-0.256,-0.256) rectangle (-0.128,-0.128); \fi
\ifnum\gridAbRd=1 \fill[black] (-0.128,-0.256) rectangle (0.000,-0.128); \else \draw[gray!55,line width=0.25pt] (-0.128,-0.256) rectangle (0.000,-0.128); \fi
\ifnum\gridAbRd=2 \fill[black] (0.000,-0.256) rectangle (0.128,-0.128); \else \draw[gray!55,line width=0.25pt] (0.000,-0.256) rectangle (0.128,-0.128); \fi
\ifnum\gridAbRd=3 \fill[black] (0.128,-0.256) rectangle (0.256,-0.128); \else \draw[gray!55,line width=0.25pt] (0.128,-0.256) rectangle (0.256,-0.128); \fi
\end{scope}
\node[font=\scriptsize,text=black] at (2.880,-8.976) {$\hat{z}_{1}$};
\draw[arrowdashed] (2.880,-7.600) -- (2.880,-8.080);
\node[anchor=west,font=\scriptsize\ttfamily\bfseries,text=gray!70!black] at (2.976,-7.859) {Posterior};
\draw[rounded corners=0.06cm,fill=actorfill,draw=actorstroke,line width=1.0pt] (4.720,-4.400) rectangle (6.480,-5.040);
\node[font=\footnotesize\bfseries,text=black] at (5.600,-4.592) {Actor 1};
\node[font=\scriptsize,text=gray] at (5.600,-4.848) {$\hat{a}^1_{1}$};
\draw[rounded corners=0.06cm,fill=actorfill,draw=actorstroke,line width=1.0pt] (4.720,-5.280) rectangle (6.480,-5.920);
\node[font=\footnotesize\bfseries,text=black] at (5.600,-5.472) {Actor 2};
\node[font=\scriptsize,text=gray] at (5.600,-5.728) {$\hat{a}^2_{1}$};
\draw[arrowactor] (3.680,-2.080) -- (4.720,-4.720);
\draw[arrowactor] (3.680,-3.680) -- (4.720,-5.520);
\draw[arrow] (6.480,-4.880) .. controls (7.200,-4.880) and (8.000,-6.448) .. (8.160,-7.120);
\draw[arrow] (5.920,-5.920) .. controls (5.920,-6.565) and (7.760,-6.448) .. (7.920,-7.120);
\node[anchor=east,font=\footnotesize,text=latentstroke] at (5.840,-6.331) {$\hat{a}^1_{1},\, \hat{a}^2_{1}$};
\draw[arrowp1] (6.480,-4.560) -- (7.520,-2.160);
\node[anchor=east,font=\footnotesize,text=p1stroke] at (6.680,-3.960) {$\hat{a}^1_{1}$};
\draw[arrowp2] (6.480,-5.440) -- (7.520,-3.760);
\node[anchor=west,font=\footnotesize,text=p2stroke] at (7.080,-4.600) {$\hat{a}^2_{1}$};
\draw[arrowp1] (3.680,-1.920) -- (7.520,-1.920);
\draw[arrowp2] (3.680,-3.520) -- (7.520,-3.520);
\draw[arrow] (3.520,-7.360) -- (7.680,-7.360);
\draw[arrow] (3.360,-8.480) -- (7.680,-7.520);
\draw[rounded corners=0.03cm,fill=startfill,draw=startstroke,line width=1.0pt] (1.280,-0.960) rectangle (4.480,-1.440);
\node[font=\scriptsize\bfseries,text=startstroke] at (2.880,-1.200) {Start (WM training)};
\draw[rounded corners=0.06cm,fill=p1fill,draw=p1stroke,line width=1.0pt] (7.520,-1.600) rectangle (9.120,-2.240);
\node[font=\footnotesize,text=black] at (8.320,-1.920) {$\hat{\mathcal{I}}^1_{2}$};
\draw[rounded corners=0.06cm,fill=p2fill,draw=p2stroke,line width=1.0pt] (7.520,-3.200) rectangle (9.120,-3.840);
\node[font=\footnotesize,text=black] at (8.320,-3.520) {$\hat{\mathcal{I}}^2_{2}$};
\draw[rounded corners=0.1cm,fill=groupfill,draw=groupstroke,line width=1.0pt,dash pattern=on 3pt off 2pt,fill opacity=0.6,draw opacity=0.6] (7.360,-6.800) rectangle (9.280,-9.040);
\node[font=\tiny\bfseries,text=groupstroke!70!black] at (8.320,-6.992) {Central WM};
\draw[rounded corners=0.06cm,fill=latentfill,draw=latentstroke,line width=1.0pt] (7.680,-7.120) rectangle (8.960,-7.600);
\node[font=\footnotesize,text=black] at (8.320,-7.360) {$\hat{h}_{2}$};
\begin{scope}[shift={(8.320,-8.480)}]
\draw[line width=0.7pt] (-0.281,-0.281) rectangle (0.281,0.281);
\pgfmathsetmacro{\gridAcRa}{int(random(0,3))}
\ifnum\gridAcRa=0 \fill[black] (-0.256,0.128) rectangle (-0.128,0.256); \else \draw[gray!55,line width=0.25pt] (-0.256,0.128) rectangle (-0.128,0.256); \fi
\ifnum\gridAcRa=1 \fill[black] (-0.128,0.128) rectangle (0.000,0.256); \else \draw[gray!55,line width=0.25pt] (-0.128,0.128) rectangle (0.000,0.256); \fi
\ifnum\gridAcRa=2 \fill[black] (0.000,0.128) rectangle (0.128,0.256); \else \draw[gray!55,line width=0.25pt] (0.000,0.128) rectangle (0.128,0.256); \fi
\ifnum\gridAcRa=3 \fill[black] (0.128,0.128) rectangle (0.256,0.256); \else \draw[gray!55,line width=0.25pt] (0.128,0.128) rectangle (0.256,0.256); \fi
\pgfmathsetmacro{\gridAcRb}{int(random(0,3))}
\ifnum\gridAcRb=0 \fill[black] (-0.256,0.000) rectangle (-0.128,0.128); \else \draw[gray!55,line width=0.25pt] (-0.256,0.000) rectangle (-0.128,0.128); \fi
\ifnum\gridAcRb=1 \fill[black] (-0.128,0.000) rectangle (0.000,0.128); \else \draw[gray!55,line width=0.25pt] (-0.128,0.000) rectangle (0.000,0.128); \fi
\ifnum\gridAcRb=2 \fill[black] (0.000,0.000) rectangle (0.128,0.128); \else \draw[gray!55,line width=0.25pt] (0.000,0.000) rectangle (0.128,0.128); \fi
\ifnum\gridAcRb=3 \fill[black] (0.128,0.000) rectangle (0.256,0.128); \else \draw[gray!55,line width=0.25pt] (0.128,0.000) rectangle (0.256,0.128); \fi
\pgfmathsetmacro{\gridAcRc}{int(random(0,3))}
\ifnum\gridAcRc=0 \fill[black] (-0.256,-0.128) rectangle (-0.128,0.000); \else \draw[gray!55,line width=0.25pt] (-0.256,-0.128) rectangle (-0.128,0.000); \fi
\ifnum\gridAcRc=1 \fill[black] (-0.128,-0.128) rectangle (0.000,0.000); \else \draw[gray!55,line width=0.25pt] (-0.128,-0.128) rectangle (0.000,0.000); \fi
\ifnum\gridAcRc=2 \fill[black] (0.000,-0.128) rectangle (0.128,0.000); \else \draw[gray!55,line width=0.25pt] (0.000,-0.128) rectangle (0.128,0.000); \fi
\ifnum\gridAcRc=3 \fill[black] (0.128,-0.128) rectangle (0.256,0.000); \else \draw[gray!55,line width=0.25pt] (0.128,-0.128) rectangle (0.256,0.000); \fi
\pgfmathsetmacro{\gridAcRd}{int(random(0,3))}
\ifnum\gridAcRd=0 \fill[black] (-0.256,-0.256) rectangle (-0.128,-0.128); \else \draw[gray!55,line width=0.25pt] (-0.256,-0.256) rectangle (-0.128,-0.128); \fi
\ifnum\gridAcRd=1 \fill[black] (-0.128,-0.256) rectangle (0.000,-0.128); \else \draw[gray!55,line width=0.25pt] (-0.128,-0.256) rectangle (0.000,-0.128); \fi
\ifnum\gridAcRd=2 \fill[black] (0.000,-0.256) rectangle (0.128,-0.128); \else \draw[gray!55,line width=0.25pt] (0.000,-0.256) rectangle (0.128,-0.128); \fi
\ifnum\gridAcRd=3 \fill[black] (0.128,-0.256) rectangle (0.256,-0.128); \else \draw[gray!55,line width=0.25pt] (0.128,-0.256) rectangle (0.256,-0.128); \fi
\end{scope}
\node[font=\scriptsize,text=black] at (8.320,-8.976) {$\hat{z}_{2}$};
\draw[arrowdashed] (8.320,-7.600) -- (8.320,-8.080);
\node[anchor=west,font=\scriptsize\ttfamily\bfseries,text=gray!70!black] at (8.416,-7.859) {Prior};
\draw[rounded corners=0.08cm,fill=white,draw=gray!50,line width=1.0pt] (1.480,-4.160) rectangle (3.720,-6.368);
\node[font=\scriptsize\bfseries,text=black] at (2.600,-4.384) {Legend};
\draw[rounded corners=0.02cm,fill=p1fill,draw=p1stroke,line width=1.0pt] (1.640,-4.576) rectangle (1.864,-4.800);
\node[anchor=west,font=\scriptsize,text=p1stroke] at (1.960,-4.688) {Player 1};
\draw[rounded corners=0.02cm,fill=p2fill,draw=p2stroke,line width=1.0pt] (1.640,-4.928) rectangle (1.864,-5.152);
\node[anchor=west,font=\scriptsize,text=p2stroke] at (1.960,-5.040) {Player 2};
\draw[rounded corners=0.02cm,fill=latentfill,draw=latentstroke,line width=1.0pt] (1.640,-5.280) rectangle (1.864,-5.504);
\node[anchor=west,font=\scriptsize,text=latentstroke] at (1.960,-5.392) {Det. latent};
\draw[rounded corners=0.02cm,fill=actorfill,draw=actorstroke,line width=1.0pt] (1.640,-5.632) rectangle (1.864,-5.856);
\node[anchor=west,font=\scriptsize,text=actorstroke] at (1.960,-5.744) {Actors};
\node[font=\scriptsize,text=gray!70!black] at (2.600,-6.144) {$\hat{x} = \text{latent}$};
\draw[rounded corners=0.06cm,fill=decodedfill,draw=black!70,line width=1.0pt] (9.760,-8.160) rectangle (11.760,-8.800);
\node[font=\footnotesize\bfseries,text=black] at (10.760,-8.352) {Decoded};
\node[font=\scriptsize,text=gray] at (10.760,-8.656) {$\hat{o}^1_{2},\, \hat{o}^2_{2}$};
\draw[arrow] (8.640,-8.480) -- (9.760,-8.480);
\draw[arrowp1] (11.360,-8.160) .. controls (11.360,-1.920) .. (9.120,-1.920);
\node[anchor=west,font=\footnotesize,text=p1stroke] at (11.300,-4.800) {$\hat{o}^1_{2}$};
\draw[arrowp2] (10.400,-8.160) .. controls (10.400,-3.520) .. (9.120,-3.520);
\node[anchor=east,font=\footnotesize,text=p2stroke] at (10.360,-5.600) {$\hat{o}^2_{2}$};
\end{tikzpicture}

%% file: appendices.tex
\newpage
\appendix
\section{Regularized Policy Optimization in Imperfect-Information Games}
\label{app:optimizers}

This appendix expands Section~\ref{sec:rnad} with the full statement of the two actor-critic algorithms that NashDreamer runs inside imagination. Their implementations and hyperparameters are given separately in Appendices~\ref{app:rnad_baseline} and~\ref{app:mmd_baseline}.

\subsection{Why Policy Gradients Cycle}

Follow the Regularized Leader (FoReL)~\citep{shalev2007forel} is a general framework from online learning. When the regularizer is entropy, FoReL yields Replicator Dynamics, which is closely connected to policy-gradient methods via Neural Replicator Dynamics (NeuRD)~\citep{neurd}. However, FoReL has been proven to cycle and fail to converge to a Nash equilibrium in both matrix games~\citep{forelcycle} and sequential IIGs~\citep{rnadtheory}.

\subsection{Regularized Nash Dynamics}

Regularized Nash Dynamics (RNaD)~\citep{rnadstratego} overcomes this by introducing policy-dependent reward regularization:
\begin{equation}
    R^{\text{reg}}(h, a, \pi, \pi^{\text{reg}}) = R(h, a) - \eta \log \frac{\pi_1(a_1 \given I_1(h))}{\pi_1^{\text{reg}}(a_1 \given I_1(h))} + \eta \log \frac{\pi_2(a_2 \given I_2(h))}{\pi_2^{\text{reg}}(a_2 \given I_2(h))},
    \label{eq:rnad_reward}
\end{equation}
where $\eta > 0$ is the regularization coefficient. This formulation naturally preserves the zero-sum property of the game. By periodically updating the regularization policy $\pi^{\text{reg}} \leftarrow \pi$, RNaD provably converges to an NE. The practical deep reinforcement learning instantiation, DeepNash~\citep{rnadstratego}, combines RNaD with a V-trace estimator~\citep{impalavtrace} for the actor-critic. We adapt the simultaneous-move variant of RNaD~\citep{lamir}, which introduces counterfactual importance weighting to enable off-policy learning.

\subsection{Magnetic Mirror Descent}

Magnetic Mirror Descent (MMD)~\citep{mmd} is the second optimizer we employ. It is a mirror-descent method whose update carries two regularizers, one anchoring the policy to a \emph{magnet} policy $\pi^{\text{mag}}$ and one to the previous iterate:
\begin{equation}
    \pi_{t+1} = \argmin_{\pi} \left\{ -\langle Q^{\pi_t}, \pi \rangle + \alpha \KL(\pi \parallel \pi^{\text{mag}}) + \tfrac{1}{\eta} \KL(\pi \parallel \pi_t) \right\}.
    \label{eq:mmd}
\end{equation}
Both methods thus regularize toward a reference policy, but through different mechanisms: RNaD transforms the \emph{reward} (Eq.~\ref{eq:rnad_reward}), which preserves the zero-sum property of the game, whereas MMD constrains the \emph{policy update} itself. We use the instantiation of~\citet{pgconvergence}, in which the magnet is held \emph{fixed and uniform}, so that $\alpha \KL(\pi \parallel \pi^{\text{mag}})$ reduces to an entropy bonus and the proximal term is realized by PPO-style clipped updates. One consequence matters for the comparisons that follow: with a fixed magnet and $\alpha > 0$, MMD converges to a quantal response equilibrium rather than to an NE, so its asymptotic exploitability is bounded away from zero. Unlike RNaD, there are, to the best of our knowledge, no convergence results for MMD when switching the magnet policy or annealing the magnet regularization.

\section{Proofs and Theoretical Analysis}
\label{app:theory}

\subsection{Definitions and Formal Setup}
\label{app:defs}

We first formalize the concepts necessary to state and prove our theoretical results. We make use of the fact that a POSG can be unrolled into an EFG structure \citep{fosg}. Because games of this class form a tree, every non-root history has a unique predecessor.

As in Section~\ref{sec:posg}, we lift the state-indexed functions of the game (transition, observation, legal actions) to histories through the last state $s(h)$ of the history, writing e.g. $T(h, a) := T(s(h), a)$. Accordingly, $H(T(h, a))$ below denotes the entropy of the chance outcome at $h$.

\begin{definition}[Predecessor Function]
Let $h$ be a non-root true history in the game tree. We define the predecessor function $\text{Prev}(h) = (h_{\text{prev}}, a_{\text{prev}}, r_{\text{prev}})$, which returns the unique immediately preceding history $h_{\text{prev}}$, the joint action $a_{\text{prev}}$ taken at $h_{\text{prev}}$ that led to $h$, and the immediate environment reward $r_{\text{prev}}$ generated by this transition.
\end{definition}

\begin{definition}[Latent Trajectory]
\label{def:latenttrajectory}
A \emph{latent trajectory} is a sequence
$$\hat{\tau} = (m_0, a_0, \hat{o}_0, \hat{I}_0, m_1, a_1, \hat{o}_1, \hat{I}_1, \dots, m_t, \hat{o}_t)$$
produced by unrolling the model, where the latent model state $m_l = (\hat{h}_l, z_l)$ encapsulates both the deterministic recurrent state $\hat{h}_l$ and the stochastic categorical state $z_l$, $a_l = (a_{1,l}, a_{2,l})$ is the joint action, $\hat{o}_l$ is the joint observation reconstruction of $m_l$, and $\hat{I}_l = (\hat{I}_{1,l}, \hat{I}_{2,l})$ are the latent infoset embeddings. Latent trajectories are the objects over which tree-structure isomorphism is stated (Definition~\ref{def:treeisomorph}).
\end{definition}

\begin{definition}[Latent Reach Probability]
\label{def:latentreach}
Let $\hat{\tau}$ be a latent trajectory as in Definition~\ref{def:latenttrajectory}. We define the reach probability of a specific latent model state $m_t$ under a joint sampling policy $\mu$ as:
$$\lreach_\mu(m_t) = p^{\theta}(z_t \mid \hat{h}_t) \prod_{l=0}^{t - 1} p^{\theta}(z_l \mid \hat{h}_l) \prod_{i = 1}^{2} \mu_i(a_{i,l} \mid \hat{I}_{i,l})$$
Similarly, the reach probability of the deterministic state $\hat{h}_t$, which strictly isolates the probability prior to the final chance node resolution, is:
$$\lreach_\mu(\hat{h}_t) = \prod_{l=0}^{t - 1} p^{\theta}(z_l \mid \hat{h}_l) \prod_{i = 1}^{2} \mu_i(a_{i,l} \mid \hat{I}_{i,l})$$
\end{definition}

\begin{definition}[Valid State and History Sets]
Let $\mathbb{M}$ be the set of all possible latent model states generated by the architecture. We define $\mathbb{M}^* \subseteq \mathbb{M}$ as the subset of latent states that are reachable ($\lreach_\mu(m) > 0$) under a sufficiently exploratory sampling policy $\mu$ with full support over all legal actions. 

Furthermore, let $\mathbb{H}$ be the set of all valid histories in the true game tree. We define $\mathbb{H}_{nc} \subseteq \mathbb{H}$ as the subset of true histories where the last node is not a chance node. 
\end{definition}
We restrict our mapping to $\mathbb{H}_{nc}$ because our architecture inherently models environmental stochasticity as latent transitions rather than explicit nodes.

\begin{definition}[$\epsilon$-Distance Matching]
\label{def:epsdist}
Let $\hat{r}(m)$ and $\hat{d}(m)$ denote the predicted reward and termination flag from the latent state $m$. Let $\hat{l}_i(m)$ and $\hat{o}_i(m)$ be the predicted legal actions and reconstructed observation for player $i$. A reachable latent model state $m \in \mathbb{M}^*$ is considered within $\epsilon$-distance of a true history $h \in \mathbb{H}_{nc}$ (where $\text{Prev}(h) = (h_{\text{prev}}, a_{\text{prev}}, r_{\text{prev}})$) if its predictions match the deterministic true environment emissions within an arbitrarily small error bound $\epsilon$, where $\|\cdot\|$ denotes the max-norm over vector components:
\begin{enumerate}
    \item $\|\hat{r}(m) - r_{\text{prev}}\| < \epsilon$
    \item $\|\hat{d}(m) - d(h)\| < \epsilon$
    \item For $i \in \{1, 2\}$: $\|\hat{l}_i(m) - L_i(h)\| < \epsilon$
    \item For $i \in \{1, 2\}$: $\|\hat{o}_i(m) - O_i(h)\| < \epsilon$
\end{enumerate}
where $d(h) = 1$ if the history is terminal and $0$ otherwise.
\end{definition}

To make condition 4 concrete, consider Leduc Poker, where each $O_i(h)$ is a binary vector whose blocks encode the player's private card, the public card (if already dealt), and the current bets. Take a history $h$ in which Player~1 holds the jack and no public card has been dealt, so the private-card block of $O_1(h)$ is one-hot on the jack and the public-card block is all zeros. A latent state $m$ whose decoder outputs $\hat{o}_1(m)$ with $0.99$ on the jack, at most $0.01$ on every other private card and at most $0.01$ on every public-card entry is within $\epsilon$-distance for $\epsilon = 0.02$: the reconstruction is unambiguous both about which private card was dealt and about no public card being present. A latent state that instead splits mass between two private cards, say $0.5$ and $0.5$, violates condition 4 for every $\epsilon < 0.5$. Such a state has aliased two distinct true histories into one latent state, which is exactly the structural failure excluded from $\mathcal{M}^*$, and which we observe empirically at the initial private deal in Leduc under the vanilla hyperparameters (Appendix~\ref{app:leducerror}).

\begin{definition}[Tree-Structure Isomorphism]
\label{def:treeisomorph}
A generative world model is tree-structure isomorphic if there exists a bijective mapping $f_\epsilon: \mathbb{H}_{nc} \to \mathbb{M}^*$, parameterized by the tolerance $\epsilon$ of Definition~\ref{def:epsdist} and evaluated at a history $h$, that satisfies the following conditions for all $h \in \mathbb{H}_{nc}$:
\begin{enumerate}
    \item $f_\epsilon(h)$ is within $\epsilon$-distance of $h$.
    \item $f_\epsilon(h)$ is the ending state of a latent trajectory $\hat{\tau}$ (Definition~\ref{def:latenttrajectory}) that contains the exact same joint action sequence as the true history $h$, and $\hat{o}_l$ is within $\epsilon$-distance of $o_l$ for each timestep $l$.
\end{enumerate}
\end{definition}

\begin{definition}[Infoset Isomorphism]
\label{def:isetisomorph}
Let $\hat{I}_i(m)$ denote the latent infoset embedding of player $i$ for latent state $m$. Assume a world model is tree-structure isomorphic with bijective mapping $f_\epsilon: \mathbb{H}_{nc} \to \mathbb{M}^*$. The model is infoset-isomorphic if, for any player $i \in \{1, 2\}$ and any pair of valid true histories $h, h' \in \mathbb{H}_{nc}$, their corresponding mapped model states $m = f_\epsilon(h)$ and $m' = f_\epsilon(h')$ satisfy the following equivalence relation:
$$\|\hat{I}_i(m) - \hat{I}_i(m')\| < \epsilon \iff I_i(h) = I_i(h')$$
where $I_i(h)$ denotes the true infoset of player $i$ at history $h$.
\end{definition}

A world model exhibits branch duplication if the encoder introduces artificial stochasticity, mapping a single true history to multiple distinct latent model states. Semi-formally, instead of a bijection $f_\epsilon$, the model induces a mapping $f'_\epsilon: \mathbb{H}_{nc} \to \mathcal{P}(\mathbb{M}^*)$ where a true history $h \in \mathbb{H}_{nc}$ maps to a set of reachable model states $\mathbb{M}_h \subseteq \mathbb{M}^*$. We assume these sets are disjoint for distinct histories and cover the entire reachable model space ($\mathbb{M}_h \cap \mathbb{M}_{h'} = \emptyset$ for $h \neq h'$), all states within $\mathbb{M}_h$ are within $\epsilon$-distance for all predictor targets. Crucially, all states within $\mathbb{M}_h$ yield the exact same joint latent infoset representation $(\hat{I}_1, \hat{I}_2)$ (this condition will be satisfied directly, since all the model states were created by the same AOH and all satisfy the $\epsilon$-distance accuracy).

\begin{definition}[Information Partition Violation]
A world model violates the game's information partitioning if it is tree-structure isomorphic but not infoset-isomorphic.
\end{definition}
 This means the model correctly captures the objective physical dynamics, but its mapped latent infosets either alias distinct true infosets together, or improperly separate histories belonging to the same true infoset.

\subsection{Proofs}

We will now analyze the loss for isomorphic model, leading up to a proof for Theorem \ref{thm:isomorphic_optimal}. We assume that the KL losses are unclipped, as using the free-nats clipping lower bounds the KL loss value at a constant rather than 0.

\begin{lemma}
\label{lem:dynamics_bound}
Let $\hat{h}$ be a deterministic recurrent state induced by the true history and joint action $(h_{\text{prev}}, a_{\text{prev}})$. For any model that perfectly reconstructs deterministic environment emissions (including both infoset-isomorphic models and models exhibiting branch duplication), the expected dynamics loss over the chance outcome is exactly equal to the inherent Shannon entropy of the true environment transition: $ H(T(h_{\text{prev}}, a_{\text{prev}}))$.
\end{lemma}

\begin{proof}
The dynamics loss is defined as the expected KL divergence between the posterior encoder $q^{\theta}(z \mid o, \hat{h})$ and the prior dynamics predictor $p^{\theta}(z \mid \hat{h})$ across all possible observations $o$ sampled from the true environment distribution $p(o \mid \hat{h})$. 

Assuming an optimal prior that perfectly matches the aggregated marginal posterior, we have $p^{\theta}(z \mid \hat{h}) = \mathbb{E}_{o \sim p} [q^{\theta}(z \mid o, \hat{h})]$. We define the induced true marginal distribution of the latent state as $p(z \mid \hat{h}) = p^{\theta}(z \mid \hat{h})$, and the joint distribution over observations and latent chance states as $p(o, z \mid \hat{h}) = p(o \mid \hat{h})q^{\theta}(z \mid o, \hat{h})$. 

Using this joint distribution, we expand the expected KL divergence:
\begin{align*}
    \mathbb{E}_{o \sim p} \left[ \text{KL} \left( q^{\theta}(z \mid o, \hat{h}) \parallel p^{\theta}(z \mid \hat{h}) \right) \right] 
    &= \sum_o p(o \mid \hat{h}) \sum_z q^{\theta}(z \mid o, \hat{h}) \log \frac{q^{\theta}(z \mid o, \hat{h})}{p^{\theta}(z \mid \hat{h})} \\
    &= \sum_{o, z} p(o, z \mid \hat{h}) \log \frac{p(o, z \mid \hat{h}) / p(o \mid \hat{h})}{p(z \mid \hat{h})} \\
    &= \sum_{o, z} p(o, z \mid \hat{h}) \log \frac{p(o, z \mid \hat{h})}{p(o \mid \hat{h})p(z \mid \hat{h})}
\end{align*}

Because the final expression is the expected log-ratio of the joint distribution to the product of its marginals, it is exactly the formal definition of Conditional Mutual Information (CMI), $I(z, o \mid \hat{h})$. Applying standard information-theoretic identities, we rewrite this CMI as the difference between conditional entropies:
\begin{align*}
    I(z, o \mid \hat{h}) &= H(o \mid \hat{h}) - H(o \mid z, \hat{h})
\end{align*}

Because both an isomorphic model and a branch-duplicated model map true histories to states within an $\epsilon$-distance bound for predictor targets, the combination of the deterministic history representation $\hat{h}$ and the chosen latent chance state $z$ must perfectly reconstruct the true environment observation $o$. Since $o$ is fully and deterministically defined by the tuple $(\hat{h}, z)$, there is zero remaining uncertainty about the observation. Therefore, the conditional entropy $H(o \mid z, \hat{h}) = 0$.

Consequently, the expected dynamics loss simplifies to $H(o \mid \hat{h})$. Because the recurrent state $\hat{h}$ uniquely identifies the predecessor history and joint action $(h_{\text{prev}}, a_{\text{prev}})$, the entropy of the subsequent observation $o$ is exactly the inherent aleatoric entropy of the true environment chance node:
\begin{align*}
    \mathbb{E}_{o \sim p} \left[ \text{KL} \left( q^{\theta}(z \mid o, \hat{h}) \parallel p^{\theta}(z \mid \hat{h}) \right) \right] &= H(T(h_{\text{prev}}, a_{\text{prev}}))
\end{align*}

Because this equality holds as long as $H(o \mid z, \hat{h}) = 0$ (which is required to maintain the $\epsilon$-distance bounds for valid reconstruction), it is invariant to how the encoder maps the observation into the latent space. Thus, branch duplication cannot decrease the expected KL penalty below the true chance entropy.
\end{proof}

With the dynamics loss lower-bounded by the inherent environment stochasticity, we now establish the convergence properties of the remaining objective terms under the assumption of unbounded network capacity.

\begin{corollary}[Distinct Posterior Supports]
\label{cor:distinct_posterior}
By the definition of tree-structure isomorphism, the bijective mapping $f_\epsilon(h)$ requires the model state to be induced by the exact same AOH as the true history. Because the recurrent state $\hat{h}$ contains all historical context up to the current observation, the stochastic state $z$ must necessarily encode the chance-induced observation. To satisfy $H(o \mid z, \hat{h}) = 0$ (as established in Lemma~\ref{lem:dynamics_bound}), the posterior encoder $q^{\theta}(z \mid o, \hat{h})$ must assign distinct support of $q^{\theta}(z \mid o, \hat{h})$ to each outcome of $T(h_{\text{prev}}, a_{\text{prev}})$, such that these supports are mutually disjoint.

\end{corollary}
We will, for simplicity, assume a unique $z$ for each outcome, thus reducing the posterior to a deterministic (Dirac delta/one-hot) distribution. Now we prove that such a distribution drives all the deterministic emissions losses to 0, before showing that non-deterministic posteriors (those of $\mathcal{M}^b$) are suboptimal.

\begin{lemma}[Predictor and Unifier Loss Minimization]
\label{lem:predictor_unifier_zero}
Let $\mathcal{M}^*$ be an infoset-isomorphic model. As $\epsilon \to 0$ in the $\epsilon$-distance bound, the expected loss for all deterministic predictors (reward, continuation, legal actions, observation reconstruction) and all latent infoset targets decreases to zero.
\end{lemma}

\begin{proof}
By the definition of $\epsilon$-distance matching, all predictions generated by a mapped state $m \in \mathbb{M}^*$ match the true environment emissions within an arbitrary bound $\epsilon$. As $\epsilon \to 0$, the Mean Squared Error (MSE) and Cross-Entropy losses for these predictive targets trivially evaluate to zero. 

Furthermore, the decentralized infoset networks process the true action-observation history (AOH) to generate $\hat{I}_1$ and $\hat{I}_2$. Given unbounded capacity, the auto-encoding targets (decoding the current observation and previous action from the infoset embedding) are perfectly reconstructed, achieving zero loss.

Crucially, under our assumptions of perfect recall and collective observability (Section~\ref{sec:posg}), the intersection of the true infosets of all players uniquely identifies the true underlying history: $I_1(h) \cap I_2(h) = \{h\}$. Because the model is infoset-isomorphic, its latent infosets perfectly preserve this partitioning ($\|\hat{I}_i(m) - \hat{I}_i(m')\| < \epsilon \iff I_i(h) = I_i(h')$). Therefore, the joint latent infoset embedding $(\hat{I}_1, \hat{I}_2)$ uniquely identifies the mapped latent state $m$. Because this mapping is uniquely resolvable, the centralized unifier network can perfectly reconstruct $m$ from $(\hat{I}_1, \hat{I}_2)$, allowing the unifier loss to minimize to zero.
\end{proof}

Using these, we can prove the first result of Section~\ref{sec:theory_ideal}.

\begin{theorem}
\label{thm:isomorphic_optimal}
Let $\mu$ be an arbitrary joint sampling policy. Let $\mathcal{M}^*$ be an infoset-isomorphic world model, $\mathcal{M}^b$ be a world model exhibiting branch duplication, and $\mathcal{M}^l$ be a model violating the game's information partitioning. It holds that $\mathbb{E}_\mu[\mathcal{L}^{model}(\mathcal{M}^*)] < \mathbb{E}_\mu[\mathcal{L}^{model}(\mathcal{M}^b)]$ and $\mathbb{E}_\mu[\mathcal{L}^{\rm{model}}(\mathcal{M}^*)] < \mathbb{E}_\mu[\mathcal{L}^{\rm{model}}(\mathcal{M}^l)]$. Furthermore, under any $\mu$, $\mathcal{M}^*$ achieves an expected loss:
$$ \mathbb{E}_\mu[\mathcal{L}^{\rm{model}}(\mathcal{M}^*)] = (\beta^{\text{enc}} + \beta^{\text{dyn}}) \mathbb{E}_{h, a \sim \mu}[H(T(h, a))], $$
where $H$ denotes Shannon entropy.
\end{theorem}

\begin{proof} (Proof of Theorem \ref{thm:isomorphic_optimal})
\label{prf:isomorphmin}
We first establish the exact expected loss for the infoset-isomorphic model, $\mathcal{M}^*$, under the sampling policy $\mu$. 

By Lemma~\ref{lem:predictor_unifier_zero}, because $\mathcal{M}^*$ is both tree-structure isomorphic and infoset-isomorphic, its $\epsilon$-distance bounds approach zero, allowing all deterministic predictor losses (reward, continuation, legal actions, observation reconstruction) and all unifier/infoset losses to minimize to zero. 

The only remaining non-zero terms in the objective are the KL divergence losses. As established in Lemma~\ref{lem:dynamics_bound}, the expected KL divergence at a given recurrent state $\hat{h}$ induced by the true predecessor $(h, a)$ evaluates exactly to the chance entropy $H(T(h, a))$. Because the objective utilizes two separate KL penalties—one for updating the dynamics prior (scaled by $\beta^{\text{dyn}}$) and one for updating the posterior encoder (scaled by $\beta^{\text{enc}}$)—the total expected loss at this specific step is:
$$ \mathcal{L}^{\text{KL}}(\hat{h}) = (\beta^{\text{enc}} + \beta^{\text{dyn}}) H(T(h, a)) $$

To compute the expected loss of the entire model under the joint sampling policy $\mu$, we aggregate this per-step loss over all reachable latent recurrent states. We denote the set of all true chance nodes as $\mathbb{H}_c = \mathbb{H} \setminus \mathbb{H}_{nc}$ and the set of all reachable recurrent states as $\hat{\mathbb{H}}^{*}$. Because the bijective mapping $f_\epsilon: \mathbb{H}_{nc} \to \mathbb{M}^*$ guarantees a one-to-one correspondence between true non-chance histories and latent states, if we reparametrize deterministic transitions to be represented via a trivial chance node with a single outcome, each latent recurrent state $\hat{h}_t$ corresponds perfectly to a true chance node transition from some $(h, a)$, i.e. $f_\epsilon$ induces also a bijective mapping  $g : \hat{\mathbb{H}}^{*} \to \mathbb{H}_c$.

Furthermore, because the optimal prior exactly captures the true marginal chance distribution, the latent reach probability of any recurrent state, $\lreach_\mu(\hat{h})$, exactly equals the true environment reach probability of its corresponding chance node, $\reach_\mu(h, a)$. Consequently, we can translate the latent summation directly into an expectation over the true game space:
\begin{align*}
    \mathbb{E}_\mu[\mathcal{L}(\mathcal{M}^*)] &= \sum_{\hat{h} \in \hat{\mathbb{H}}^{*}} \lreach_\mu(\hat{h}) \mathcal{L}^{\text{KL}}(\hat{h}) \\
    &= \sum_{h \in \mathbb{H}_{nc}} \sum_{a \in \mathcal{A}} \reach_\mu(h, a) \left( (\beta^{\text{enc}} + \beta^{\text{dyn}}) H(T(h, a)) \right) \\
    &= (\beta^{\text{enc}} + \beta^{\text{dyn}}) \mathbb{E}_{h, a \sim \mu}[H(T(h, a))]
\end{align*}

Next, we establish the strict suboptimality of the model exhibiting branch duplication, $\mathcal{M}^b$. By Lemma~\ref{lem:dynamics_bound}, $\mathcal{M}^b$ achieves the same KL loss. However, by definition, $\mathcal{M}^b$ maps a single true history $h \in \mathbb{H}_{nc}$ to a set of distinct reachable latent states $\mathbb{M}_h \subseteq \mathbb{M}^*$. Crucially, all duplicated states $m \in \mathbb{M}_h$ yield the exact same joint latent infoset $(\hat{I}_1, \hat{I}_2)$. 

The centralized unifier network is a deterministic function trained to predict the central model state $m = (\hat{h}, z)$ from the joint latent infosets. In $\mathcal{M}^b$, the input $(\hat{I}_1, \hat{I}_2)$ remains constant across the set $\mathbb{M}_h$, but the target state $m$ varies as the encoder distributes the observation across different  stochastic outcomes $z$. Because a deterministic neural network mapping a constant input to multiple varying targets incurs an irreducible variance penalty, it must yield a strictly positive unifier loss. Therefore, $\mathbb{E}_\mu[\mathcal{L}(\mathcal{M}^b)] > \mathbb{E}_\mu[\mathcal{L}(\mathcal{M}^*)]$.

Finally, we establish the strict suboptimality of the model violating the game's information partitioning, $\mathcal{M}^l$. By definition, $\mathcal{M}^l$ is tree-structure isomorphic but fails the infoset isomorphism condition. This means there exists at least one player $i$ and two valid true histories $h, h' \in \mathbb{H}_{nc}$ mapped to latent states $m, m'$ such that the equivalence relation $\|\hat{I}_i(m) - \hat{I}_i(m')\| < \epsilon \iff I_i(h) = I_i(h')$ is broken. This violation must manifest in one of two cases:

\textbf{Case 1:} $I_i(h) = I_i(h')$ but $\|\hat{I}_i(m) - \hat{I}_i(m')\| \ge \epsilon$. In games with perfect recall, histories in the same true infoset share the exact same Action-Observation History (AOH) for player $i$. Because the decentralized infoset network processes this AOH to generate $\hat{I}_i$, it acts as a deterministic function. It is impossible for a deterministic function to output distinct embeddings for identical inputs. Thus, this case cannot practically occur.

\textbf{Case 2:} $I_i(h) \neq I_i(h')$ but $\|\hat{I}_i(m) - \hat{I}_i(m')\| < \epsilon$. Because the true infosets are distinct, perfect recall dictates that the underlying AOH for player $i$ must differ in at least one previous action or observation. However, the model aliases them to the same latent infoset embedding. The architecture utilizes decentralized auto-encoding targets, forcing the network to decode the current observation and previous action from $\hat{I}_i$. Because a deterministic decoder cannot map a single, aliased embedding to two distinct true AOH targets, it necessarily incurs a strictly positive reconstruction error. 

We note that forcing this individual infoset reconstruction loss is a theoretical necessity, not merely an empirical stabilization trick. In an asymmetric game where Player 1 possesses perfect information, Player 1's infoset alone could uniquely identify the history $h$. Without individual decentralized targets, the centralized unifier loss could perfectly minimize to zero relying exclusively on Player 1, entirely failing to constrain Player 2's latent partition. 

Because Case 1 is structurally impossible, the violation in $\mathcal{M}^l$ must manifest as Case 2, incurring a strictly positive individual infoset loss. Therefore, $\mathbb{E}_\mu[\mathcal{L}(\mathcal{M}^l)] > \mathbb{E}_\mu[\mathcal{L}(\mathcal{M}^*)]$. 

\end{proof}

\begin{remark}[Optimization and Benign Branch Duplication]
\label{rem:benign_duplication}
While Theorem~\ref{thm:isomorphic_optimal} establishes that the branch-duplicated model $\mathcal{M}^b$ has a strictly higher global minimum than the perfectly isomorphic model $\mathcal{M}^*$, standard one-way optimization (where the centralized unifier target is treated as a constant with stop-gradients) does not provide a direct gradient signal to the encoder to actively merge these duplicated branches. Consequently, in practice, the model may converge to $\mathcal{M}^b$ rather than $\mathcal{M}^*$. However, because all duplicated states within $\mathbb{M}_h$ accurately reconstruct the true environment emissions and preserve the true marginal chance probabilities, this branch duplication is a benign structural artifact that does not degrade downstream policy optimization, unlike information partition violations or posterior collapse.
\end{remark}

Because $\mathcal{M}^*$ accurately captures the underlying POSG structure, which can be ``unrolled'' into an EFG~\citep{fosg}, and because $\mathcal{M}^b$ is indistinguishable from it from the actor-critic's view, the RNaD convergence result applies directly.

\begin{corollary}[Nash Equilibrium Convergence]
\label{cor:ne_convergence}
Assume that (i) the world model has converged to the infoset-isomorphic model $\mathcal{M}^*$ (or to $\mathcal{M}^b$) and is held fixed during policy optimization, and (ii) the RNaD dynamics are realized exactly, i.e., with exact policy evaluation and without function-approximation or sampling error. Then NashDreamer converges to a Nash equilibrium of the underlying game.
\end{corollary}

\begin{proof}[Proof of Theorem \ref{thm:aliasedcounterexample}]
\label{prf:counterexample}
To prove the existence of this degenerate global minimum, we construct a specific counterexample environment and evaluate the expected losses of the isomorphic model $\mathcal{M}^*$ and the collapsed model $\mathcal{M}^{\text{collapse}}$. 

Assume the standard Dreamer loss scaling coefficients: $\beta^{\text{enc}} = 0.1$, and all other coefficients (dynamics, reward, observation, and infoset) are $1.0$.

\textbf{Counterexample Environment Construction:} 
Consider an environment that begins with a single chance node with 101 uniformly distributed outcomes. The remainder of the environment is deterministic, and a deterministic joint policy $\mu$ guarantees reaching the terminal state. 
The initial chance node outcome has two distinct effects:
\begin{enumerate}
    \item It dictates a continuous component of a subsequent observation for Player 1, $o_1^t$, selecting a value such that its transformed representation in symlog space falls uniformly in the set $\{0, 0.01, 0.02, \dots, 1\}$.
    \item It penalizes the final terminal reward, subtracting a value such that the final reward in symlog space falls uniformly in the same set $\{0, 0.01, 0.02, \dots, 1\}$.
\end{enumerate}

\textbf{Expected Loss of the Isomorphic Model ($\mathcal{M}^*$):}
By Theorem~\ref{thm:isomorphic_optimal}, the isomorphic model incurs zero prediction and unifier losses, leaving only the KL divergence penalty. The inherent entropy of the 101-outcome uniform chance node is $\log(101) \approx 4.615$. Applying the loss coefficients, the expected loss is exactly:
$$ \mathbb{E}_\mu[\mathcal{L}(\mathcal{M}^*)] = (\beta^{\text{enc}} + \beta^{\text{dyn}}) H(T) = (0.1 + 1.0) \log(101) \approx 5.08 $$

\textbf{Expected Loss of the Collapsed Model ($\mathcal{M}^{\text{collapse}}$):}
Now consider a degenerate model where the posterior completely ignores the observation and collapses to the prior. Because $q^{\theta}(z \mid o_t, \hat{h}_t) = p^{\theta}(z \mid \hat{h}_t)$, the KL divergence loss is exactly $0$. The encoder maps all 101 chance outcomes to a single model state. Consequently, the model must rely on this static representation to predict the varying observations and rewards. 

For the observation reconstruction (optimized via Mean Squared Error), the target symlog values span $\{0, \dots, 1\}$. A sufficiently trained predictor will output the mean value of $0.5$, the maximum possible squared error for any outcome is $(1 - 0.5)^2 = 0.25$. Therefore, the expected observation loss is bounded: $\mathcal{L}^{\text{obs}} \le 0.25$.

For the reward prediction, we utilize the standard two-hot cross-entropy loss over discretized bins. The true targets in symlog space fall uniformly in $\{0, \dots, 1\}$. In the two-hot encoding scheme, a target $y \in [0, 1]$ is represented as a linear mixture: $(1-y)$ probability mass on the bin for $0$, and $y$ probability mass on the bin for $1$. A sufficiently trained predictor thus outputs a uniform distribution over the two adjacent bins (predicting $0.5$ for bin $0$ and $0.5$ for bin $1$), the cross-entropy loss for any target $y$ is:
$$ - \left( (1-y)\log(0.5) + y\log(0.5) \right) = -\log(0.5) = \log(2) \approx 0.69 $$
Thus, the expected reward loss is bounded: $\mathcal{L}^{\text{reward}} \le 0.7$.

The individual infoset losses can all minimize to zero. For Player 2, all true histories fall into the same infoset, requiring only a constant embedding. For Player 1, the decentralized network can simply produce a distinct latent infoset embedding for each unique observation $o_1^t$. Finally, because the encoder collapses all outcomes into a single latent state, the centralized unifier network only needs to map these varying joint latent infosets to a single, constant target state. Learning this trivial surjective mapping allows the unifier loss to reach $0$. We explicitly note that this degenerate mapping does not contradict our previous result (Theorem~\ref{thm:isomorphic_optimal}) regarding the strict suboptimality of information partition violations. That penalty fundamentally relied on the model maintaining tree-structure isomorphism (perfect deterministic reconstruction), whereas this collapsed model deliberately abandons structural isomorphism to save on KL divergence.

Summing the bounded predictor penalties, the total expected loss of the collapsed model is:
$$ \mathbb{E}_\mu[\mathcal{L}^{model}(\mathcal{M}^{\text{collapse}})] \le 0.25 + 0.69 = 0.94 $$

Because $\mathbb{E}_\mu[\mathcal{L}^{model}(\mathcal{M}^{\text{collapse}})] \le 0.94$ and $\mathbb{E}_\mu[\mathcal{L}^{model}(\mathcal{M}^*)] \approx 5.08$, the degenerate collapsed model achieves a strictly lower objective loss than the perfectly structure-preserving isomorphic model. This proves that under standard static KL-balancing, the global minimum is vulnerable to posterior collapse in environments with high chance entropy and low-magnitude predictor variance.
\end{proof}

We note that, in our counterexample environment, we let only one player's observations and reward be determined by the chance outcome. Since our world model is also centralized, an analogous counterexample single-agent environment could be constructed, proving this property also for DreamerV3.

Finally, we discuss that with the encoder/representation loss disabled, our model will, in theory, converge to an infoset-isomorphic model or a model with duplicated branches, both of which enable sound actor-critic learning.

\section{Training Algorithm}
\label{app:algorithm}

Algorithm~\ref{alg:nashdreamer} states the complete NashDreamer training loop, collecting in one place the world-model objective of Section~\ref{sec:worldmodel_training}, the imagination procedure of Section~\ref{sec:imagination} and the actor-critic of Section~\ref{sec:actor_critic}. Three properties of the loop are worth making explicit, as they are easy to miss in the prose. Real environment data is collected at \emph{every} iteration. The world model is updated at every iteration and is not frozen during the warm-up period. The warm-up period gates the imagination phase alone: during it the actor-critic trains on real trajectories only, and the world model trains exactly as it does afterwards. All hyperparameter values are listed in Table~\ref{tab:shared_hyperparameters}.

\begin{algorithm}[h]
\caption{NashDreamer[RNaD]. $N^{\text{warm}}$ is the world-model warm-up length, $N^{\text{reg}}$ the regularization-policy update period, $|B|$ the batch size and $T^{\max}$ the maximum trajectory length of the game (Table~\ref{tab:shared_hyperparameters}). Phase (ii) corresponds to Figure~\ref{fig:ndwm} and phase (iii) to Figure~\ref{fig:ndimag}.}
\label{alg:nashdreamer}
\begin{algorithmic}[1]
\REQUIRE game $\mathcal{G}$; world model parameters $\theta$ comprising the MARSSM $(\text{seq}, \text{enc}, \text{dyn}, \text{iset})$ and the predictors; decentralized actors $\pi_1, \pi_2$; centralized critic $v$
\STATE $\pi_i^{\text{reg}} \gets \pi_i$ for $i \in \{1, 2\}$
\FOR{$n = 1, \dots, N^{\text{steps}}$}
    \STATE \textit{(i) Real environment interaction --- at every iteration}
    \STATE collect a minibatch $B$ of $|B|$ complete trajectories from $\mathcal{G}$ under the joint sampling policy $\mu$, each unrolled from $s_0$ to termination
    \STATE \textit{(ii) World model update --- at every iteration; $\theta$ is never frozen}
    \FORALL{trajectories in $B$, for $t = 0, \dots, T$}
        \STATE $\hat{h}_t \gets \text{seq}(m_{t-1}, (a_1^{t-1}, a_2^{t-1}))$ \COMMENT{joint action}
        \STATE $q^{\theta}_t \gets \text{enc}(\hat{h}_t, (o_1^t, o_2^t))$, \quad $p^{\theta}_t \gets \text{dyn}(\hat{h}_t)$, \quad $z_t \sim q^{\theta}_t$ \COMMENT{posterior}
        \STATE $m_t \gets [\hat{h}_t, z_t]$
        \STATE $\hat{I}_i^t \gets \text{iset}(\hat{I}_i^{t-1}, a_i^{t-1}, o_i^t)$ for $i \in \{1, 2\}$ \COMMENT{own AOH only}
    \ENDFOR
    \STATE take a gradient step on $\theta$ with $\mathcal{L}^{\text{model}}$ of Eq.~\ref{eq:total_loss}
    \STATE \textit{(iii) Imagination --- gated by the warm-up; $\theta$ frozen throughout}
    \STATE $\hat{B} \gets \emptyset$
    \IF{$n > N^{\text{warm}}$}
        \FORALL{non-chance nodes of each trajectory in $B$, terminals included}
            \STATE $m_0, \hat{I}_1^0, \hat{I}_2^0 \gets$ posterior states at $t = 0$ computed in (ii) \COMMENT{all rollouts start at the root}
            \FOR{$t = 0, \dots, T^{\max} - 1$}
                \STATE $\hat{l}_1^t, \hat{l}_2^t, \hat{r}_t, \hat{d}_t \gets \text{predictors}(m_t)$
                \STATE $a_i^t \sim \pi_i(\cdot \given \hat{I}_i^t, \hat{l}_i^t)$ for $i \in \{1, 2\}$ \COMMENT{decentralized}
                \STATE $\hat{h}_{t+1} \gets \text{seq}(m_t, (a_1^t, a_2^t))$
                \STATE $p^{\theta}_{t+1} \gets \text{dyn}(\hat{h}_{t+1})$, \quad $z_{t+1} \sim p^{\theta}_{t+1}$ \COMMENT{prior}
                \STATE $m_{t+1} \gets [\hat{h}_{t+1}, z_{t+1}]$
                \STATE $\hat{o}_i^{t+1} \gets \text{observation\_decoder}(m_{t+1})$ for $i \in \{1, 2\}$ \COMMENT{decoder is generative here}
                \STATE $\hat{I}_i^{t+1} \gets \text{iset}(\hat{I}_i^t, a_i^t, \hat{o}_i^{t+1})$ for $i \in \{1, 2\}$
            \ENDFOR
            \STATE append the rollout to $\hat{B}$
        \ENDFOR
    \ENDIF
    \STATE \textit{(iv) Actor-critic --- one RNaD step on real and imagined data jointly}
    \STATE update $\pi_1, \pi_2, v$ by off-policy simultaneous-move RNaD on $B \cup \hat{B}$, with rewards regularized toward $\pi^{\text{reg}}$ as in Eq.~\ref{eq:rnad_reward}
    \IF{$n \bmod N^{\text{reg}} = 0$}
        \STATE $\pi_i^{\text{reg}} \gets \pi_i$ for $i \in \{1, 2\}$
    \ENDIF
\ENDFOR
\end{algorithmic}
\end{algorithm}

The centralized critic $v$ is evaluated on the joint representation $(\hat{I}_1^t, \hat{I}_2^t)$ and is used only during training; each actor $\pi_i$ conditions on $\hat{I}_i^t$ alone, so the policies remain executable from one player's action-observation history at test time. Note that phase (iii) generates one full rollout from the root per non-chance node of the real minibatch, terminal nodes included, so the number of imagined trajectories per gradient step grows with the length of the game rather than being a fixed multiple of $|B|$: Goofspiel 5 yields five imagined trajectories per real trajectory, Goofspiel 13 thirteen, and Battleship $5\times5$ twenty-eight.

\textbf{On the warm-up period.} The warm-up of phase (iii) is a stabilizing mechanism rather than a theoretical requirement, and nothing about it is specific to imperfect information. Policy-gradient actor-critics, including the REINFORCE actor of DreamerV3, are not designed for a non-stationary optimization target, so a warm-up would be equally defensible in a single-agent POMDP. What the imperfect-information setting adds is that the non-stationarity arrives from several directions at once, and each family of methods available here pays for it somewhere: through careful hyperparameter tuning (MMD, PPO-style methods), through a regularization mechanism that is itself a source of non-stationarity (the periodic $\pi^{\text{reg}}$ update of RNaD, Section~\ref{sec:rnad}), or by requiring explicit access to the true game/estimating the true distributions via importance sampling (CFR and its variants). Imagination from a world model that is still changing rapidly adds a further source on top of these, and the warm-up removes it from the phase of training where it is most harmful. Appendix~\ref{app:warmup} sweeps its length in Leduc Poker and finds the method not to be sensitive to it.

The warm-up should not be read as realizing assumption (i) of Corollary~\ref{cor:ne_convergence}, which requires a world model that has converged \emph{and is held fixed} during policy optimization. As Algorithm~\ref{alg:nashdreamer} makes explicit, the model is updated for the entire run; the warm-up narrows the gap to that assumption without closing it, and it is one of the approximations discussed in Section~\ref{sec:theory_ideal}.

\subsection{The Imagination Step}
\label{app:imagination_step}

Section~\ref{sec:imagination} summarizes the imagined step in prose. Stated in full, and starting from a model state $m_0$ obtained during world model training, each step proceeds as follows.
\begin{enumerate}
    \item Reconstruct current per-player legal actions $\hat{l}_i^{t}$, the reward $\hat{r}_t$ and the termination flag $\hat{d}_t$ from $m_t$
    \item Each actor samples an action $a_i^t$ from its policy conditioned on the player's latent infoset and legal actions: $a_i^t \sim \pi_i(\cdot \given \hat{I}_i^t, \hat{l}_i^t)$. The legal actions are used to explicitly mask out the policy output, rather than letting the actor learn the illegal actions.
    \item The sequence model advances the recurrent state: $\hat{h}_{t+1} = \text{seq}(m_t, (a_1^t, a_2^t))$.
    \item The dynamics predictor produces the prior distribution $p^{\theta}_{t+1} = \text{dyn}(\hat{h}_{t+1})$, from which a stochastic state $z_{t+1}$ is sampled, yielding the next model state $m_{t+1} = [\hat{h}_{t+1}, z_{t+1}]$.
    \item We obtain per-player observation reconstructions $\hat{o}_i^{t+1}$ from $m_{t+1}$ using the predictor.
    \item The infoset model updates each player's latent infoset: $\hat{I}_i^{t+1} = \text{iset}(\hat{I}_i^t, a_i^t, \hat{o}_i^{t+1})$.
\end{enumerate}
Step~5 is the key difference from standard Dreamer, where the decoder is merely a training signal. Here, the infoset model requires observation input (step~6), so the decoder becomes an essential generative component of the imagination pipeline.

\section{Contemporary approaches}
\label{app:contemporary}

So far, Dreamer family algorithms in multi-agent setting employ standard MARL techniques for non-stationarity mitigation ~\citep{nonstationaritysurvey} and decentralized models. We argue that these fail in adversarial settings. Explicit communication~\citep{mamba, dreamercf} is incompatible with zero-sum games, as rational players will learn to ignore it. Alternatively, population self-play or explicit opponent modeling~\citep{dreamerzerosum,dreamerhockey} suffers from the aforementioned policy-physics entanglement, which can have catastrophic results upon encountering a novel opponent. 

This is also why we do not report these systems as empirical baselines (Section~\ref{sec:experiments}): each is structurally inapplicable to 2p0s IIGs rather than merely inconvenient to run.
\begin{itemize}
    \item MAMBA~\citep{mamba} and DreamerCF~\citep{dreamercf} are cooperative methods built around inter-agent communication. In a zero-sum game, a communication channel carries no equilibrium value, as rational opponents ignore or exploit it, so porting these methods amounts to deleting their central mechanism.
    \item The opponent-embedding approach of~\citet{dreamerzerosum} conditions the world model on an embedding of the opponent's identity that must be randomly initialized at test time against an unseen opponent. There is then no fixed joint policy to evaluate, and NashConv, which requires one, is not well defined for such a model.
    \item Population self-play with single-agent models~\citep{dreamerhockey} instantiates exactly the policy-physics entanglement described above: what is learned is a joint-system model of the environment together with the historical opponent population, not a model of the game. Our decentralized MARSSM ablation is a controlled instance of this family, with the same optimizer and capacity and only the centralization removed, which we consider a fairer comparison than re-implementing any single such system together with its confounding design choices.
    \item LAMIR~\citep{lamir} extends MuZero-style search to IIGs, but assumes deterministic transitions, which excludes both Leduc and Goofspiel 13, and targets search-time planning rather than model learning for sample efficiency, so the comparison would be ill-posed along the axis we measure.
\end{itemize}

As a motivating example, consider the Battleship game. Since a decentralized model is, from a particular player point of view, effectively a POMDP marginalized over past opponents, the player might hallucinate the opponent's fleet placement based purely on historical correlations with the ego-agent's own layout. Observing a hit or miss that contradicts this entrenched hallucination breaks the internal belief representation, rendering all subsequent latent planning nonsensical.

\section{Architectural design}
\subsection{Mitigating Compounding Errors in High-Dimensional Spaces}
\label{app:compounding_errors}

As noted in Section \ref{sec:imagination}, our current MARSSM formulation reconstructs the explicit observation $\hat{o}_{i,t}$ at each imagination step to update the individual latent infosets $\hat{I}_{i,t}$. Because our empirical evaluations isolate structural learning by using compact infoset representation vectors, this explicit decoding is computationally inexpensive and highly interpretable. 

However, scaling this architecture to high-dimensional raw observations (such as images) presents a vulnerability. Standard Dreamer architectures transition strictly within the latent space to avoid the compounding auto-regressive errors and computational overhead of decoding raw pixels at every step of imagination. 

To resolve this when scaling MARSSM, the architecture can seamlessly adopt a latent bottleneck approach. Rather than decoding the raw observation directly, the centralized world model would predict a low-dimensional observation embedding $e_{i,t}$ (this approach is already employed by the encoder, before predicting the posterior latent stochastic state). The decentralized infoset networks would then aggregate these compact embeddings instead of raw observations. The actual high-dimensional observation reconstruction would be decoupled from the autoregressive imagination loop, acting solely as an auxiliary loss on $e_{i,t}$ during training. This two-stage approach preserves the causal information chain required for IIGs while maintaining the computational efficiency and robustness to compounding errors characteristic of the broader Dreamer family.

\subsection{Replay Buffer Design}
\label{app:replay}

Our replay buffer stores complete trajectories and unrolls sequences strictly from the initial state $t=0$. We store fixed-size trajectories and pad with invalid data if needed. This differs fundamentally from standard MBRL practice of sampling random mid-episode chunks. In IIGs, games often have rigid starting configurations, and uncertainty is generated iteratively by hidden opponent actions. Initializing a recurrent state at $t > 0$ severs the causal chain of strategic deductions, making it impossible for the infoset networks to reconstruct accurate beliefs. While this presents a scaling limitation for very long-horizon games, it ensures rigorous adherence to the game's information-theoretic properties.

While it is theoretically possible to bypass the $t=0$ requirement by storing the latent model states and infoset vectors directly into the buffer~\citep{dreamerv3}, doing so in an IIG introduces two fatal issues. First, due to representation drift, latent states stored in the buffer are generated by stale network parameters and lose their semantic meaning as the networks evolve. 

Second, and more fundamentally, because a player's policy conditions solely on their infoset, optimizing that policy requires the empirical distribution of the underlying histories within that infoset to perfectly match their true reach probabilities under the current joint policy. Sampling mid-episode histories from an off-policy replay buffer intrinsically violates this property, as the sampled histories reflect stale opponent strategies. Correcting this mismatch requires importance sampling or counterfactual unbiasing, which injects prohibitively high variance into the learning process.

How costly that correction is for a replay buffer of full trajectories, which is the design we do use, is game-dependent. The importance sampling ratio accumulates along a trajectory, so the variance grows with the length of the game, and with the change of policies from the moment a trajectory enters the buffer and the moment it is replayed (so, also with the replay buffer size). We consequently clip the counterfactual importance-sampling ratio (Table~\ref{tab:shared_hyperparameters}); in Battleship the gradients diverge without it. Empirically, the gains in Goofspiel 5 from replay are mostly seen early in the training, whereas in Leduc and in Phantom Tic-Tac-Toe it helps consistently, and in the two largest domains it is the weaker of the two RNaD baselines under both head-to-head play (Appendix~\ref{app:headtohead}) and approximate exploitability (Table~\ref{tab:approx_expl}). We attribute this failure to increased variance, as we also observed higher magnitude of gradients and the necessity of clipping.

\section{Environments}
\label{app:environments}

\textbf{Point Card Matching (PCM).} A single-agent environment inspired by Goofspiel. The agent holds $M$ cards ($1, \ldots, M$); point cards are revealed in descending order. Matching the point card scores a point. Reward is sparse (total score at termination). Features dynamic action masking and strict sequentiality.

\textbf{Imperfect Information Goofspiel.} A two-player simultaneous-move card game. Players simultaneously play cards; the higher card wins the revealed point card (ties discard it). We use an imperfect-information variant where players observe round outcomes (win/loss/tie) but not the opponent's specific card. We evaluate on 3-card, 5-card (descending point order), and 13-card (random point order) variants.

\textbf{Leduc Poker~\citep{leduc}.} A condensed Texas Hold'em with a 6-card deck (J, Q, K in two suits). Features private cards, a public community card, and two betting rounds. As this is a turn-based game, it is adapted to our framework via dummy actions. Trajectories reach up to 8 steps with multiple chance nodes.

\textbf{Battleship.} A two-player board game with placement and battle phases. Players simultaneously and secretly place ships on their grid (ships may not overlap or be adjacent), then simultaneously target opponent tiles each round, receiving hit/miss feedback. We evaluate on a  5$\times$5 board, with two size-2 ships and one size-3 ship for head-to-head comparison.

\textbf{Phantom Tic-Tac-Toe.} A turn-based, imperfect-information variant of Tic-Tac-Toe on a $3\times3$ board in which neither player observes the opponent's marks. A player attempting to claim a tile already held by the opponent is informed of the collision and \emph{replays} their move, repeating until they place on a free tile, so a turn can consist of several attempts. Following the OpenSpiel~\citep{openspiel} and exp-a-spiel~\citep{pgconvergence} implementations, we additionally forbid a player from attempting a tile they already know the opponent occupies. This keeps the game length bounded, and it is without loss of information under our perfect-recall assumption (Section~\ref{sec:posg}), since a player remembers every collision they have encountered. This is a deterministic turn-based game, with trajectories reaching up to 17 steps. It is also adapted via dummy actions.

\textbf{Perturbed Rock-Paper-Scissors.} The canonical one-shot matrix game, modified so that the rewards for playing Scissors are multiplied by two, which moves the equilibrium off the uniform strategy; the uniform policy has a NashConv of $2/3$. It is the smallest simultaneous-move game we use, and serves as an isolated testbed for the identifiability barriers of Section~\ref{sec:ctde} (Appendix~\ref{app:decentralized}).

\section{Implementation Details}
\label{app:implementation}

All models and environments were implemented in Python 3.12 using the JAX~\citep{jax} ecosystem and Flax~\citep{flax} 0.10.7 (NNX). To keep consistent with the DreamerV3 architecture \citep{dreamerv3}, we use a single LaProp optimizer~\citep{laprop} with a learning rate of $\alpha = 3 \times 10^{-4}$ for all network parameters, warmed up over the first $1{,}000$ gradient steps. Training utilizes a batch size of 64 complete trajectories per gradient step. All experiments were conducted on a cluster, utilizing 2 cores of AMD server processors, a single NVIDIA RTX A4000 GPU and up to 12 GB of system RAM (up to 30 GB for the PPO budgeted approximate NashConv across 10 stored model seeds), running on Debian 12 with CUDA 13.0. To upper-bound the complexity, our most time-demanding experiments (Battleship 5x5) required a maximum of 24 hours of total compute time to train for 30,000 minibatches across 10 seeds.
The estimated compute time for all the experiments (including preliminaries not reported in the paper) is around 4 weeks of CPU/GPU time.

\subsection{Network Architecture}

As our work focuses on algorithmic adaptations rather than novel network structures, we use the standard DreamerV3 network architecture throughout. Each predictor component is implemented as a Multi-Layer Perceptron (MLP) utilizing layer normalization~\citep{layernorm} and SiLU activations. The sequential network embeds input components separately, creates a skip connection with the recurrent state, concatenates the features, passes them through an MLP, and uses a GRU cell~\citep{gru} as the gate unit. The posterior encoder consists of two primary pathways: one embedding the joint observation, and one predicting the stochastic latent distribution conditioned on both the recurrent state and this observation embedding.

\subsection{Capacity Constraints}

\textbf{MLP Dimensions:} Every MLP predictor across the architecture consists of an input layer, a single hidden layer of 256 units, and an output layer (with layer normalization and SiLU applied after the input and hidden layers). The dimensions of the input and output layers are determined based on the specific input and target sizes of the respective environment.

To test the model's ability to capture dynamics under heavily constrained capacity, we strictly constrain the size of latent state representations.

\textbf{Deterministic States:} We constrain the dimensionalities of the deterministic latent vectors to correspond directly to the true information content of the environment. Specifically, the central recurrent state $\hat{h}_t$ matches the dimensionality of the true joint infoset representation, and each decentralized latent infoset $\hat{I}_{i,t}$ matches the true individual player's infoset dimensionality. The specific per-environment infoset representation sizes are:
\begin{itemize}
    \item \textbf{Point Card Matching with 3 cards (PCM 3):} 16
    \item \textbf{Goofspiel:} 35 (3 cards), 87 (5 cards), 535 (13 cards)
    \item \textbf{Leduc Poker:} 39
    \item \textbf{Phantom Tic-Tac-Toe:} 104
    \item \textbf{Battleship:} 716 (for a 5x5 board with 3 ships each)
\end{itemize}

\textbf{Stochastic Latent Space:} Rather than relying on the excessively large stochastic capacity of standard DreamerV3 (e.g., $32 \times 32$ classes), we deliberately restrict the categorical distribution to the minimal capacity required to uniquely identify chance outcomes without forcing the network to ``temporally multiplex'' representations. 

Temporal multiplexing occurs when the latent capacity is so constrained that the network must assign the exact same latent code $z_t$ to completely distinct physical chance outcomes, relying entirely on the recurrent state $\hat{h}_t$ to disambiguate them (e.g., class 1 representing a private card at step 1, but a public card at step 3). This places immense optimization strain on the dynamics MLP. 

To prevent this, we allocate exactly enough capacity to assign a distinct latent configuration to each unique class of chance node. For Goofspiel 13, we use a single distribution of 13 categories. While a single distribution of 34 categories should be sufficient for Leduc, we find that under our capacity constraints, the model fails to learn to perfectly model the underlying environment. So, our main results use a factored stochastic state of 2 categorical distributions with 6 categories. We present a comparison in Appendix \ref{app:leducerror}. We explicitly note that utilizing domain knowledge to size the latent space is an experimental choice designed to evaluate the framework under minimal capacity conditions. It is not a structural requirement of MARSSM. For environments with unknown or intractably large chance dynamics, the framework natively supports scaling back up to latent spaces with many distinct outcomes (e.g., $32 \times 32$ classes) typical of the Dreamer family.

In deterministic environments, the free-nats KL clipping threshold is set to 0, while in stochastic environments it is maintained at 0.2.

\subsection{Hyperparameters}
\label{app:hyperparameters}

Unless explicitly stated otherwise in the environment descriptions, all NashDreamer experiments share the hyperparameters listed in Table~\ref{tab:shared_hyperparameters}. The architecture is trained entirely online for $30,000$ gradient steps, where, at each step, a new minibatch is collected from the real environment, used to update the world model, which then generates a minibatch of imagined trajectories. Then, the actor-critic performs a step on the concatenation of the imagined minibatch with the real minibatch. For the significantly larger state spaces of Goofspiel 13, Battleship $5\times5$ and Phantom Tic-Tac-Toe, the initial world model warm-up phase is extended from the default $1,000$ to $2,000$ gradient steps to ensure stable structural representations before policy learning begins.

To ensure numerical stability during training, we apply several targeted regularizations following \citep{dreamerv3}. First, we add a $0.01$ uniform mixture to both the prior and posterior categorical distributions to prevent KL divergence spikes to infinity. Second, during imagination, we apply a state sampling threshold of $0.05$ (thresholding individual category probabilities and renormalizing before sampling the stochastic state $z_t$) to prevent the actor from exploring highly improbable hallucinated branches. The only exception to this threshold is Leduc under the flat stochastic space, where this threshold is set to 0.01 (as the initial chance outcomes have probability $\frac{1}{30} < 0.05$). Finally, the target predictors for terminal states and legal actions are thresholded at $0.5$ to yield binary masks during execution. For the off-policy RNaD optimization, we also clip the counterfactual importance sampling weights to prevent exploding gradients. We further adopt the DreamerV3 practice of zero-initializing the output layers of the reward and critic predictors, so that both heads predict exactly zero at the start of training; this keeps large, arbitrary early predictions from dominating the world-model loss and the imagined returns during the first iterations, and complements the warm-up period discussed in Appendix~\ref{app:algorithm}.

\begin{table}[h]
\centering
\caption{Shared NashDreamer hyperparameters used across all experiments.}
\label{tab:shared_hyperparameters}
\begin{tabular}{ll}
\toprule
\textbf{Hyperparameter} & \textbf{Value} \\
\midrule
\textit{General Training} & \\
Total gradient steps & $30,000$ \\
World model warm-up period & $1,000$ ($2,000$ for large envs) \\
Batch size & $64$ \\
\midrule
\textit{Optimizer (Shared)} & \\
Learning rate & $3 \times 10^{-4}$ \\
Learning rate warmup steps & $1,000$ \\
Betas ($\beta_1, \beta_2$) & $(0.9, 0.999)$ \\
Epsilon ($\epsilon$) & $1 \times 10^{-8}$ \\
Adaptive Gradient Clipping (AGC) & $0.3$ \\
\midrule
\textit{World Model Targets \& Regularization} & \\
Reward symlog bins (one-sided) & $20$ \\
Categorical uniform mixture & $0.01$ \\
State sampling threshold & $0.05$ \\
Terminal / Legal prediction threshold & $0.5$ \\
Dynamics KL scale ($\beta^{\text{dyn}}$) & $1.0$ \\
Encoder KL scale ($\beta^{\text{enc}}$) & $0.1$ \\
Prediction loss scales & $1.0$ \\
Infoset loss scale & $1.0$ \\
\midrule
\textit{Actor-Critic \& RNaD} & \\
Critic symlog bins (one-sided) & $20$ \\
RNaD entropy scale ($\eta$) & $0.2$ \\
V-trace parameters ($\bar{\rho}, \bar{c}, \gamma, \lambda$) & $(\infty, 1.0, 1.0, 1.0)$ \\
Target network EMA rate ($\tau$) & $0.001$ \\
Counterfactual IS clipping threshold & $100.0$ \\
NeuRD advantage absolute clip threshold & $10,000$ \\
NeuRD advantage relative clip threshold & $2.0$ \\
Imagination uniform policy mixture ($\epsilon$) & $0.2$ \\
Loss weighting (Real : Imagined) & $1.0 : 1.0$ \\
\bottomrule
\end{tabular}
\end{table}

\subsection{Standalone RNaD Baseline Configuration}
\label{app:rnad_baseline}

To ensure a rigorous and fair empirical comparison, the model-free standalone RNaD baseline utilizes the exact same Actor-Critic hyperparameters and optimizer configurations as NashDreamer (listed in Table~\ref{tab:shared_hyperparameters}). 

Because the standalone baseline relies on off-policy experience replay rather than online model-based imagination, we maintain a per-trajectory replay buffer with a capacity of $10,000$ trajectories. The baseline is trained using an environment-dependent replay ratio (the number of trajectories sampled from the buffer per newly collected online trajectory) to compensate for the additional steps seen in imagination by the model-based variant. For Goofspiel 13, this ratio is set to $26$, and for Battleship $5\times5$, it is set to $40$. 

We train the baseline for $30,000$ gradient steps, matching NashDreamer. Each gradient step entails the collection of a new minibatch from the environment, concatenation with replayed trajectories (if used) and performing the step on this larger minibatch.

\subsection{MMD Baseline Configuration}
\label{app:mmd_baseline}

The MMD baseline~\citep{mmd} shares the network architecture, optimizer and batch size of the RNaD baseline; the two differ only in the policy-optimization rule, with MMD replacing RNaD's policy-dependent reward regularization by magnet-policy regularization. Concretely, we follow the OpenSpiel PPO implementation, with an entropy exploration bonus and the additional inverse-KL penalty described by~\citet{pgconvergence}, and we adopt their ``generic'' hyperparameter setting. All network architectures and layer sizes are identical to RNaD (Appendix~\ref{app:implementation}). The MMD-specific values are listed in Table~\ref{tab:mmd_hyperparameters}.

\begin{table}[h]
\centering
\caption{MMD-specific hyperparameters, shared by the model-free MMD baseline and NashDreamer[MMD], and held fixed across all domains.}
\label{tab:mmd_hyperparameters}
\begin{tabular}{lll}
\toprule
\textbf{Hyperparameter} & \textbf{Value} & \textbf{Role} \\
\midrule
PPO epochs & $7$ & optimization epochs per batch \\
PPO clipping coefficient & $0.1$ & realizes the proximal term of Eq.~\ref{eq:mmd} \\
Inverse-KL coefficient & $0.05$ & loss multiplicand of the inverse KL penalty \\
Magnet coefficient ($\alpha$) & $0.2$ & fixed uniform magnet, i.e.\ an entropy bonus \\
Advantage-normalization $\epsilon$ & $1 \times 10^{-8}$ & numerical guard, not a tuned quantity \\
Discount / bootstrapping ($\gamma, \lambda$) & $(1.0, 1.0)$ & as for RNaD (Table~\ref{tab:shared_hyperparameters}) \\
\bottomrule
\end{tabular}
\end{table}

The magnet is fixed and uniform, so the magnet term contributes an entropy bonus rather than an anchor that moves with the policy; as noted in Section~\ref{sec:rnad}, this means MMD targets a quantal response equilibrium rather than an NE. Advantages are normalized in the standard PPO fashion over the valid (non-padded) entries of the batch,
$$ \tilde{A} = \frac{A - \mathbb{E}[A]}{\sqrt{\mathrm{Var}[A]} + \epsilon}, $$
with $\gamma$ and $\lambda$ set to $1.0$ to match the RNaD configuration rather than as independent choices.

We hold these values fixed across domains and do not tune them per environment. The primary reason is that tuning would not affect the claim under test: the model-free MMD baseline and NashDreamer[MMD] use the \emph{same} MMD hyperparameters, so whichever values are chosen, the comparison isolates the effect of the world model on a given optimizer. An untuned MMD is thus a conservative baseline for the modularity claim of Section~\ref{sec:actor_critic} and a fair one for the sample-efficiency comparison.  Additionally, the domains that motivate this work are those in which environment interaction is the dominant cost, which raises the value of an optimizer that performs acceptably without per-domain tuning; we note as a secondary point that tuning against a world model that is itself still changing is of questionable soundness. We acknowledge that per-domain tuning would likely improve MMD's absolute numbers, and we make no claim that the values used are optimal for it.

\section{Exploitability Evaluation Protocols}
\label{app:nashconv}

\subsection{Exact NashConv}

All NashConv values reported in Section~\ref{sec:experiments} are computed \emph{exactly} against the true game. They are not computed against the learned world model, and they are not estimated by sampling.

At evaluation time we first extract a behavioral policy over the true game. We traverse every infoset of the true game and query the trained actor network at it; for NashDreamer this means replaying that infoset's AOH $\tau_i$ through the infoset model to obtain the latent embedding $\hat{I}_i$ and reading off $\pi_i(\cdot \given \hat{I}_i)$, with the true legal-action mask applied. The evaluated object is therefore a well-defined joint policy over the true game's infosets, which is what makes NashConv meaningful as an equilibrium metric.

Given this joint policy, each player's best response is computed by dynamic programming over the full true game tree, following the OpenSpiel exploitability routine~\citep{openspiel}; for Phantom Tic-Tac-Toe we instead use the exp-a-spiel library of~\citet{pgconvergence}, whose implementation of the game matches the variant described in Appendix~\ref{app:environments}. Chance nodes are enumerated exactly with their true probabilities rather than sampled. Simultaneous-move nodes are expanded in the standard turn-based form, in which the second player's infoset does not contain the first player's action at that node, so the best-response computation respects the original information structure. The reported values are exact for the extracted policies; the only approximation in the pipeline is the extraction itself, i.e., the network's own numerical output at each infoset. NashConv curves report values per training seed, evaluated at fixed intervals of environment interaction.

\subsection{Best-Response Learner}
\label{app:approx_expl}

Section~\ref{sec:approx_expl} states the budgeted best-response protocol and reports its results; this section documents the learner that protocol runs. Taking the maximum over the ten PPO seeds rather than their mean makes the estimate the tightest lower bound the search budget supports, which is the conservative choice for a metric on which we claim an improvement.

The best-response learner is plain PPO: the same implementation and hyperparameters as the MMD baseline of Appendix~\ref{app:mmd_baseline}, with the inverse-KL penalty and the entropy exploration bonus removed. It is run identically in both games, for $3{,}000$ gradient steps with batch size $64$, $\gamma = 1.0$, $\lambda = 1.0$, clipping coefficient $0.1$ and $7$ inner epochs. The budget of $3{,}000$ steps is a tenth of that of the largest evaluated checkpoint, which suffices because computing a best response against a fixed opponent is a single-agent problem and materially easier than equilibrium finding.

A budget that is too small manifests as a best response failing to reach the value the evaluated profile achieves against itself, which indicates that the search fell short rather than that the checkpoint is unexploitable. We mark any such cell with $\dagger$. No cell of Table~\ref{tab:approx_expl} is affected, so the budget is inside the useful range everywhere the main text draws a conclusion from it.

\subsection{Aggregation over PPO Seeds}
\label{app:approx_expl_agg}

Table~\ref{tab:approx_expl_avg} reports the same experiment aggregated by averaging over the $10$ PPO seeds instead of taking the worst case. The values are uniformly smaller, as expected of a looser bound, and the least exploitable method is the same one under both aggregations at every checkpoint of both games. The full ranking is identical throughout Goofspiel 13; in Battleship a single adjacent pair exchanges places at each checkpoint, the two RNaD baselines at $10$k, MMD and NashDreamer[RNaD] at $20$k, and NashDreamer[RNaD] and model-free RNaD at $30$k. Each of those exchanges is between values lying well inside one another's intervals, and every claim made in Section~\ref{sec:approx_expl} holds under either aggregation. One cell is flagged under this looser bound, NashDreamer[RNaD] at $30$k in Goofspiel 13, on two of the ten training seeds.

\begin{table}[h]
\centering
\footnotesize
\setlength{\tabcolsep}{3pt}
\caption{Budgeted approximate exploitability averaged over the $10$ PPO seeds, rather than taking the worst case as in Table~\ref{tab:approx_expl}. Mean $\pm 2\sigma$ over the 10 training seeds. $^{\dagger}$ marks the one cell in which the best response failed to reach the profile's own value, on two of the ten training seeds, i.e.\ where the budget was too small.}
\label{tab:approx_expl_avg}
\begin{tabular}{lcccccc}
\toprule
 & \multicolumn{3}{c}{Goofspiel (13 cards)} & \multicolumn{3}{c}{Battleship (5$\times$5)} \\
\cmidrule(lr){2-4} \cmidrule(lr){5-7}
Method & 10k & 20k & 30k & 10k & 20k & 30k \\
\midrule
MMD & $0.789 \pm 0.076$ & $0.695 \pm 0.098$ & $0.684 \pm 0.067$ & $1.202 \pm 0.058$ & $0.986 \pm 0.100$ & $0.831 \pm 0.147$ \\
ND[MMD] & $0.805 \pm 0.069$ & $0.744 \pm 0.085$ & $0.708 \pm 0.082$ & $1.069 \pm 0.110$ & $0.781 \pm 0.099$ & $0.612 \pm 0.127$ \\
RNaD & $0.675 \pm 0.368$ & $0.381 \pm 0.258$ & $0.279 \pm 0.276$ & $1.312 \pm 0.067$ & $1.135 \pm 0.172$ & $0.937 \pm 0.163$ \\
RNaD w/ replay & $0.701 \pm 0.075$ & $0.476 \pm 0.073$ & $0.434 \pm 0.081$ & $1.306 \pm 0.057$ & $1.235 \pm 0.089$ & $1.007 \pm 0.152$ \\
ND[RNaD] & $0.408 \pm 0.133$ & $0.184 \pm 0.143$ & $0.093 \pm 0.134^{\dagger}$ & $0.979 \pm 0.105$ & $0.951 \pm 0.140$ & $0.912 \pm 0.235$ \\
\bottomrule
\end{tabular}
\end{table}

\section{Latent Error Trees}
\label{app:letrees}

To qualitatively assess the structural fidelity of the learned world model, we introduce Latent Error Trees (LE-Trees). Starting from the initial state, we recursively unroll the game tree following the agent's current joint policy, pruning any branch where a player's action probability falls below a threshold of $0.05$. 

In these visualizations, blue diamond nodes represent the ground-truth game states, while the adjacent rectangular nodes represent their corresponding latent reconstructions. These latent nodes are color-coded based on the maximum prediction error across all deterministic decoder heads (reward, continuation, legal actions, and individual infoset reconstruction): green indicates low error (error $< 0.2$), orange indicates moderate error ($0.2 \leq$ error $< 0.5$), and red indicates severe error (error $\geq 0.5$). Any non-green node is explicitly annotated with the specific predictor component responsible for the maximum error. Finally, all transitions (edges) are visually weighted according to their corresponding chance or policy probabilities.

The root latent states come from the posterior after the first observation, and every latent state below them comes from the prior. We keep every stochastic-state configuration that survives the state sampling threshold of $0.05$ (Appendix~\ref{app:hyperparameters}); at a chance node, we first assign each configuration to the outcome whose true observation it decodes closest to, in squared $L^2$ distance over both players' observations. For the Leduc Poker trees (Appendix~\ref{app:leducerror}), we fix the private deal and keep only the two most probable configurations at the root, after each player action and per outcome at each chance node. The round-2 trees are rooted at the public-card chance node reached by replaying the round-1 actions, following the most probable latent configuration at every step.

\section{Leduc Failure Modes and Latent Space Capacity}
\label{app:leducerror}

As described in Section~\ref{sec:experiments}, the world model trained with the vanilla hyperparameters fails to capture the dynamics of Leduc Poker already at the initial private deal: its posterior after the deal splits between two latent states, both with severe infoset reconstruction error, and the latent states below them stay mostly in moderate or severe error through round 2 (Figure~\ref{fig:letree_leduc_r1_vanilla}). The tuned variant matches the dynamics in both rounds (Figure~\ref{fig:letree_leduc_r1_tuned}; the LE-Trees for round 2 are contained in the supplementary material due to size constraints). The tuned variant's error of dealing two identical private cards (Section~\ref{sec:experiments}) does not appear in these trees, as their root is the posterior after a fixed deal, whereas the error lies in the prior over the deal (Appendix~\ref{app:calibration_leduc}). 

Our Leduc implementation utilizes binary observation and infoset tensors, alongside rewards divided by the maximal bet (13), effectively rescaling the reward signal into the $[-1, 1]$ range. This bounded structure meets the criteria of low prediction magnitude changes combined with high stochasticity, which inherently exacerbates the tension between posterior collapse and representation drift in the late stages of the game.

Figure~\ref{fig:leduc_categorical} compares four configurations of NashDreamer[RNaD] we ran on this domain: our tuned variant, a variant with the vanilla $\beta^{\text{enc}}=0.1, \beta^{\text{dyn}} = 1.0$, a one-way variant with $\beta^{\text{enc}}=0.0, \beta^{\text{dyn}} = 0.1$ and a flat variant with the tuned $\beta^{\text{enc}}=0.01, \beta^{\text{dyn}} = 0.1$, but replacing the factored latent with a minimal flat one, a single categorical with $34$ categories. The flat variant costs surprisingly little in NashConv, but the world model then fails to capture the environment's topological structure (Figure~\ref{fig:letree_leduc_r1_bad}), which Appendix~\ref{app:collapse_leduc} quantifies as posterior collapse at the initial deal.  The reason why NashConv survives is that the model learned to disambiguate Player~1's card at turn 1 and Player~2's card at turn 2, effectively splitting the chance node into two steps. Such an environment does not correspond to our environment structure of all dealt at the initial chance deal and constitutes a local optimum of the loss, but is strategically benign.

The one-way variant is the worst of the four in the later stages of training. It resolves every outcome of the initial deal, but at the public chance node its prior spreads over far more latent configurations than the node has outcomes (Appendices~\ref{app:collapse_leduc} and~\ref{app:calibration_leduc}), which we attribute to representation drift. The variant our theory endorses is therefore the one that can perform the worst in practice, which is the tension stated in Section~\ref{sec:theory_ideal}.

\begin{figure}[t]
    \centering
    \includegraphics[width=\linewidth]{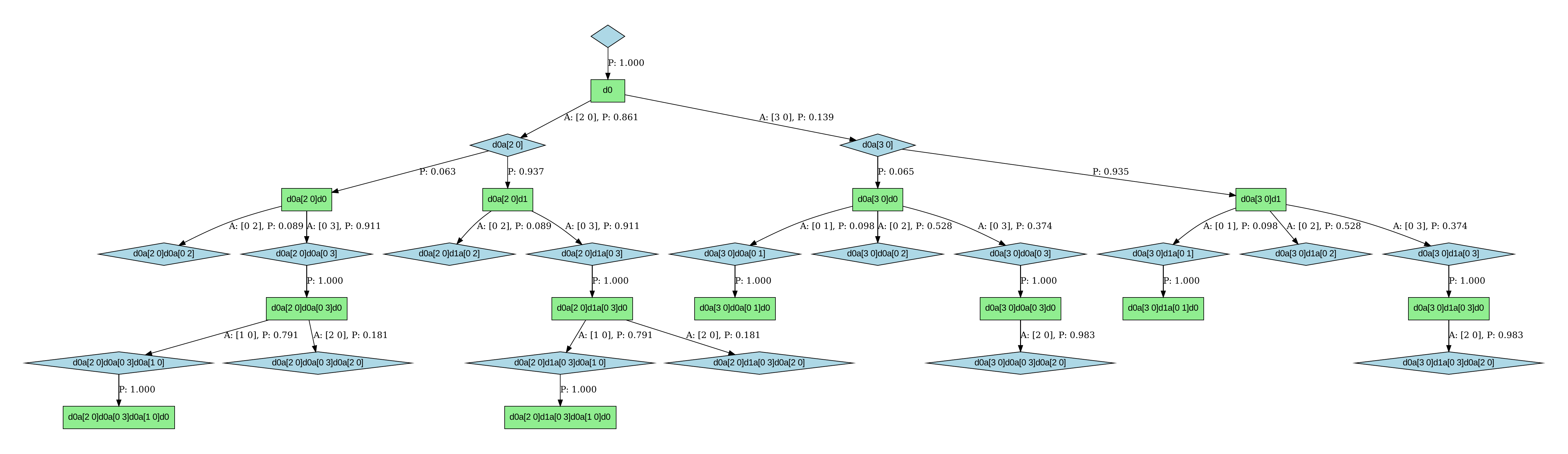}
    \caption{NashDreamer[RNaD] LE-Tree for round 1 of Leduc Poker with the tuned factored latent, for the deal of the jack to Player~1 and the queen to Player~2 (final checkpoint of a single seed). The model accurately learns the underlying game dynamics.}
    \label{fig:letree_leduc_r1_tuned}
\end{figure}

\begin{figure}[t]
    \centering
    \includegraphics[width=\linewidth]{figures/leduc_r1_vanilla_tree.pdf}
    \caption{The LE-Tree of Figure~\ref{fig:letree_leduc_r1_tuned} for the factored latent under the vanilla hyperparameters. The posterior after the private deal splits between two latent states, both with severe infoset reconstruction error.}
    \label{fig:letree_leduc_r1_vanilla}
\end{figure}

\begin{figure}[t]
    \centering
    \begin{minipage}{0.48\textwidth}
        \centering
        \includegraphics[width=\linewidth]{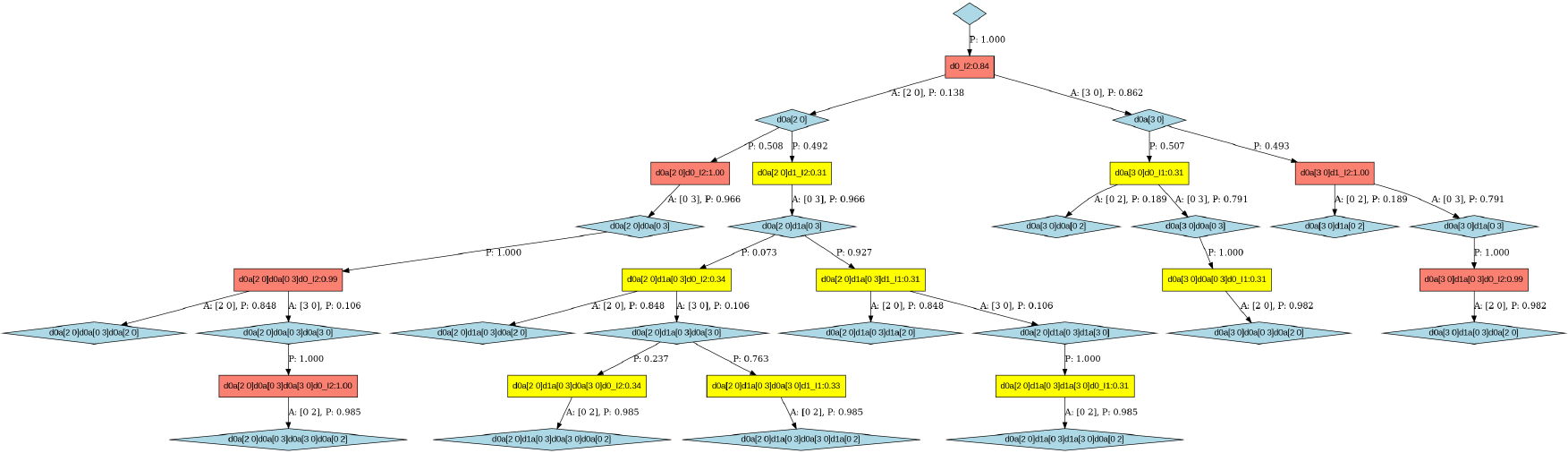}
        \caption{NashDreamer LE-Tree for round 1 of Leduc Poker under a flat latent space ($1 \times 34$). The single categorical fails to learn the true environment structure.}
        \label{fig:letree_leduc_r1_bad}
    \end{minipage}
\end{figure}

\begin{figure}[h]
        \centering
        \includegraphics[width=0.55\linewidth]{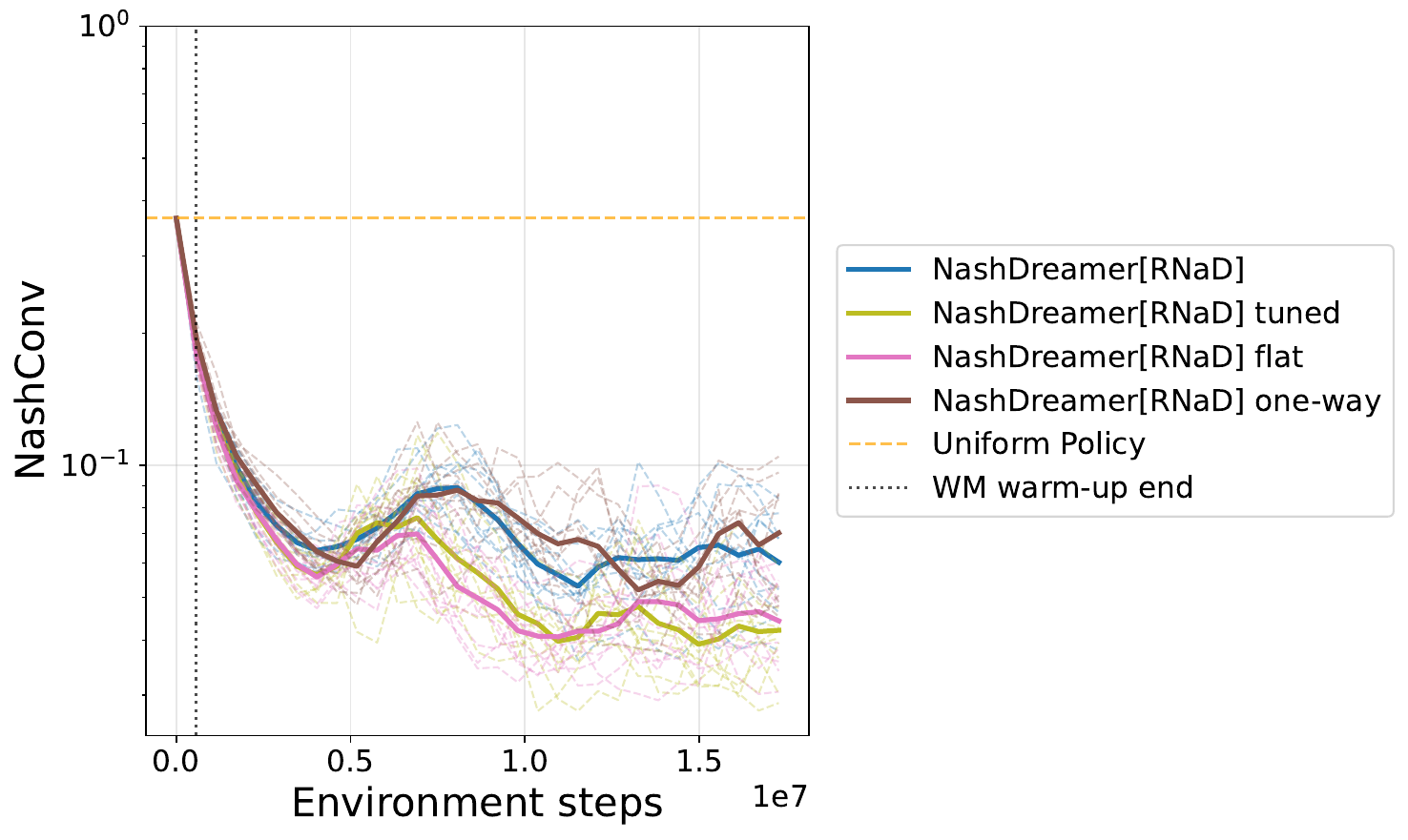}
        \caption{NashConv in Leduc Poker for the latent and loss configurations of NashDreamer[RNaD]: the factored latent in its tuned and vanilla settings (the latter without a suffix), the flat single-categorical latent, and the one-way variant with $\beta^{\text{enc}} = 0$. Bold curves are means over the $10$ training seeds, faint curves the individual seeds; the dotted line marks the end of the world-model warm-up.}
        \label{fig:leduc_categorical}
\end{figure}

\section{Quantitative World-Model Diagnostics}
\label{app:diagnostics}

The Latent Error Trees of Appendix~\ref{app:letrees} are qualitative. To attribute the policy results of Section~\ref{sec:experiments} to world-model quality rather than to actor-critic dynamics alone, we measure the learned model directly under three protocols, defined in Appendices~\ref{app:rollout}, \ref{app:collapse_protocol} and~\ref{app:calibration_protocol}: \emph{rollout validity}, the per-step emission error of a trajectory generated entirely by the model; the \emph{posterior-collapse indicator}, the expected $\KL(q^{\theta} \parallel p^{\theta})$ at chance nodes; and \emph{chance-outcome calibration}, the agreement between the prior's distribution over successors and the true one. Results are reported in Appendices~\ref{app:collapse_leduc} to~\ref{app:rollout_goofspiel}.

All three protocols filter the model's categorical stochastic state before use, applying the thresholding used during imagination (Appendix~\ref{app:hyperparameters}), so that the model is scored on the states it generates at run time. The filter acts on each categorical distribution independently: within a distribution, categories below the threshold are set to zero and the remainder renormalized, so a factored state whose two categoricals read $[0.95, 0.01, 0.04]$ and $[0.03, 0.01, 0.96]$ becomes $[1, 0, 0] \otimes [0, 0, 1]$ at a threshold of $0.05$. The threshold must lie below the smallest outcome probability the game presents, or real outcomes are removed before they can be scored; we use $0.05$ for the factored variants and $0.01$ for the flat one, whose single categorical has to resolve the $1/30$ of Leduc's initial deal.

In \textbf{Goofspiel 5}, which has no chance nodes, every emission is eventually converging, with only the reward exhibiting larger instability: terminal flags and legal actions within $1{,}000$ training steps, observations about an order of magnitude later, and the reward last of all, sometimes spiking. Appendix~\ref{app:rollout_goofspiel} reports the four curves and the ordering in full.

In \textbf{Leduc Poker}, the picture is qualitatively different and matches the deterioration visible in Figure~\ref{fig:nashconv}. Terminal flags and legal actions are again identified without error within $1{,}000$ steps, and deterministic transitions are reconstructed exactly by our tuned variant and the one-way variant, but the observation and reward errors never reach zero and the chance transitions stay miscalibrated for the whole of training. Because the four tested model variants in Leduc fail in visibly different ways, we report each diagnostic separately for all four rather than summarizing them here: Appendix~\ref{app:collapse_leduc} for the posterior-collapse indicator, Appendix~\ref{app:calibration_leduc} for chance-outcome calibration and Appendix~\ref{app:rollout_leduc} for the prediction accuracies under rollouts.

\subsection{Rollout Validity}
\label{app:rollout}

This protocol measures whether a trajectory generated entirely by the model corresponds to a realizable trajectory of the true game. Unlike a single-step emission score, errors compound along the rollout, as they do when the actor-critic trains on imagined data.

The model rollout and the true game are advanced with the same joint actions, sampled from the model. The recurrent state is initialized at $h_0$ and the latent infoset embeddings at their initial values; the initial stochastic state is sampled from the prior $p^{\theta}$. Each step proceeds as follows.

\begin{enumerate}
    \item At a chance node of the true game, the outcome is not sampled from the true chance distribution but chosen to minimize the $L^{\infty}$ distance between the true joint observation and the one reconstructed from the current model state. Divergence later in the rollout is then attributable to the learned dynamics rather than to the two trajectories having drawn different chance outcomes.
    \item Each player's action is sampled as $a_i^t \sim \pi_i(\cdot \given \hat{I}_i^t)$ from the latent infoset embedding, or from the infoset reconstructed from the model state when the model was trained on observations alone.
    \item $(a_1^t, a_2^t)$ advances the recurrent state through the centralized sequence model and advances the true game; $z_{t+1}$ is then sampled from the prior at the new recurrent state.
    \item The observations reconstructed from $m_{t+1}$ update each $\hat{I}_i^{t+1}$.
\end{enumerate}

The rollout terminates when the true game reaches a terminal state.

\textbf{Error statistics.} At each step the model's emissions are compared with the true ones and the discrepancy accumulated: $L^1$ for the reward, $L^{\infty}$ for the joint observation, and $1$ per boolean mismatch for the terminal flag and the legal-action mask. Each statistic is the accumulated magnitude divided by the number of valid steps, so all four are mean per-step magnitudes, on different scales.

In our case, just like in the main body experiments, we used infoset representation provided by the environment as the observation.

\subsection{Posterior Collapse at Chance Nodes}
\label{app:collapse_protocol}

By Theorem~\ref{thm:isomorphic_optimal}, an infoset-isomorphic model attains an expected $\mathbb{E}_{o}\KL(q^{\theta} \parallel p^{\theta})$ at a chance node equal to the entropy of that node, whereas posterior collapse drives the same quantity to zero.

The model is conditioned on the true action sequence (teacher forcing). At each non-chance node the stochastic state is taken from the posterior, selecting the outcome whose reconstructed joint observation is closest to the true one in $L^{\infty}$ norm, which keeps the model state aligned with the true history. At a chance node we enumerate the outcomes, evaluate the posterior $q^{\theta}(\cdot \given o)$ induced by each, and record their expected divergence from the single prior $p^{\theta}$ at that node,
\begin{equation}
    \mathcal{C} = \sum_{o} p(o) \, \KL\!\left(q^{\theta}(\cdot \given o) \;\middle\|\; p^{\theta}\right).
    \label{eq:collapse}
\end{equation}
Chance nodes are grouped by their number of outcomes $K$, and $\mathcal{C}$ is averaged, unweighted, over all nodes with the same $K$. Every chance node in our environments is uniform, so an isomorphic model attains $\mathcal{C} = \log K$.

Deviation from $\log K$ has two distinct causes. $\mathcal{C} < \log K$ indicates posterior collapse: the posterior depends only weakly on the chance outcome. $\mathcal{C} > \log K$ indicates that the prior assigns less mass to the realized posterior mode than a calibrated prior over $K$ equiprobable outcomes would, i.e.\ that it is spread over more latent configurations than the node has outcomes. A modest excess is expected when the latent space is larger than $K$ and distinct histories reaching the same outcome occupy distinct configurations; a large excess indicates that the prior does not track the posterior.

\subsection{Chance-Outcome Calibration}
\label{app:calibration_protocol}

This protocol measures the prior's distribution over successors: what fraction of its mass corresponds to realizable outcomes, and whether those outcomes carry the correct probabilities.

Teacher forcing is applied as in Appendix~\ref{app:collapse_protocol}, but the stochastic states are drawn from the prior $p^{\theta}$. At each node the model's stochastic states are enumerated and assigned to true successors: a state is assigned to the successor whose joint observation is nearest its reconstruction in $L^{\infty}$ norm, or to a separate \emph{unmatched} bucket if that distance exceeds $0.3$ for every successor. Summing model probabilities within each bucket gives a marginal over the true successors, plus the unmatched mass. Deterministic nodes are included, their true successor distribution being a Dirac.

Two statistics are reported: the $L^1$ distance between the bucketed, unnormalized marginal (so the model marginal need not sum up to 1) and the true successor distribution, bounded by $2$, and the unmatched mass, bounded by $1$. Both are averaged over all nodes of a kind, with chance and deterministic nodes reported separately, and are not grouped by $K$. The two isolate different failures: unmatched mass measures prior mass on latent states with no corresponding true successor, whereas a large $L^1$ distance with little unmatched mass indicates successors that are represented but mis-weighted.

We note that if no outcome is upweighted by the model, the $L^1$ distance and the unmatched mass coincide: writing $p^{\theta}(o)$ for the bucketed model marginal, every term $|p(o) - p^{\theta}(o)|$ then equals $p(o) - p^{\theta}(o)$, so the $L^1$ distance is $\sum_o p(o) - \sum_o p^{\theta}(o) = 1 - \sum_o p^{\theta}(o)$, which is the unmatched mass. Consequently, at deterministic nodes the two metrics always coincide, and so we report only the unmatched mass.

\subsection{Posterior Collapse in Leduc Poker}
\label{app:collapse_leduc}

Leduc has chance nodes of two sizes: the initial private deal, with $K = 30$ outcomes, and the second-round public card, with $K = 4$. We measure the statistic of Appendix~\ref{app:collapse_protocol} at both over $10$ seeds and $30{,}000$ gradient steps, for the four configurations the paper uses: the \emph{flat} latent, a single categorical with $34$ categories (Appendix~\ref{app:leducerror}); the \emph{factored} latent, with $36$ configurations, under both the tuned and the vanilla hyperparameters; and the \emph{one-way} variant, which keeps the factored latent but sets $\beta^{\text{enc}} = 0$.

\textbf{At the initial node ($K = 30$) the flat latent and vanilla hyperparameters cannot represent all outcomes} (Figure~\ref{fig:collapse_k30}). They settle near 1.5, less than half of the $\log 30 = 3.401$ baseline. The factored latent recovers most of the gap at $2.9$, leaving a residual collapse, and the one-way variant sits just above the baseline at $3.45$, which follows from its $36$ configurations: a prior spread over all of them exceeds $\log 30$. 

\textbf{At the public node ($K = 4$) the ordering inverts} (Figure~\ref{fig:collapse_k4}). The flat latent sits on the $\log 4 = 1.386$ baseline, the factored latent variants just above at $1.45$, and the one-way variant far above at $2.4$, with a visibly wider spread across seeds. Its prior is therefore spread over many more configurations than the node has outcomes, which we attribute to the prior being unable to track the moving posterior.

\begin{figure}[h]
    \centering
    \begin{minipage}{0.48\textwidth}
        \centering
        \includegraphics[width=\linewidth]{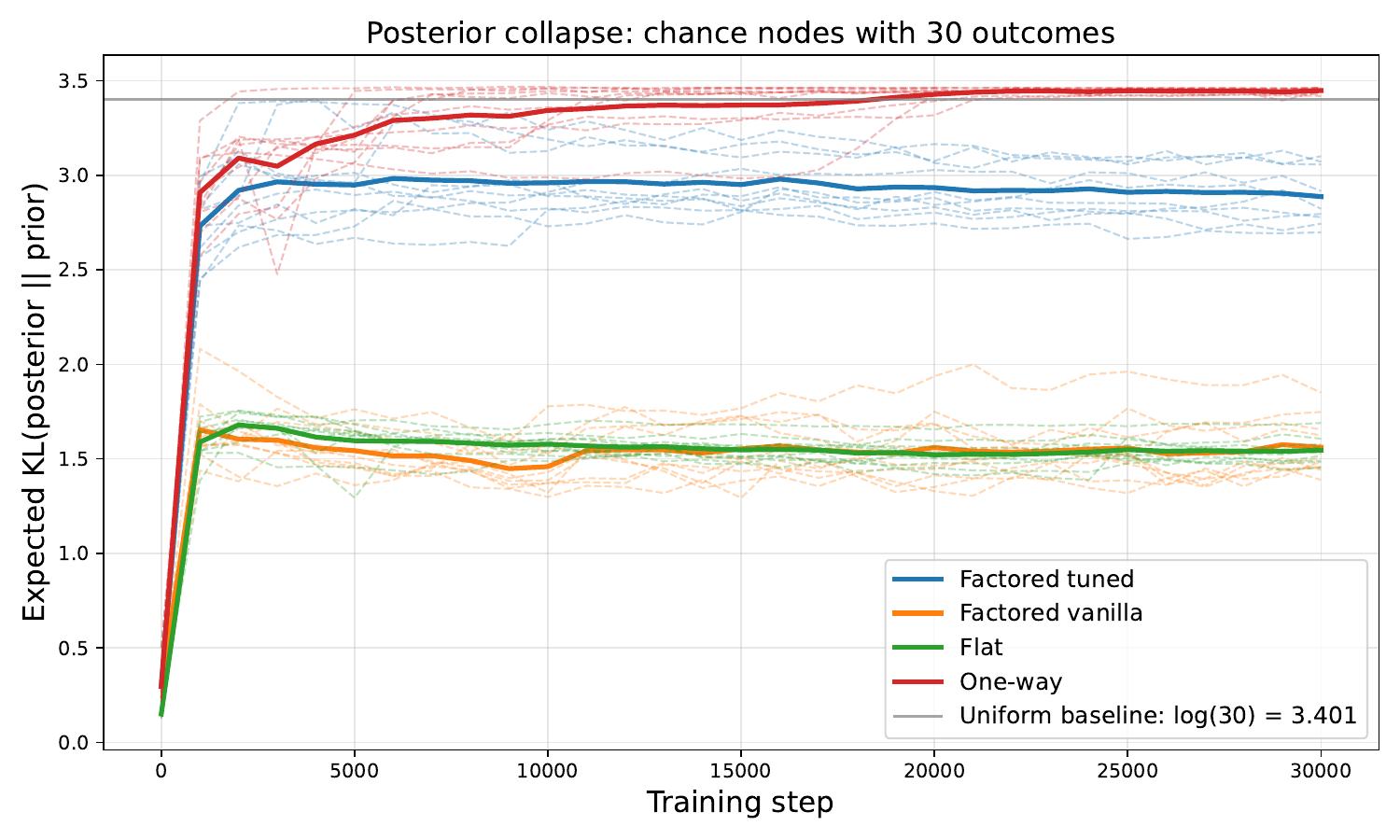}
        \caption{Expected $\KL(q^{\theta} \parallel p^{\theta})$ at the $K = 30$ private-deal chance node of Leduc Poker, averaged over nodes of that size. Bold curves are means over $10$ seeds, faint curves the individual seeds. The grey line is the uniform baseline $\log 30$.}
        \label{fig:collapse_k30}
    \end{minipage}\hfill
    \begin{minipage}{0.48\textwidth}
        \centering
        \includegraphics[width=\linewidth]{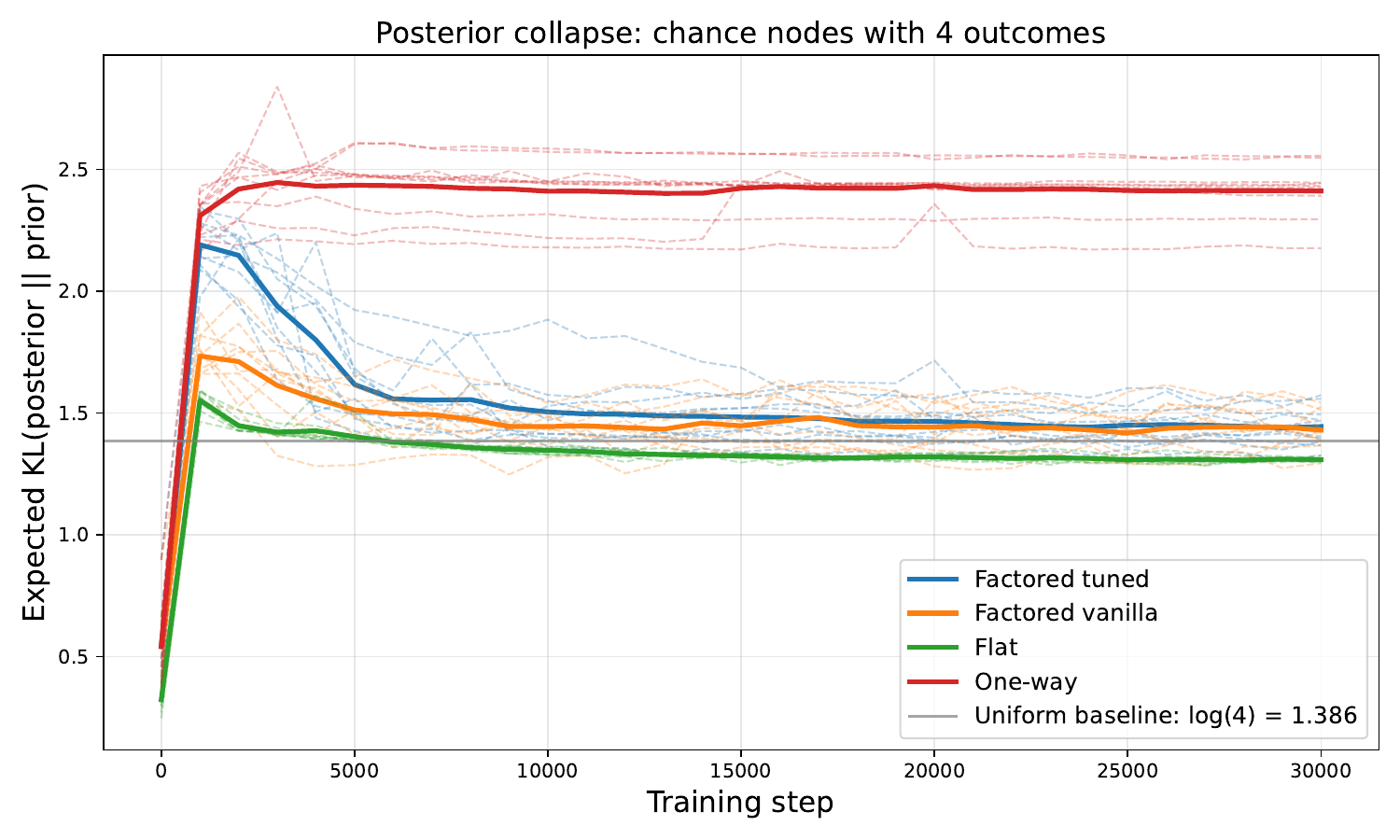}
        \caption{The same statistic at the $K = 4$ public-card chance nodes of Leduc Poker. The grey line is the uniform baseline $\log 4$.}
        \label{fig:collapse_k4}
    \end{minipage}
\end{figure}

\subsection{Chance-Outcome Calibration in Leduc Poker}
\label{app:calibration_leduc}

At deterministic nodes the factored tuned and one-way variants drive both statistics to zero within roughly $7{,}000$ steps, while the flat latent settles near $0.03$ on both and the vanilla hyperparameters are not able to match the deterministic transitions, ending at error near 0.1 (Figure~\ref{fig:det_mag}). At chance nodes the separation is even more visible: the factored variant reaches an $L^1$ distance near $0.1$ with $0.07$ of its mass unmatched. For the vanilla hyperparameters and one-way variant, no real outcome is upweighted and hence the unmatched mass and the $L^1$ distance coincide. The one-way variant ends around 0.63 and the vanilla hyperparameters around 0.43 (Figures~\ref{fig:chance_l1} and~\ref{fig:chance_mag}).

Put together with the posterior collapse statistic, where at $K = 30$ the one-way posterior still resolves every deal whereas the vanilla one sits at about half of $\log K$ (Appendix~\ref{app:collapse_leduc}), we conclude that the one-way variant represents every outcome but spreads the prior too wide, diverting two thirds of its prior mass to configurations the game cannot produce. 

The vanilla hyperparameters suffer from a similar phenomenon as the flat latent, of spreading the singular initial chance node into two events. Unlike the flat latent it is also the only configuration that fails to resolve a deterministic successor. A closer study reveals that this variant fails to disambiguate outcomes and is often uncertain about which card was dealt (rather than outputting a value close to 1 for one element in the one-hot representation, the output was $\geq 0.3$ at the spots of multiple card bits). 

The tuned variant is the only one whose $L^1$ distance exceeds its unmatched mass, so beyond a small leakage it misallocates probability among the outcomes it does represent. The small error in probability mass in this case comes from a variational gap introduced by the factoring itself, as the DreamerV3-style loss treats the categorical distributions as independent. In this specific case, that allows the prior to encode outcomes that duplicate the same private card. It is nevertheless the best configuration at chance nodes and exact at deterministic ones.

\begin{figure}[h]
    \centering
    \begin{minipage}{0.48\textwidth}
        \centering
        \includegraphics[width=\linewidth]{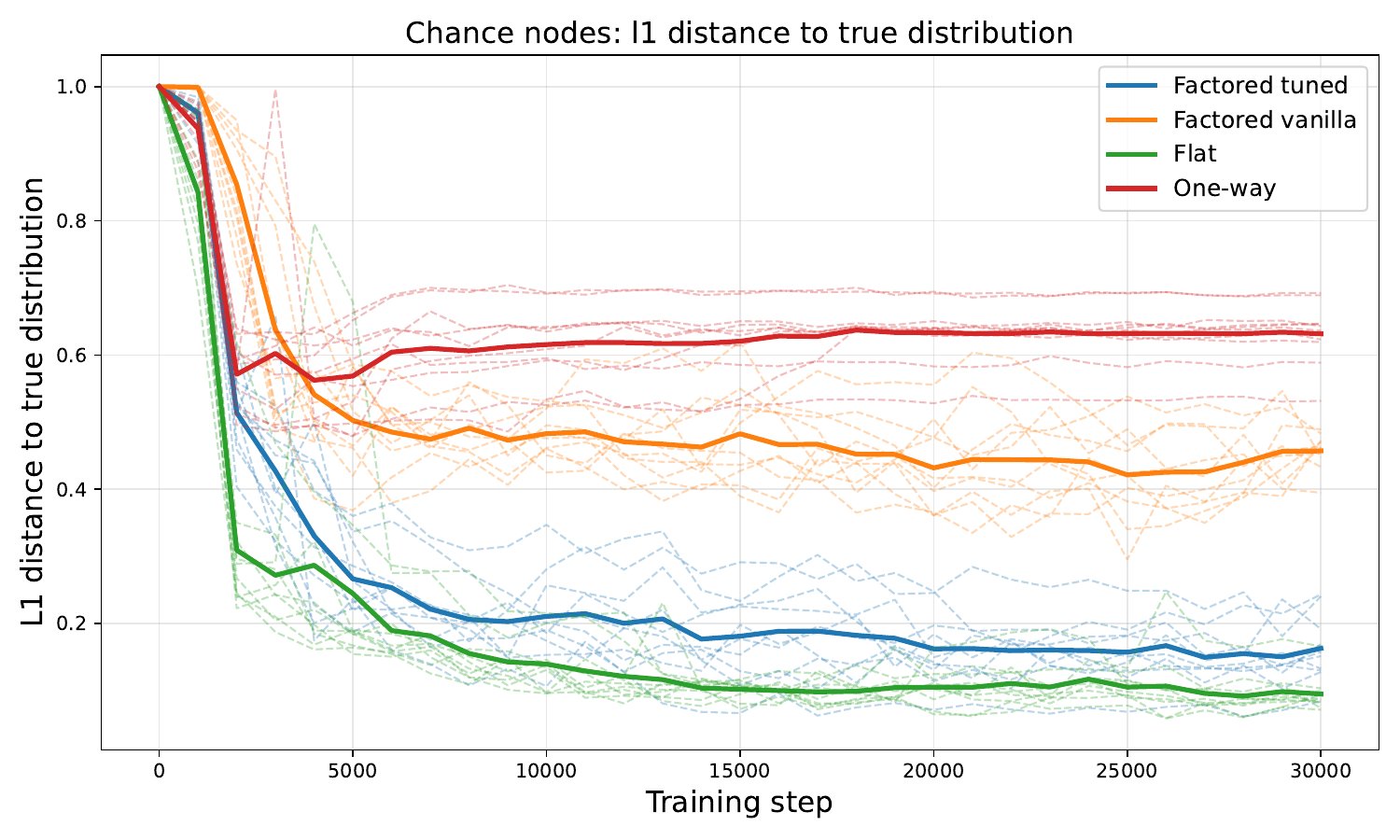}
        \caption{$L^1$ distance between the bucketed marginal of the prior and the true outcome distribution at the chance nodes of Leduc Poker. Bold curves are means over $10$ seeds, faint curves the individual seeds.}
        \label{fig:chance_l1}
    \end{minipage}\hfill
    \begin{minipage}{0.48\textwidth}
        \centering
        \includegraphics[width=\linewidth]{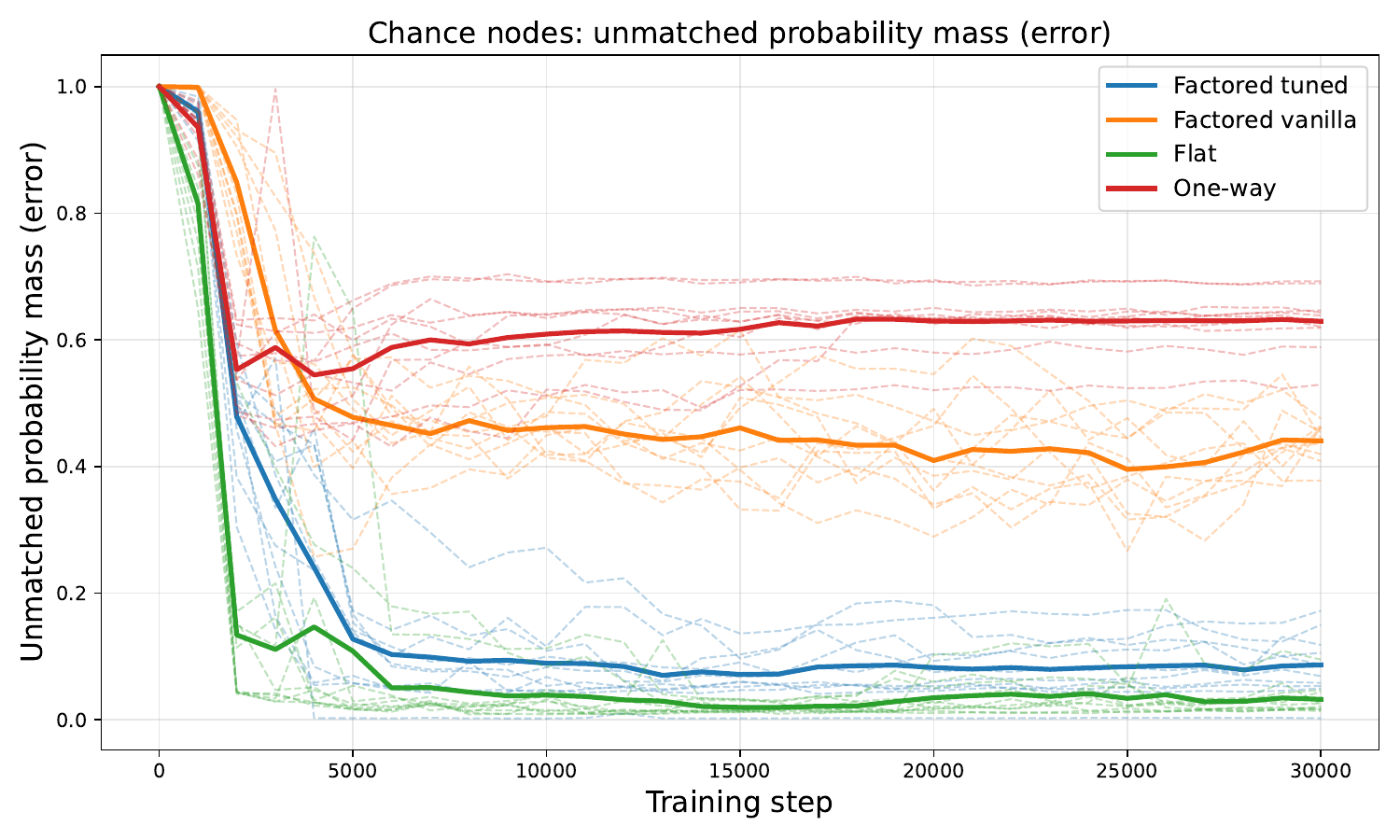}
        \caption{Unmatched probability mass at the chance nodes of Leduc Poker: the share of the prior that decodes to no real outcome. Bold curves are means over $10$ seeds, faint curves the individual seeds.}
        \label{fig:chance_mag}
    \end{minipage}
\end{figure}

\begin{figure}[h]
 \centering   
        \includegraphics[width=0.5\linewidth]{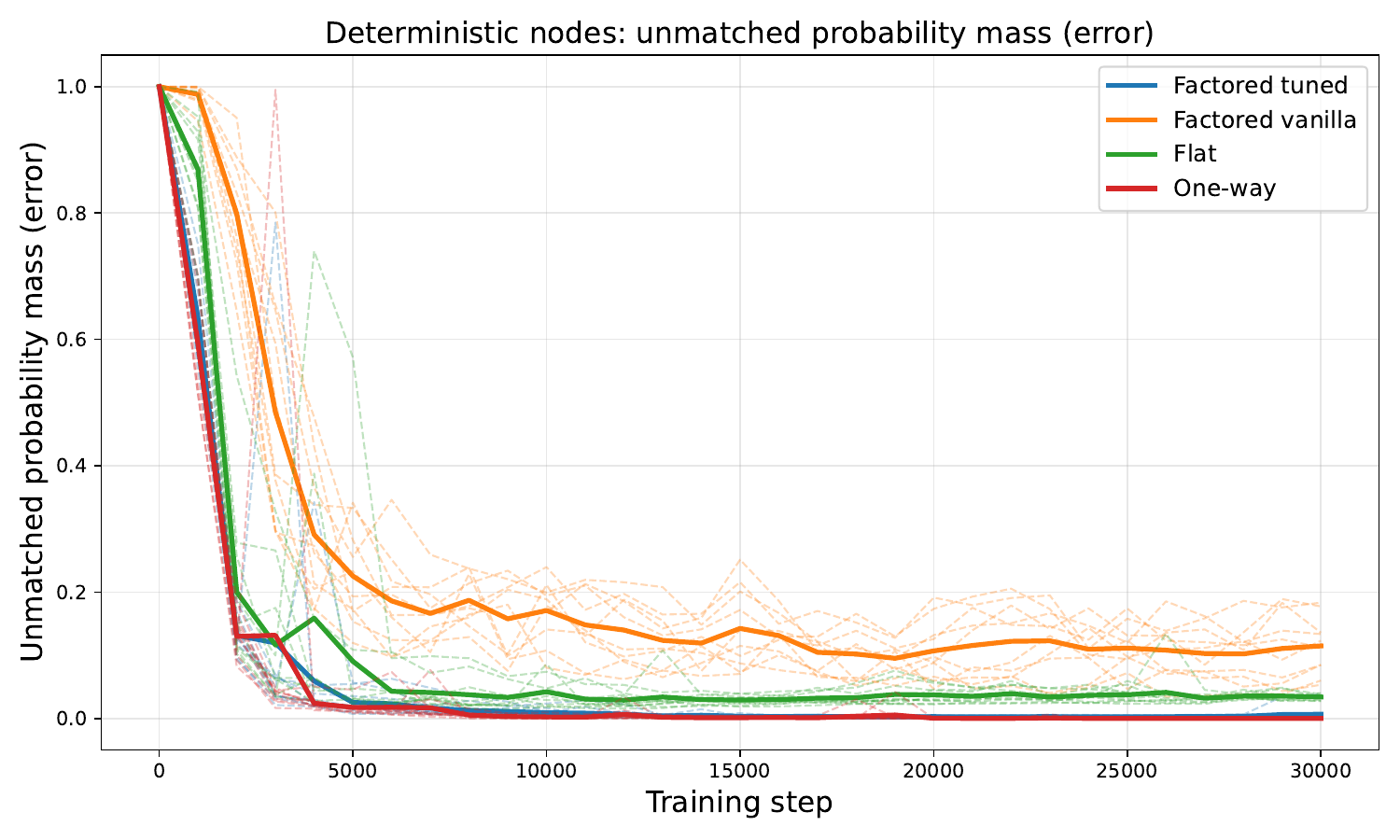}
        \caption{Unmatched probability mass at the deterministic nodes of Leduc Poker. Bold curves are means over $10$ seeds, faint curves the individual seeds.}
        \label{fig:det_mag}
\end{figure}

\subsection{Rollout Validity in Leduc Poker}
\label{app:rollout_leduc}

Rollout validity (Appendix~\ref{app:rollout}) removes the teacher forcing of the two statistics above, so errors compound as they do during imagination. We report it for the same four configurations over $10$ seeds and $30{,}000$ gradient steps. Leduc rewards are divided by the maximal bet of $13$ into $[-1, 1]$ (Appendix~\ref{app:leducerror}), so reward errors are on that normalized scale.

\textbf{The deterministic emissions survive the rollout} (Figures~\ref{fig:leduc_roll_terminal} and~\ref{fig:leduc_roll_legal}): terminal flags and legal-action masks reach zero within roughly $1{,}000$ gradient steps for all four configurations and stay there. The Leduc failure is not a loss of track of when the game ends or which actions are available.

\textbf{Neither observations nor rewards reach zero} (Figures~\ref{fig:leduc_roll_obs} and~\ref{fig:leduc_roll_reward}), and they disagree on which configuration is best. On observations the one-way variant leads at $0.36$, ahead of the factored at $0.54$ and the flat/vanilla at $0.83$, all down from an initial $4.1$. We attribute the near identical performance of flat and vanilla variants to failing to identify the initial private card for Player~2. On rewards, the first two reverse, at $0.024$ and $0.029$, with the vanilla variant again the worst at $0.037$. The vanilla variant is the only one worst on both, consistent with its failure to resolve even deterministic successors (Appendix~\ref{app:calibration_leduc}).

That the one-way variant leads here despite leaving $0.63$ of its chance-node prior mass unmatched is an effect of measuring the $L^\infty$ norm, as the effect is most destructive at the public-card chance node, where the mass is spread out and produces uncertainty over the card bits rather than sharp errors.

That the tuned variant nonetheless trails it on observations, despite being far better calibrated at chance nodes, reflects two different notions of a latent state supporting an outcome. Bucketing (Appendix~\ref{app:calibration_protocol}) credits an outcome whenever \emph{some} configuration in the prior's support decodes close to it, whereas a rollout uses the configuration actually sampled. The size of the gap suggests the factored variant's best supporter for an outcome is considerably better for the rest of the rollout than its typical sampled one and the errors compound.

\begin{figure}[h]
    \centering
    \begin{minipage}{0.48\textwidth}
        \centering
        \includegraphics[width=\linewidth]{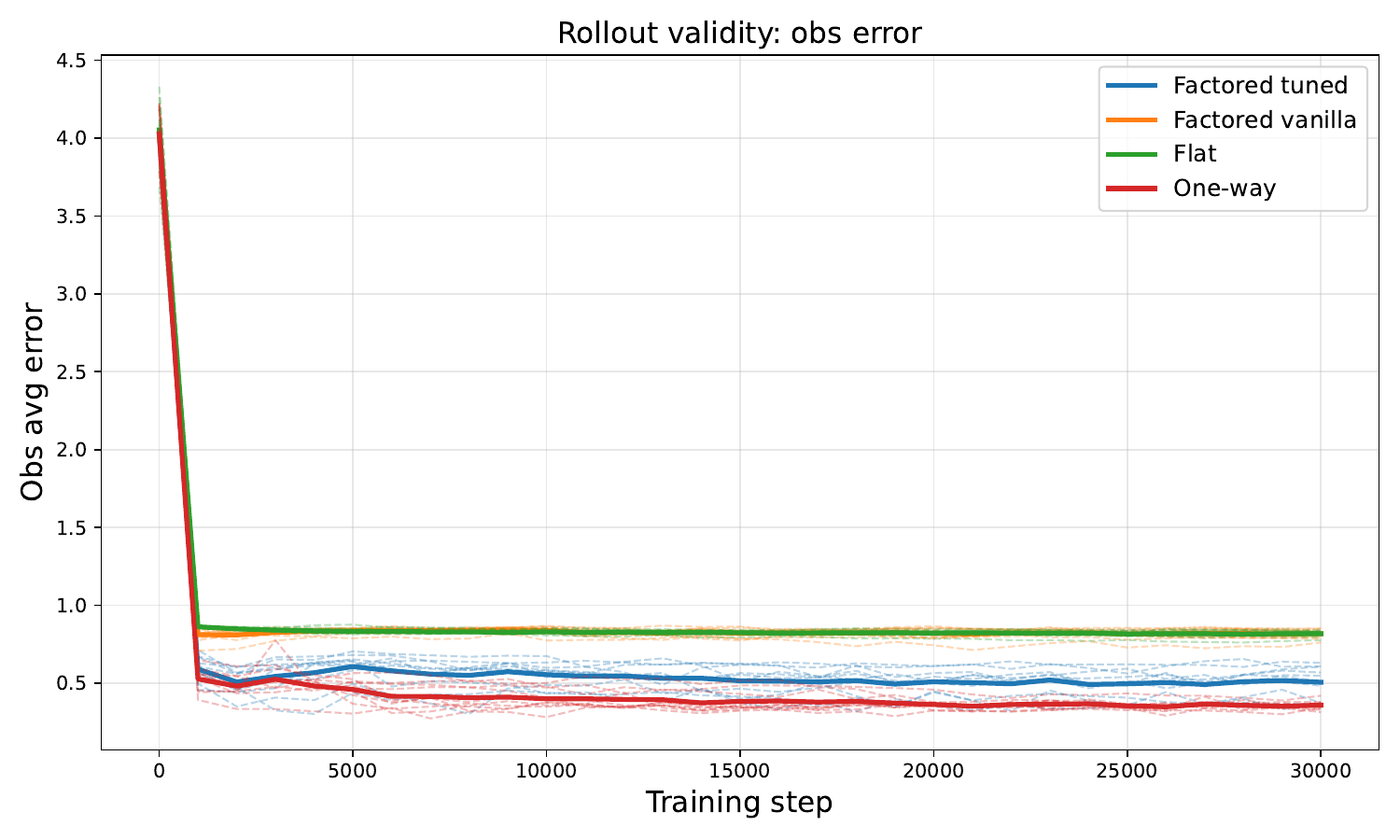}
        \caption{Rollout observation error in Leduc Poker, as the mean per-step $L^{\infty}$ error of the joint observation (in our case, infoset representation). Bold curves are means over $10$ seeds, faint curves the individual seeds. No configuration reaches zero.}
        \label{fig:leduc_roll_obs}
    \end{minipage}\hfill
    \begin{minipage}{0.48\textwidth}
        \centering
        \includegraphics[width=\linewidth]{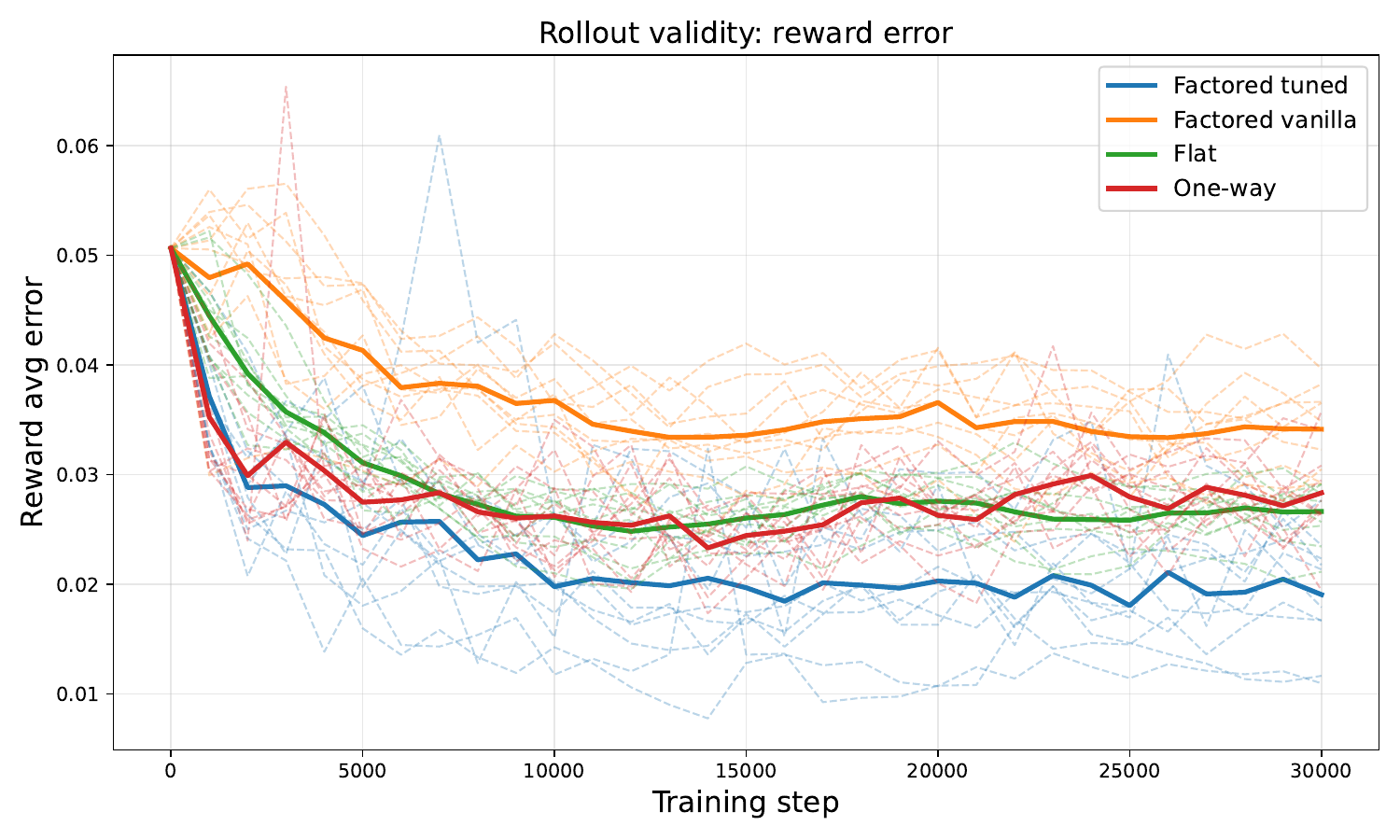}
        \caption{Rollout reward error in Leduc Poker, as the mean per-step $L^1$ error on rewards normalized to $[-1, 1]$. The ordering of the factored and one-way variants is the reverse of the one on observations.}
        \label{fig:leduc_roll_reward}
    \end{minipage}
\end{figure}

\begin{figure}[h]
    \centering
    \begin{minipage}{0.48\textwidth}
        \centering
        \includegraphics[width=\linewidth]{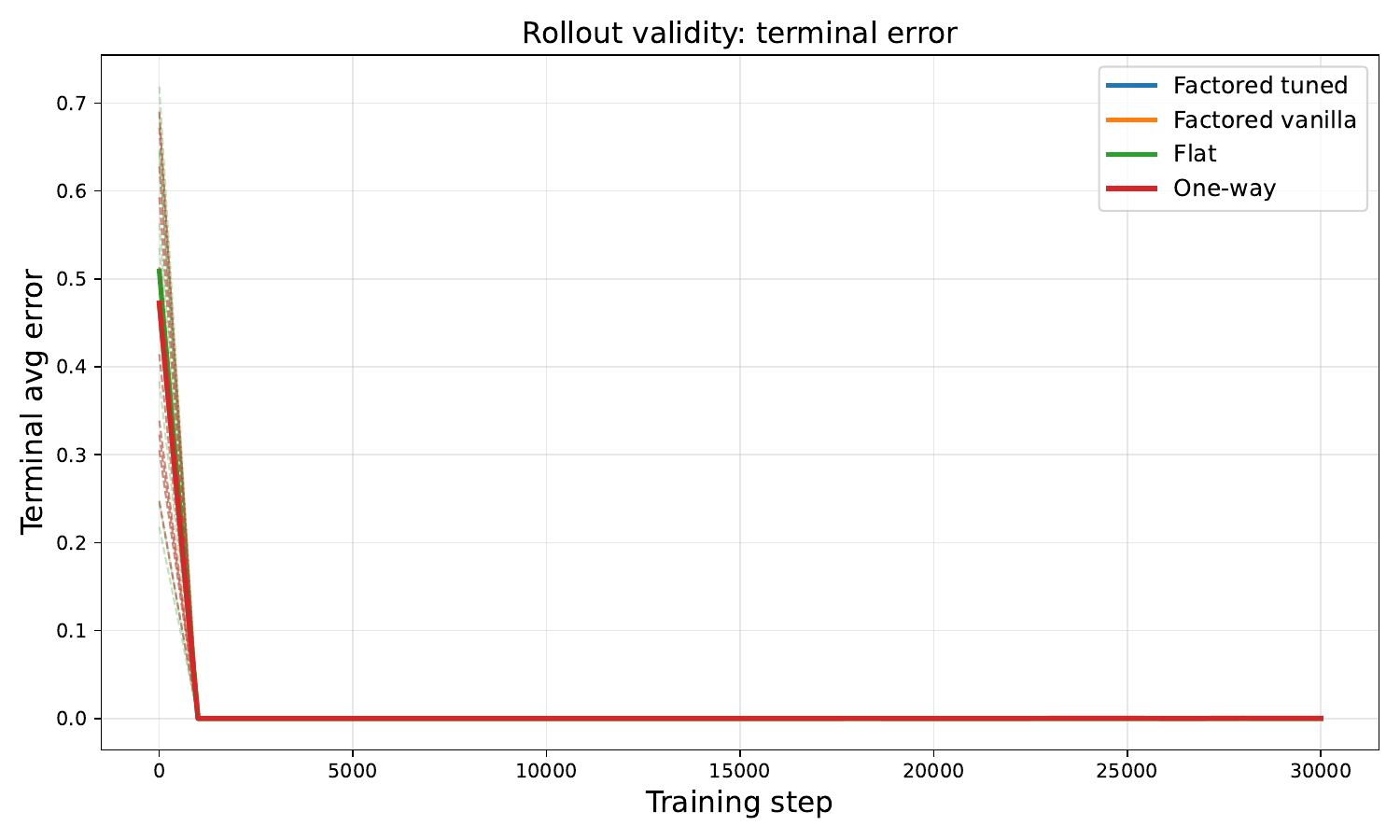}
        \caption{Rollout terminal-flag error in Leduc Poker. All configurations reach zero within roughly $1{,}000$ gradient steps and the curves coincide thereafter.}
        \label{fig:leduc_roll_terminal}
    \end{minipage}\hfill
    \begin{minipage}{0.48\textwidth}
        \centering
        \includegraphics[width=\linewidth]{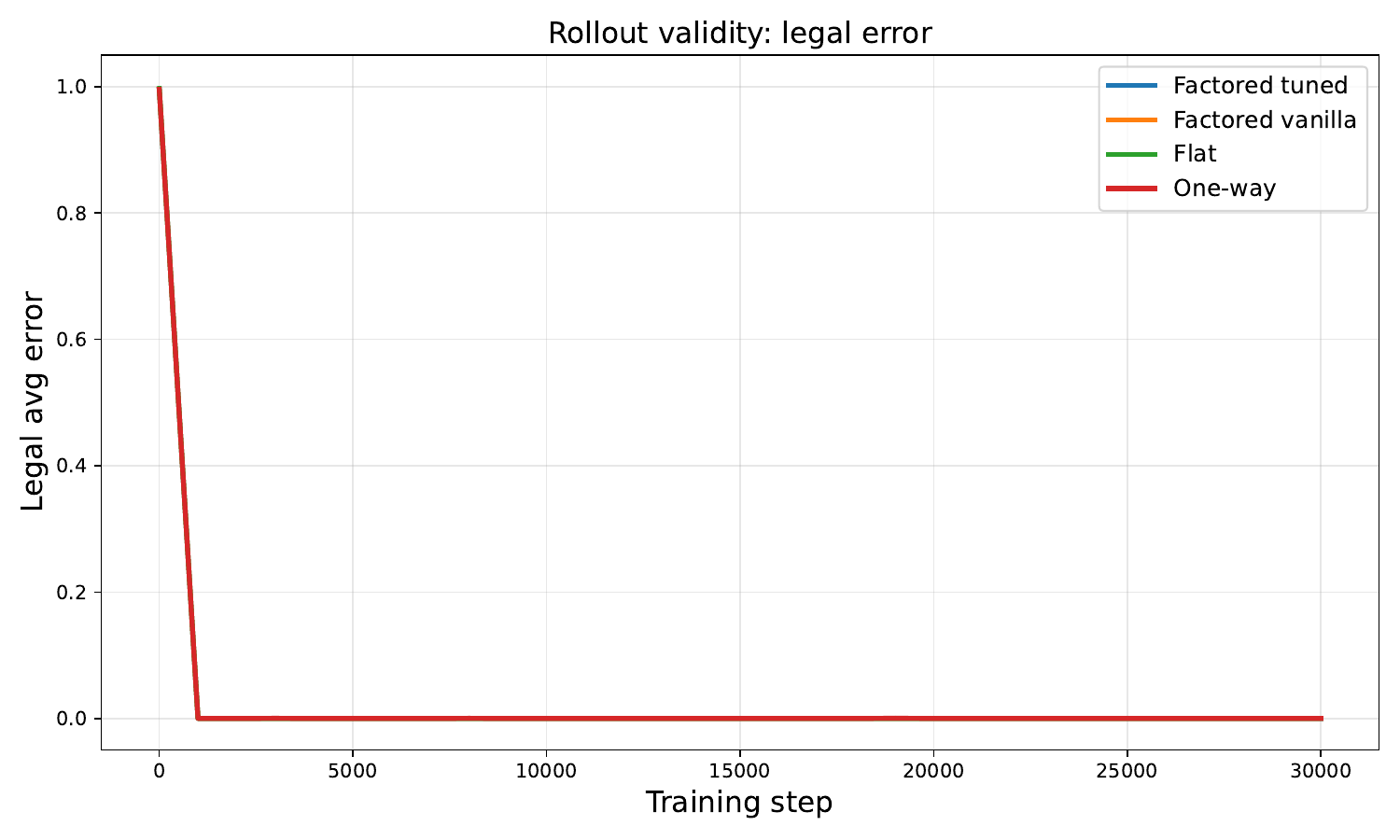}
        \caption{Rollout legal-action-mask error in Leduc Poker, which likewise reaches zero for all configurations and stays there.}
        \label{fig:leduc_roll_legal}
    \end{minipage}
\end{figure}

\subsection{Rollout Validity in Goofspiel 5}
\label{app:rollout_goofspiel}

Goofspiel 5 reveals its point cards in descending order (Appendix~\ref{app:environments}) and therefore contains no chance nodes, so the failure mode of Theorem~\ref{thm:aliasedcounterexample} cannot arise. It serves as the positive control for the Leduc results: it establishes that the architecture can recover the rules of a game exactly. We report rollout validity for NashDreamer[RNaD] over $10$ seeds and $30{,}000$ gradient steps.

Every statistic converges. Terminal flags and legal-action masks reach zero error within $1{,}000$ gradient steps and remain there (Figures~\ref{fig:goof_roll_terminal} and~\ref{fig:goof_roll_legal}). The observation error falls from roughly $4.4$ to $0.25$ over the same $1{,}000$ steps and to near zero by $10{,}000$ (Figure~\ref{fig:goof_roll_obs}). The reward is the hardest to learn (Figure~\ref{fig:goof_roll_reward}): it reaches roughly $0.04$ within $1{,}000$ steps, then descends slowly with individual seeds spiking up to $0.15$, and is stabilizing only in the final third of training.

The reward is also the slowest statistic in Leduc (Appendix~\ref{app:rollout_leduc}), which suggests a property of the reward signal rather than of either game. We attribute it to the reward being terminal: it is zero at every step except the one ending the game, and small in magnitude, so it contributes little to the prediction loss and is resolved last.

\begin{figure}[h]
    \centering
    \begin{minipage}{0.48\textwidth}
        \centering
        \includegraphics[width=\linewidth]{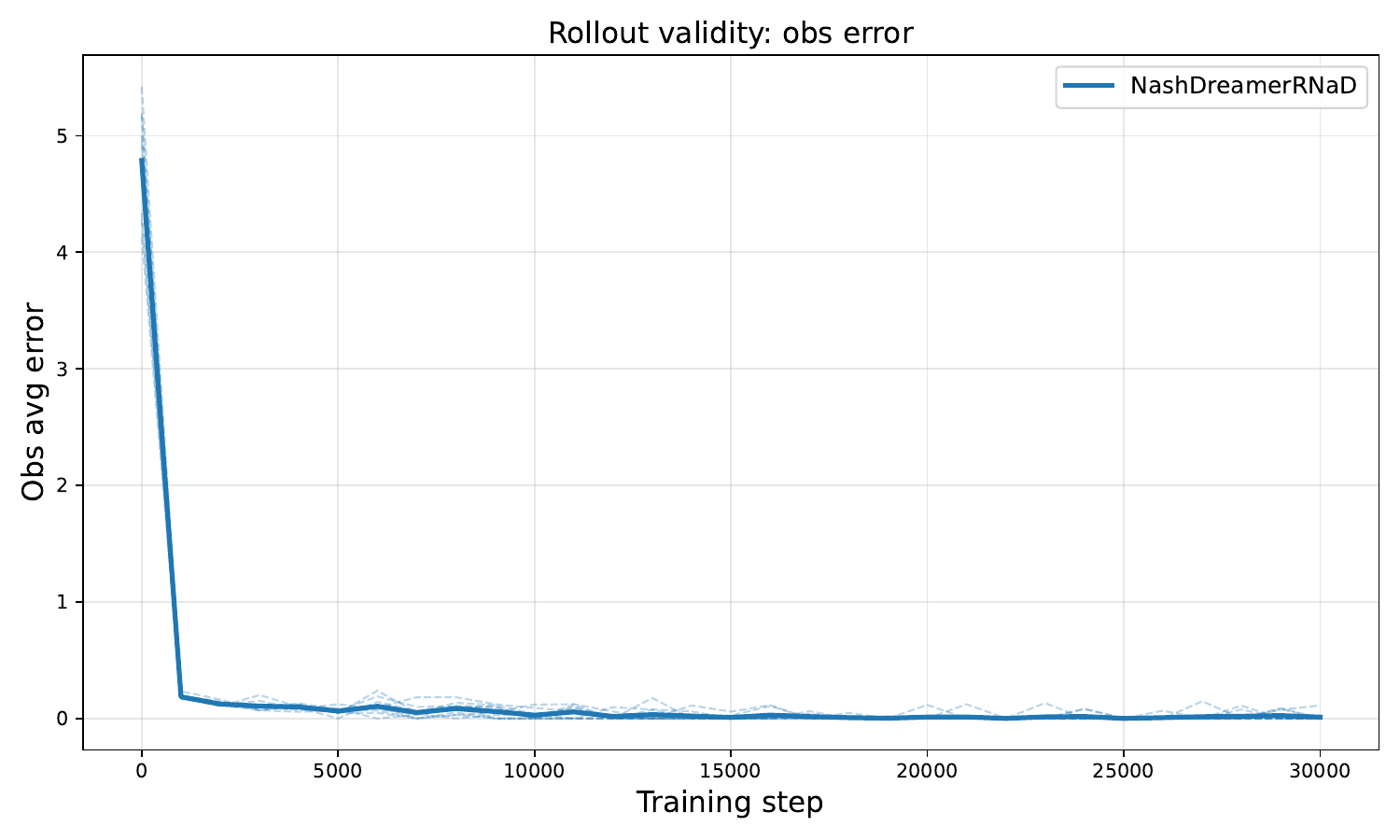}
        \caption{Rollout observation error in Goofspiel 5, as the mean per-step $L^{\infty}$ error of the joint observation. Bold curve is the mean over $10$ seeds, faint curves the individual seeds.}
        \label{fig:goof_roll_obs}
    \end{minipage}\hfill
    \begin{minipage}{0.48\textwidth}
        \centering
        \includegraphics[width=\linewidth]{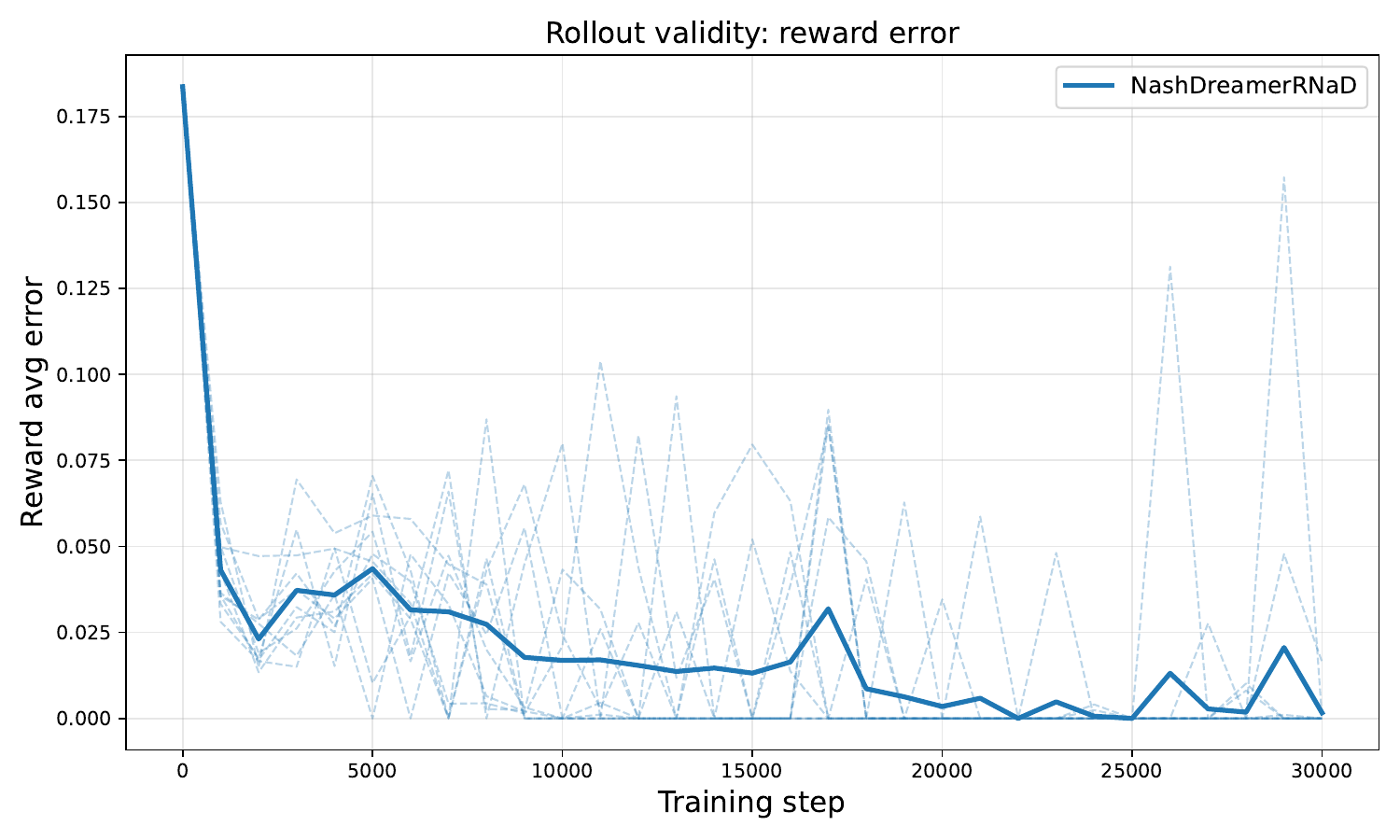}
        \caption{Rollout reward error in Goofspiel 5, the slowest of the four to converge, with individual seeds spiking before falling back.}
        \label{fig:goof_roll_reward}
    \end{minipage}
\end{figure}

\begin{figure}[h]
    \centering
    \begin{minipage}{0.48\textwidth}
        \centering
        \includegraphics[width=\linewidth]{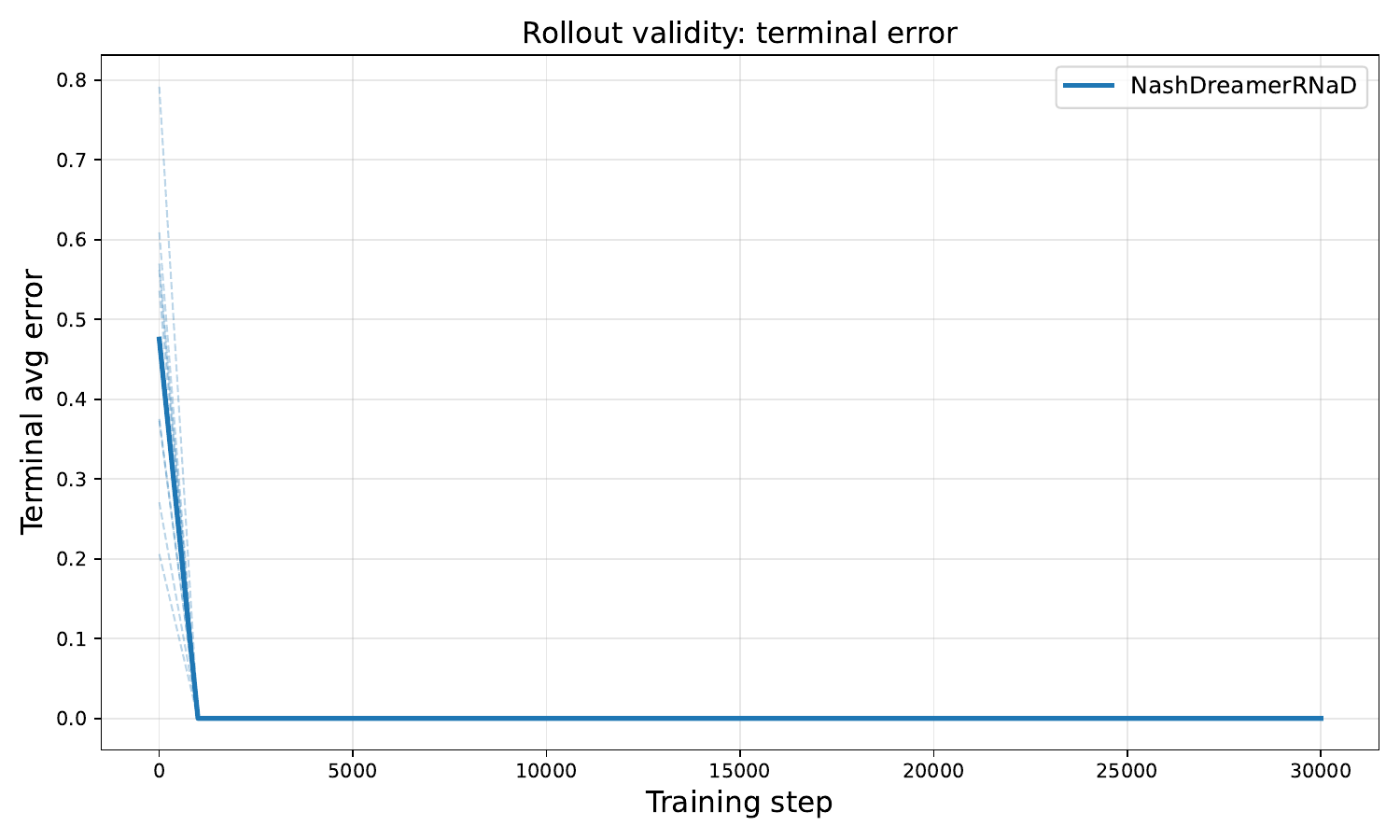}
        \caption{Rollout terminal-flag error in Goofspiel 5, at zero within $1{,}000$ gradient steps.}
        \label{fig:goof_roll_terminal}
    \end{minipage}\hfill
    \begin{minipage}{0.48\textwidth}
        \centering
        \includegraphics[width=\linewidth]{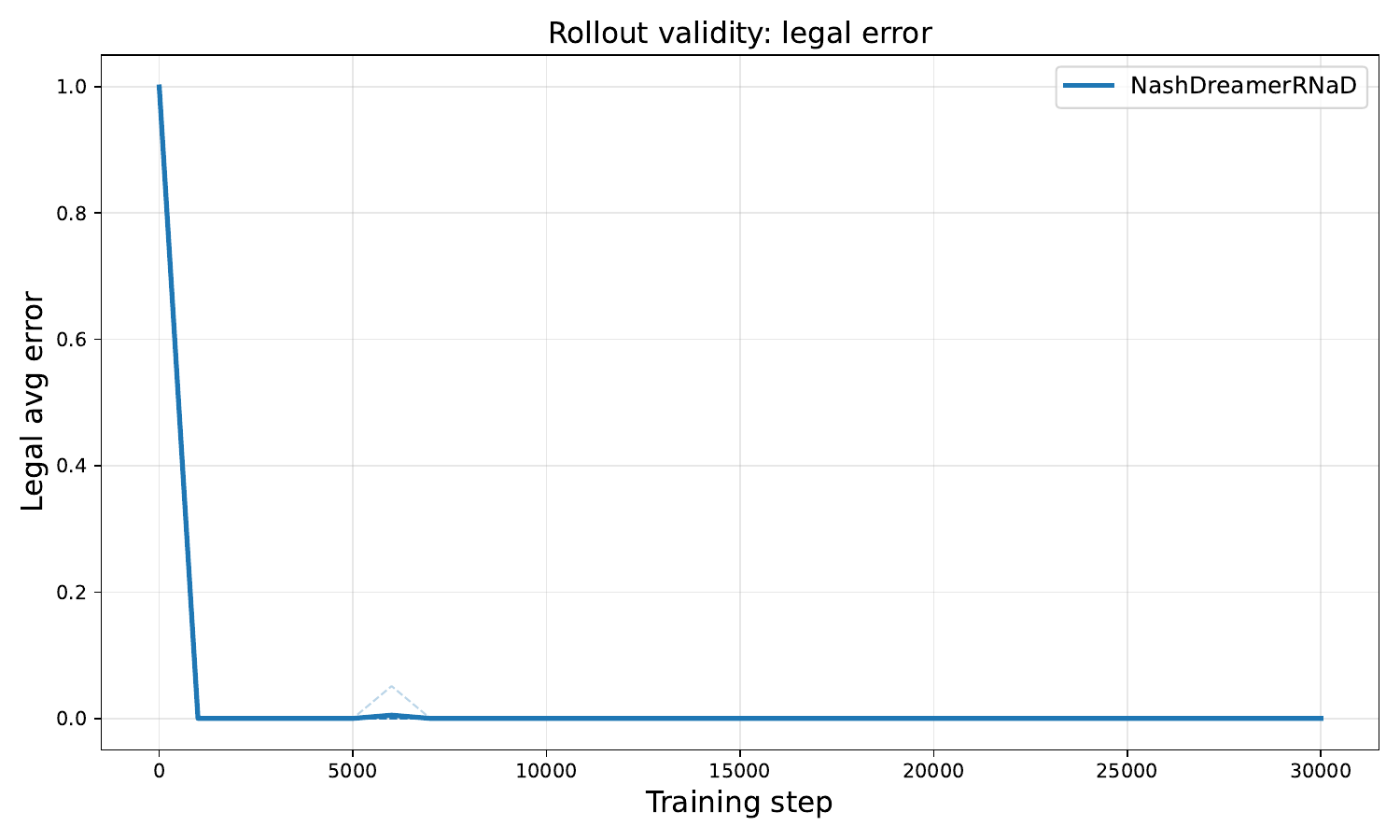}
        \caption{Rollout legal-action-mask error in Goofspiel 5, likewise at zero within $1{,}000$ gradient steps.}
        \label{fig:goof_roll_legal}
    \end{minipage}
\end{figure}

\section{Centralized versus Decentralized World Models}
\label{app:decentralized}

Section~\ref{sec:ctde} argues that a decentralized world model faces identifiability barriers rather than merely an accuracy penalty. This appendix tests that claim directly by ablating the one property under dispute: we train a \emph{decentralized} MARSSM in which each player's world model conditions only on that player's own actions and observations, leaving the architecture, the objective and the optimizer untouched.

\textbf{Setup.} The domain is perturbed Rock-Paper-Scissors (Appendix~\ref{app:environments}), a matrix game with non-uniform equilibrium. Both variants use RNaD as the actor-critic and run for $1{,}000$ gradient steps without switching the regularization policy, over $10$ seeds. At each evaluated checkpoint we collect $4$ batches of $2{,}058$ rollouts and score them by the protocol of Appendix~\ref{app:rollout}.

\textbf{The decentralized model learns everything except the reward.} Figures~\ref{fig:rps_terminal}--\ref{fig:rps_reward} report the four error statistics. On the three quantities that do not depend on the opponent's action, the two variants are essentially the same model: terminal flags and legal-action masks reach zero error within $100$ gradient steps for both, and the joint observation error falls to zero for both, the decentralized variant simply taking longer to get there, around $600$ steps against roughly $200$ for the centralized one.

However, in terms of rewards (the strategically decisive component in this game), the models differ greatly. While the centralized model is able to learn it perfectly, the decentralized model eventually diverges to an average prediction error of 1.0 of the reward. This is due to the reward of this model being unidentifiable and dependent on historic play.

\begin{figure}[h]
    \centering
    \begin{minipage}{0.48\textwidth}
        \centering
        \includegraphics[width=\linewidth]{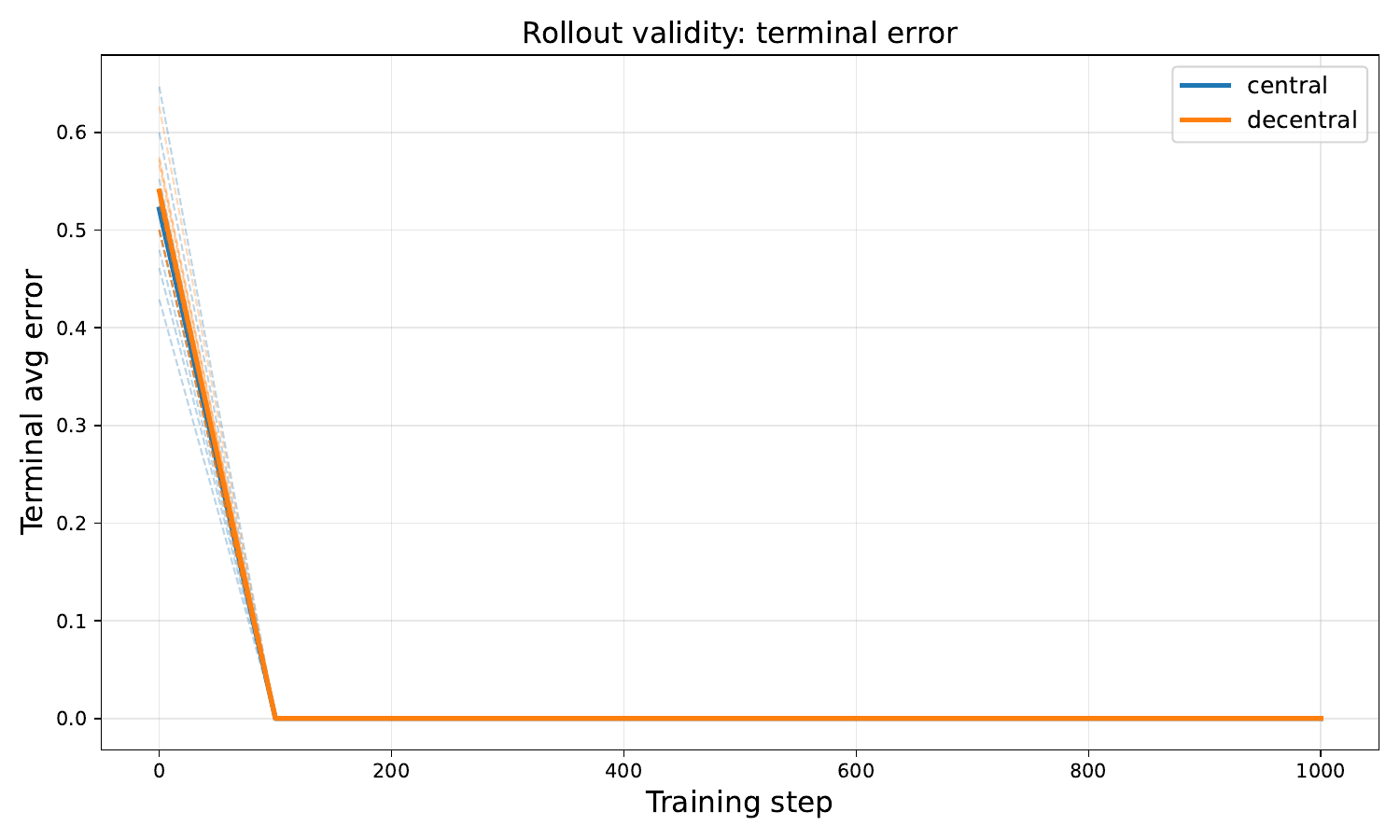}
        \caption{Terminal-flag error over rollouts in perturbed Rock-Paper-Scissors. Bold curves are means over the $10$ seeds, faint curves the individual seeds. The two variants coincide.}
        \label{fig:rps_terminal}
    \end{minipage}\hfill
    \begin{minipage}{0.48\textwidth}
        \centering
        \includegraphics[width=\linewidth]{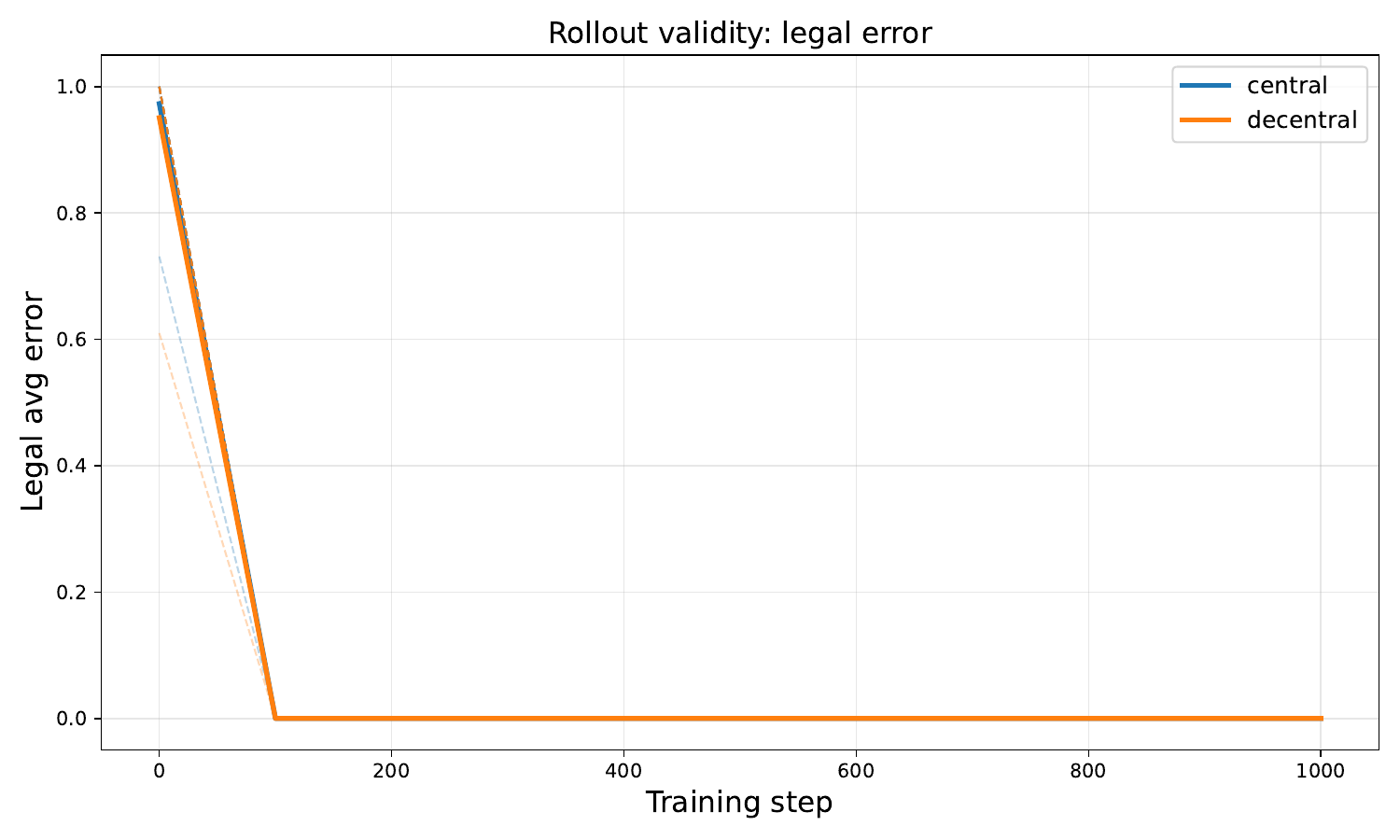}
        \caption{Legal-action-mask error in perturbed Rock-Paper-Scissors, on the same convention. The two variants again coincide, both reaching zero within $100$ gradient steps.}
        \label{fig:rps_legal}
    \end{minipage}
\end{figure}

\begin{figure}[h]
    \centering
    \begin{minipage}{0.48\textwidth}
        \centering
        \includegraphics[width=\linewidth]{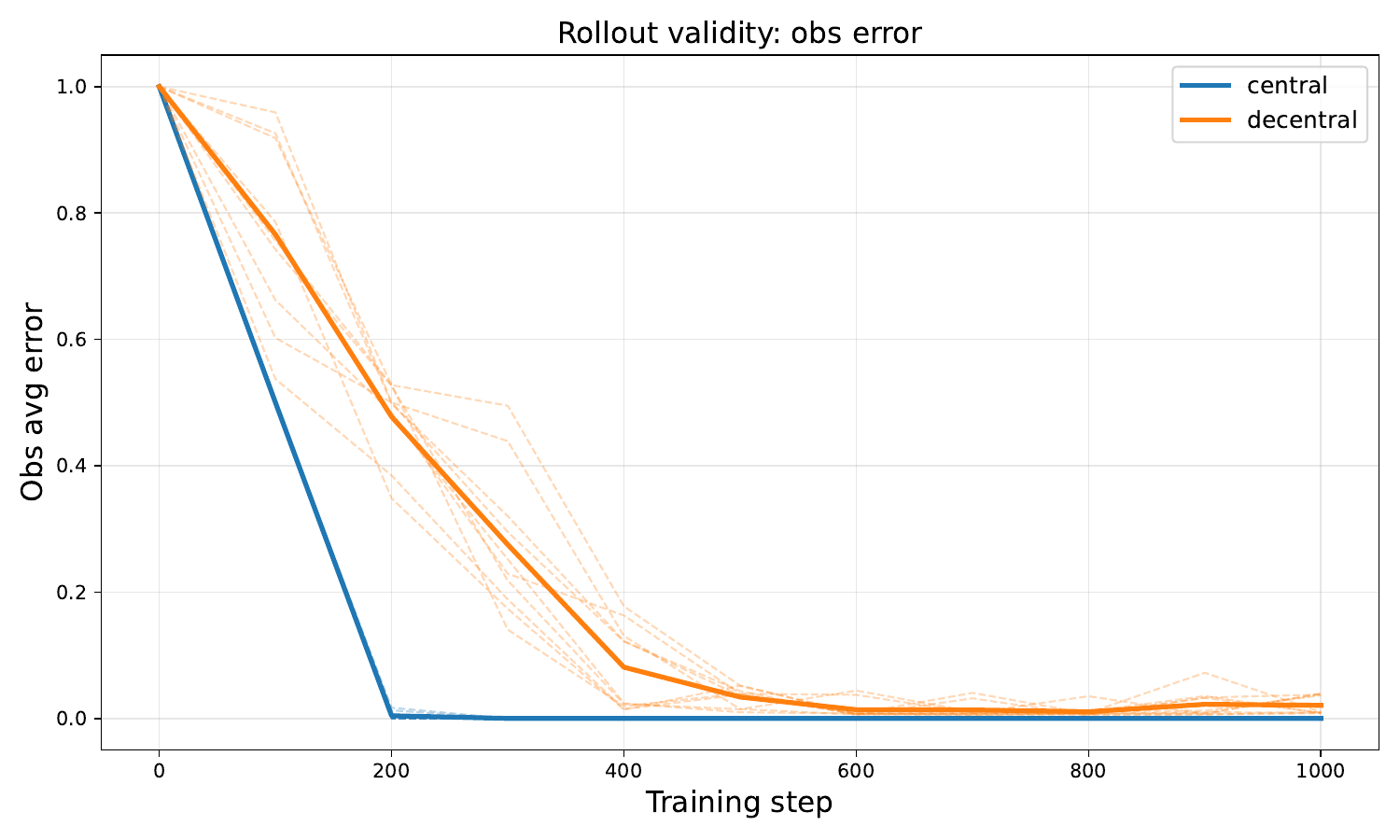}
        \caption{Joint-observation error in perturbed Rock-Paper-Scissors. Both variants converge to zero; the decentralized model is slower but not obstructed.}
        \label{fig:rps_obs}
    \end{minipage}\hfill
    \begin{minipage}{0.48\textwidth}
        \centering
        \includegraphics[width=\linewidth]{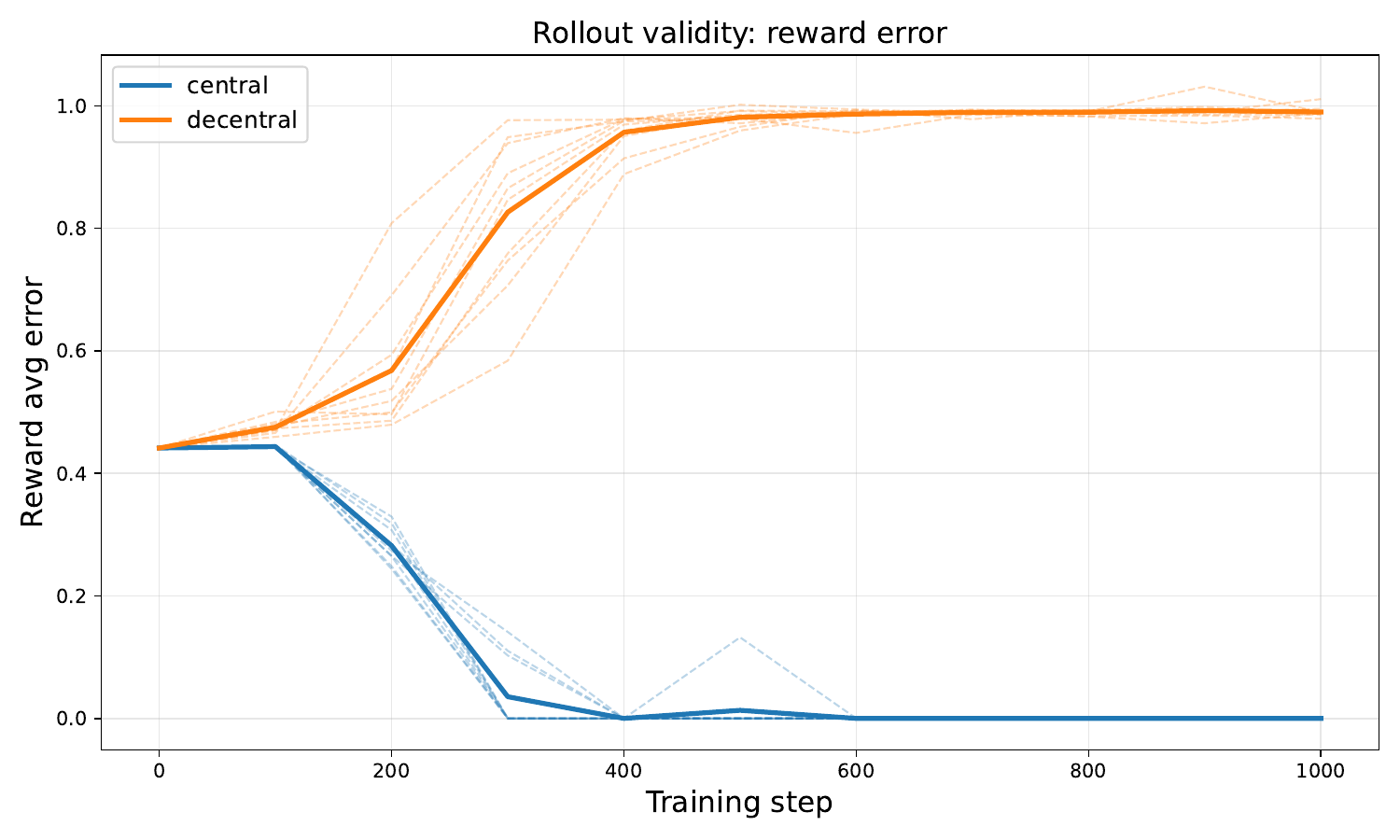}
        \caption{Reward error in perturbed Rock-Paper-Scissors. The centralized model converges to zero, whereas the decentralized model diverges to approximately $1.0$: without observing the joint action,  the reward is unidentifiable from its inputs.}
        \label{fig:rps_reward}
    \end{minipage}
\end{figure}

\section{Single-Agent Validation}
\label{app:pcm_results}

We validated NashDreamer against the reference DreamerV3 implementation on PCM with 3 cards. To ensure a fair comparison, we adapted the reference baseline to accommodate the specific constraints of the environment. Standard DreamerV3 utilizes a burn-in period to approximate the recurrent state from sampled trajectory chunks. However, in a strictly sequential game like PCM, a short burn-in period discards critical historical context, causing structure learning to fail. To resolve this, we configured the reference implementation to sample full trajectories starting from the root node, exactly mirroring the NashDreamer setup. Furthermore, rather than relying on standard reward penalties and episode truncation for invalid moves, we explicitly masked illegal actions during reference sampling.

Under these matched conditions, both models successfully learned the optimal policy within 1000 gradient steps with a minibatch size of 32. However, deeper analysis via \emph{Latent Error Trees} (LE-Trees; see Appendix~\ref{app:letrees} for the visualization methodology and full figures) revealed a NashDreamer structural advantage. 

NashDreamer's LE-Tree was an exact match with the true reachable dynamics (Figure~\ref{fig:letree_pcm_nd}), while DreamerV3 retained artificial stochasticity in its transitions (Figure~\ref{fig:letree_pcm_ref}) (exactly the aforementioned branch duplication). We attribute this structural precision to two factors: (1) the additional legal action prediction signal, and (2) the centralized unifier loss, which penalizes unpredictable variance in the stochastic component. Crucially, as established in our theoretical analysis, without the unifier loss acting as a structural ``tiebreaker,'' the perfectly isomorphic model and a model exhibiting benign branch duplication would incur the exact same dynamics loss.

\begin{figure}[h]
    \centering
    \begin{minipage}{0.48\textwidth}
        \centering
        \includegraphics[width=0.55\linewidth]{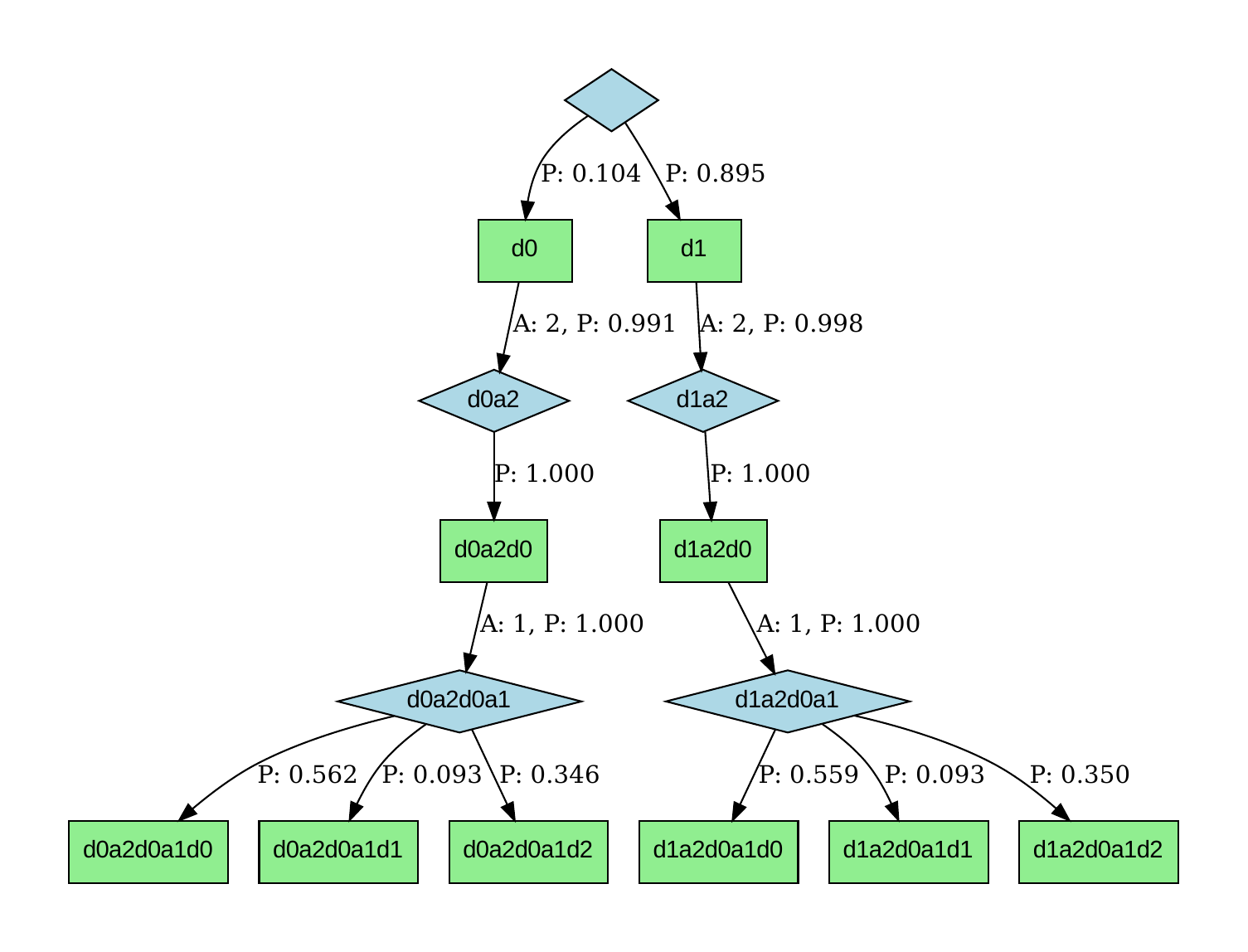}
        \caption{DreamerV3 LE-Tree on PCM~3 (1000 gradient steps). Predictions are accurate along the optimal path, but the model fails to learn that the environment is deterministic.}
        \label{fig:letree_pcm_ref}
    \end{minipage}\hfill
    \begin{minipage}{0.48\textwidth}
        \centering
        \includegraphics[width=0.22\linewidth]{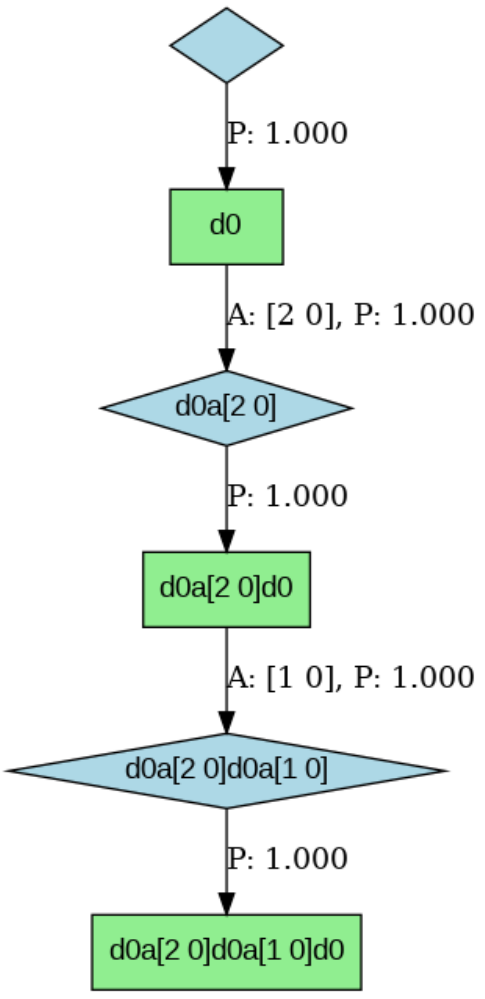}
        \caption{NashDreamer LE-Tree on PCM~3 (1000 gradient steps). The learned model is an exact match of the true reachable dynamics within the 0.2 accuracy threshold.}
        \label{fig:letree_pcm_nd}
    \end{minipage}
\end{figure}

\section{REINFORCE vs. RNaD}
\label{app:reinforce_vs_rnad}
Since REINFORCE has no convergence guarantees in two-player zero-sum IIGs and statically tuning it is not sound in the context of changing world model, we opted not to use it as our policy optimization algorithm. Nevertheless, we test it in Rock-Paper-Scissors (RPS) as a standalone algorithm, and in Goofspiel 3 descending wrapped in NashDreamer to validate our choice.

\subsection{REINFORCE in RPS}
Since every decision node in a POSG (Section~\ref{sec:posg}) is effectively a matrix game, we tested model-free REINFORCE without neural network function approximation in the canonical matrix game Rock-Paper-Scissors (RPS). We evaluated a low entropy bonus of $3 \times 10^{-4}$, characteristic of single-agent environments and a higher entropy bonus of $0.2$, typical of multi-agent environments.

We tested two distinct initializations: a uniform initialization, which initializes the algorithm in the equilibrium of the game, and a pure (Rock, Rock) initialization. The results are presented for 10 distinct seeds, with a 1-$\sigma$ bar interval.

The algorithm always diverged away from equilibrium, even with the higher entropy bonus, which only slowed the divergence (Figure \ref{fig:rpsuniform}). Under the deterministic strategy initialization, low entropy bonus causes the algorithm to remain trapped in that determinism (Figure \ref{fig:rpspure}).

\begin{figure}[t]
    \centering
    \begin{minipage}{0.48\textwidth}
        \centering
        \includegraphics[width=\linewidth]{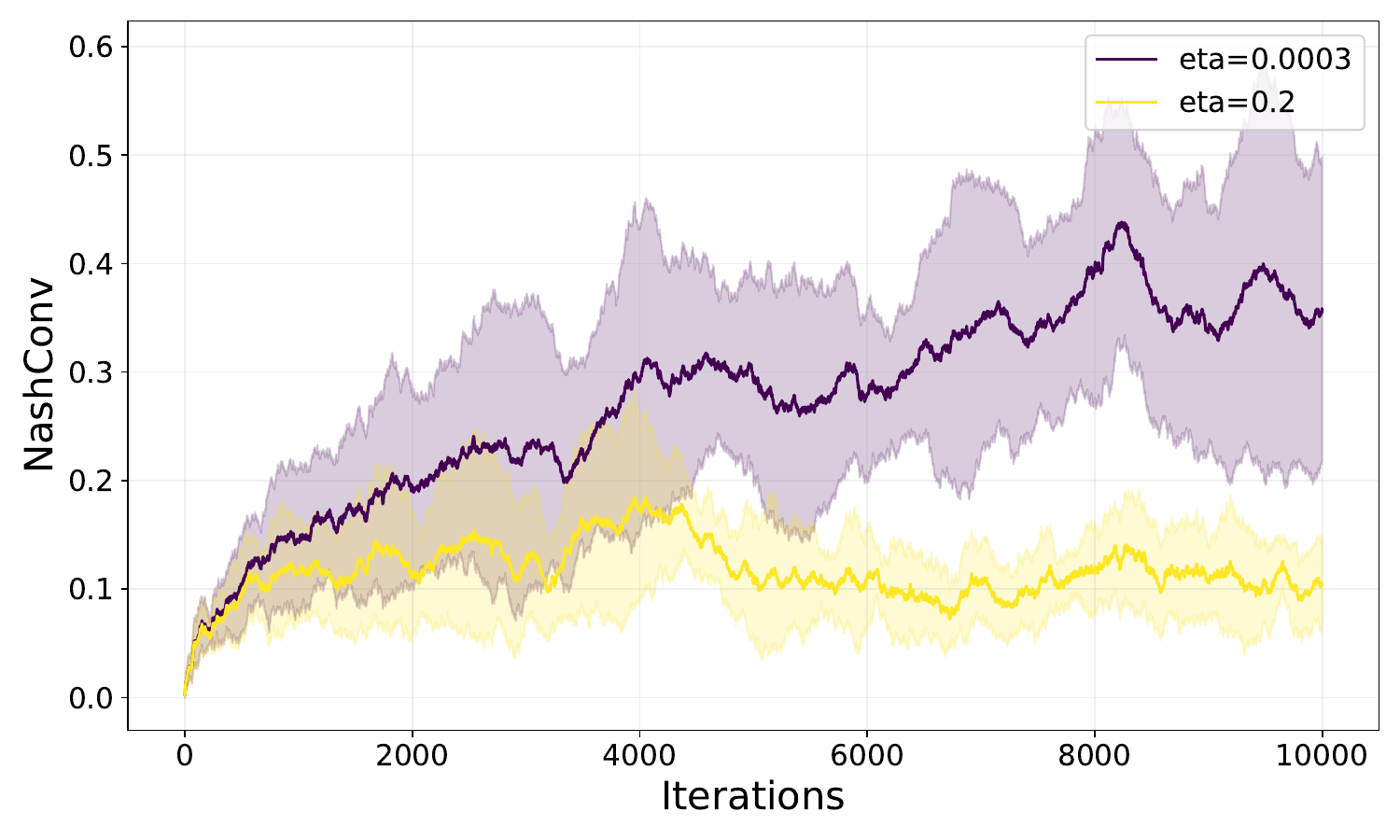}
        \caption{REINFORCE in RPS with uniform initialization. Even at equilibrium, low entropy bonus causes divergence; high bonus slows but does not prevent it.}
        \label{fig:rpsuniform}
    \end{minipage}\hfill
    \begin{minipage}{0.48\textwidth}
        \centering
        \includegraphics[width=\linewidth]{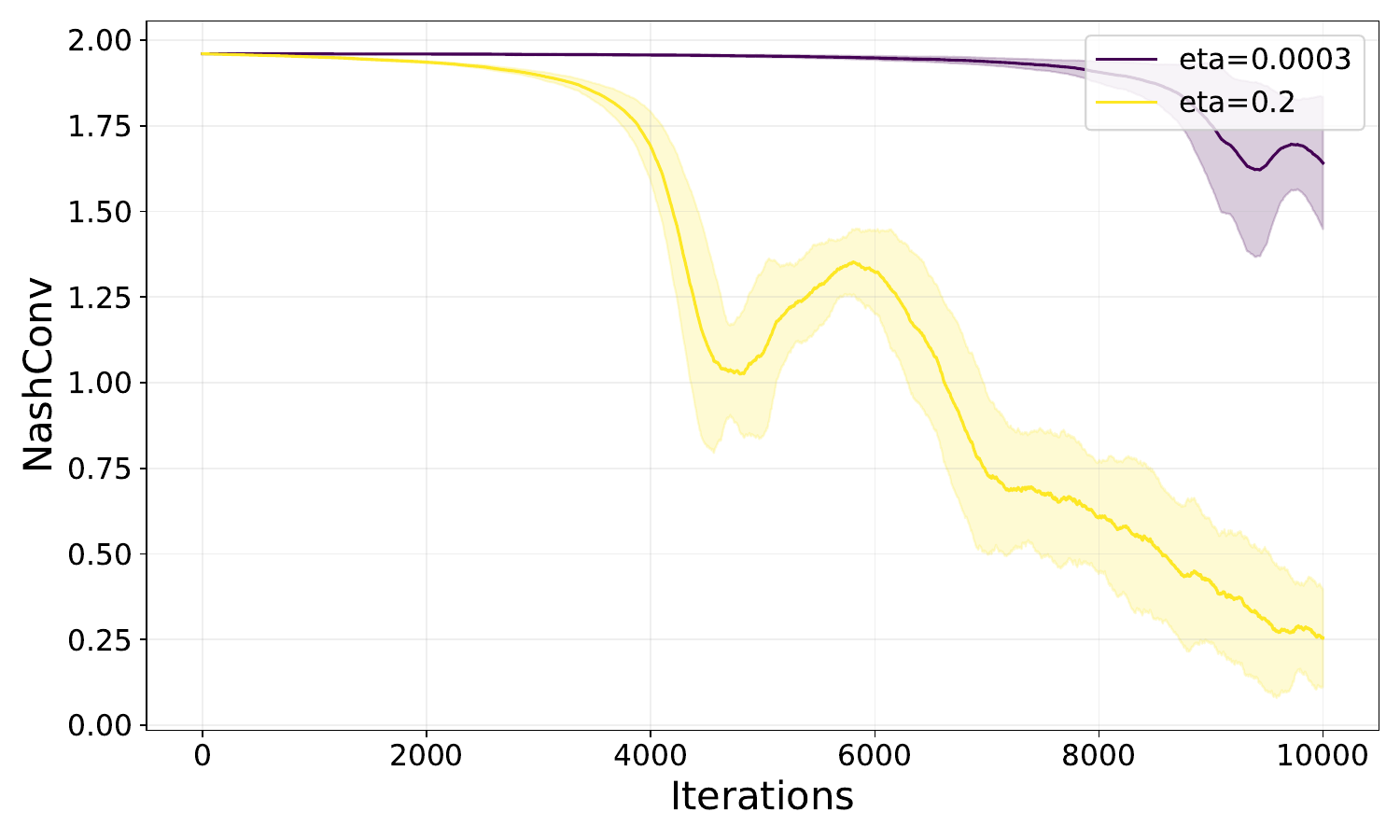}
        \caption{REINFORCE in RPS with pure initialization. Low entropy bonus is trapped in determinism.}
        \label{fig:rpspure}
        
    \end{minipage}
\end{figure}

\begin{figure}
    \begin{minipage}{0.48\textwidth}
        \centering
        \includegraphics[width=\linewidth]{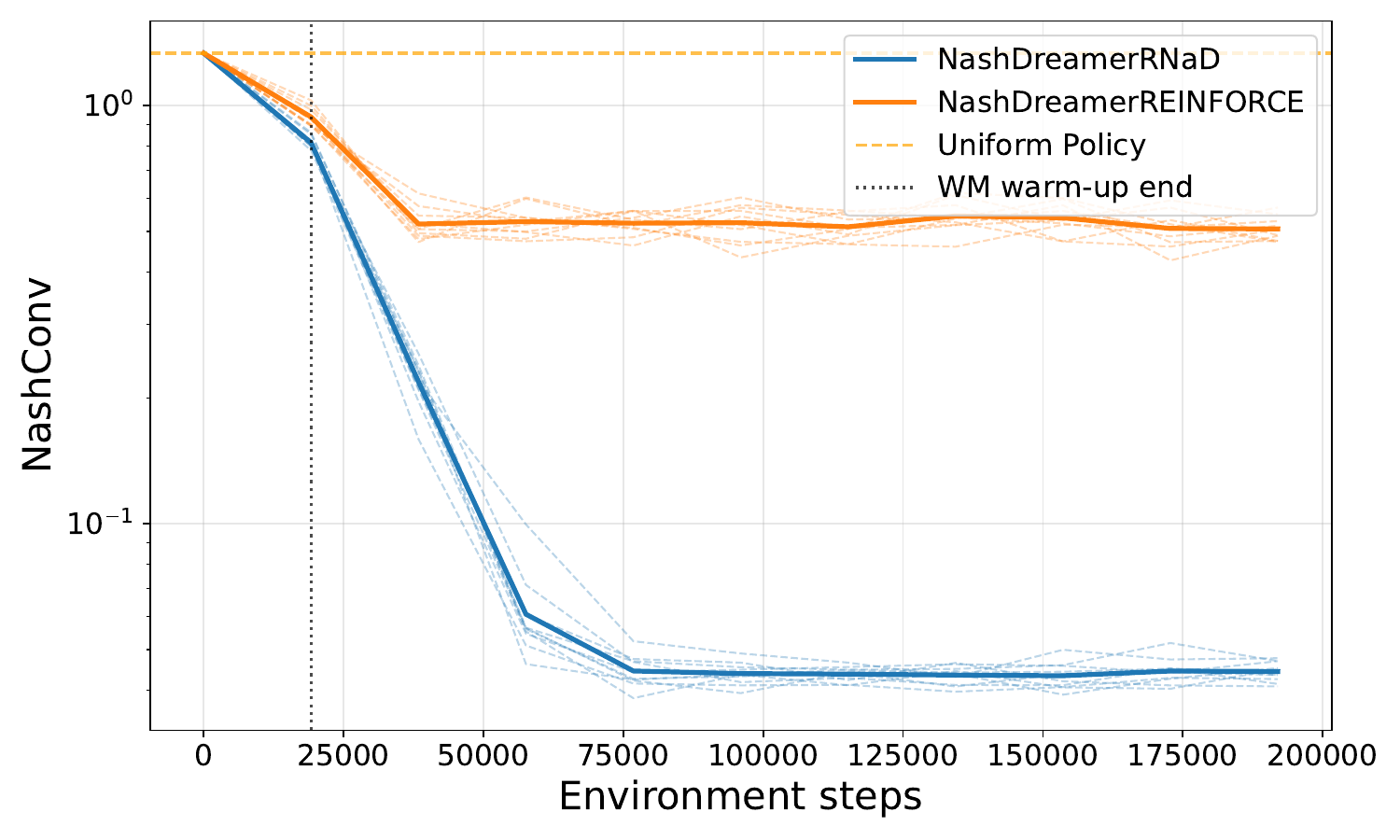}
        \caption{NashConv on Goofspiel 3. RNaD converges to the pure NE; REINFORCE with $\eta=0.2$ entropy bonus cannot.}
        \label{fig:goofspiel3}
    \end{minipage}\hfill
\end{figure}

\subsection{Goofspiel 3}

To evaluate REINFORCE as a component of the full NashDreamer architecture, we perform experiments in Goofspiel 3 (which has a pure equilibrium).

We test a reimplementation of the DreamerV3 version of REINFORCE \citep{dreamerv3}, with identical hyperparameters, except the entropy bonus, which we set to $0.2$ to match the RNaD regularization strength.

RNaD successfully converged to the pure NE, while REINFORCE, burdened by the entropy exploration bonus, was unable to collapse into the required pure strategies (Figure~\ref{fig:goofspiel3}). Analysis of the LE-Trees (Figure~\ref{fig:letree_goof_rnad} for RNaD; REINFORCE is contained in the supplementary material due to size) confirmed that REINFORCE spread probability mass across all actions rather than concentrating on the equilibrium.

\begin{figure}[h]
            \centering
        \includegraphics[width=0.22\linewidth]{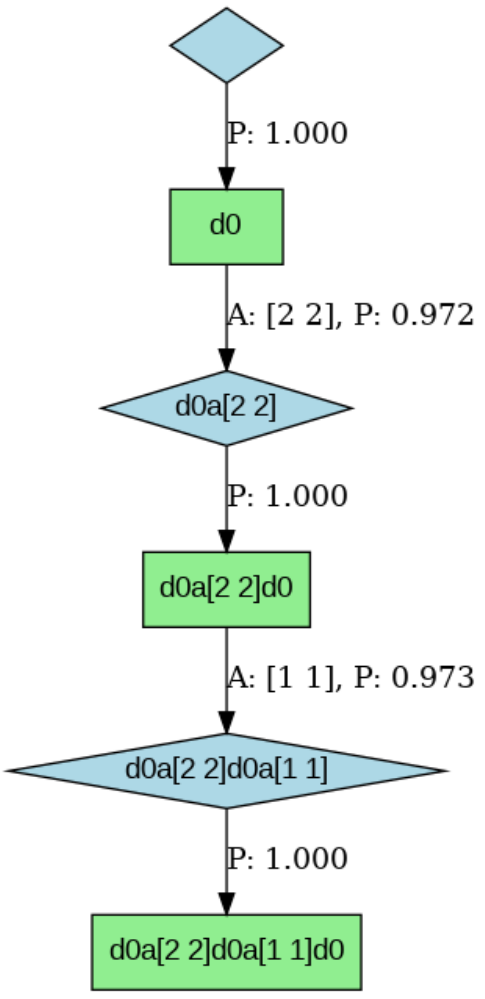}
        \caption{NashDreamer with RNaD LE-Tree on Goofspiel~3 (1000 gradient steps). The model learns both the pure equilibrium and accurate predictions along the equilibrium path.}
        \label{fig:letree_goof_rnad}

\end{figure}

\section{Learning from partial observations}
\label{app:observations}

In our primary experiments (Section \ref{sec:experiments}), we provide both NashDreamer and the standalone RNaD baseline with pre-aggregated infoset representation vectors. This design choice was made to isolate the primary experimental variable, the sample efficiency of the centralized world model versus an off-policy replay buffer, and to prevent differing input representations from confounding the comparison.

However, the MARSSM architecture is designed to handle aggregating sequences of raw, partial observations via the decentralized recurrent infoset model. To empirically validate this capability, we conducted an ablation on the Goofspiel 5 environment.

In this experiment, we tested NashDreamer both with pre-aggregated infoset input and with only the observation in the current state (e.g., the current point card and win/loss/tie information). As demonstrated in Figure \ref{fig:goofspiel_5_obs_vs_iset}, the difference in performance is marginal, validating the capability of NashDreamer to aggregate per-player information and model the environment dynamics simultaneously. 

\begin{figure}[h]
        \centering
        \includegraphics[width=0.5\linewidth]{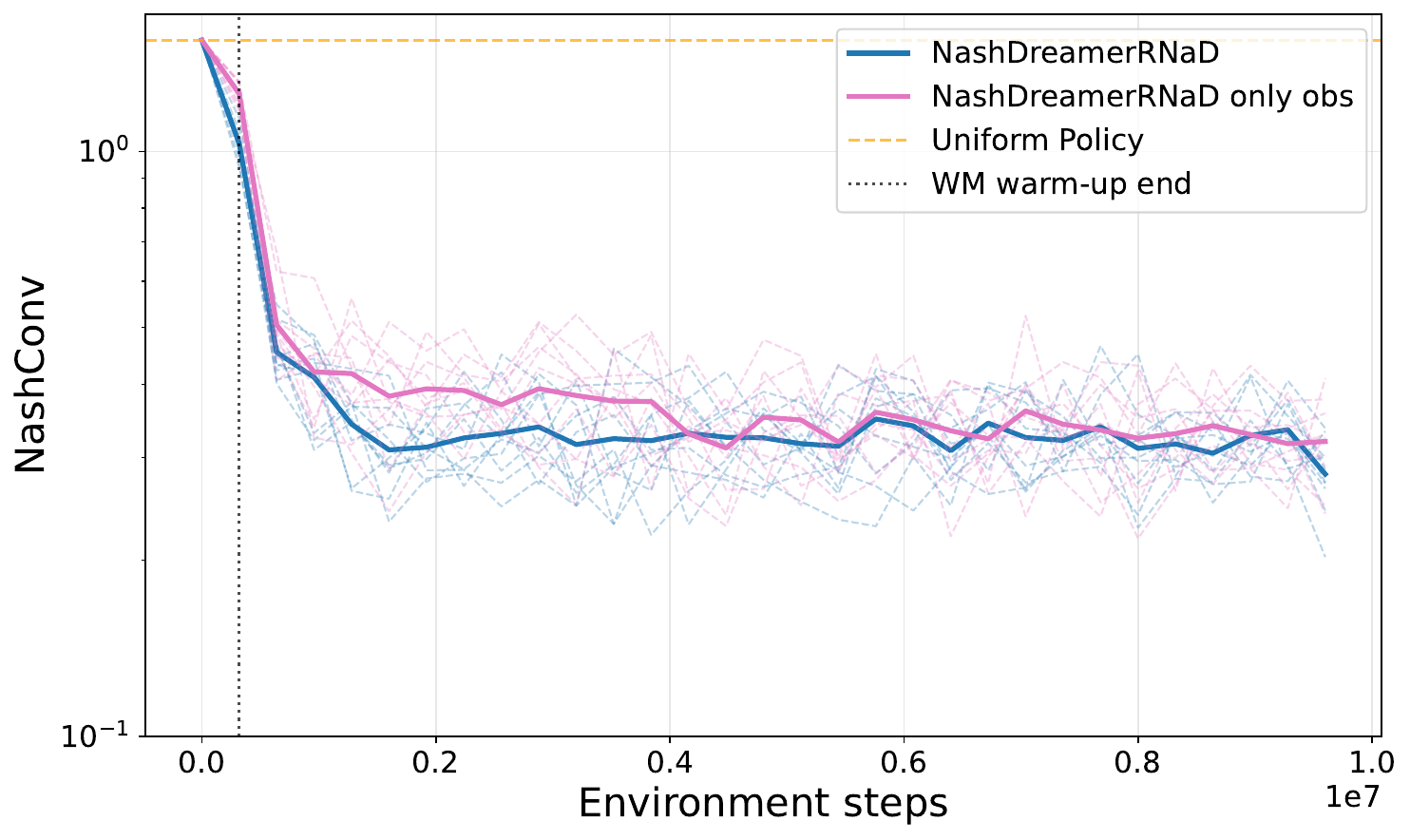}
        \caption{Comparison of NashConv in Goofspiel 5 when using pre-aggregated infoset representation versus latent infoset modelling from partial observations. The model demonstrates the capability to model accurate latent infosets.}
        \label{fig:goofspiel_5_obs_vs_iset}
\end{figure}

\section{NashConv Against Total Actor-Critic Steps}
\label{app:seensteps}

The NashConv curves of Section~\ref{sec:sample_efficiency} are plotted against \emph{real} environment interaction, because that is the quantity our sample-efficiency claim concerns. For completeness, we report here the same runs plotted against the total number of transitions consumed by the actor-critic, counting real, imagined and replayed transitions identically. Figure~\ref{fig:seensteps} gives this view for Goofspiel 5, Leduc Poker and Phantom Tic-Tac-Toe. In every panel, the horizontal axis is logarithmic, evaluated checkpoints are marked on the curves, and the plotted range begins at the end of the world-model warm-up, before which no imagined data exists. The Leduc runs are those of Figure~\ref{fig:nashconv}: the factored latent stochastic state of Appendix~\ref{app:leducerror} under both the vanilla and the tuned hyperparameters ($\beta^{\text{dyn}} = 0.1$, $\beta^{\text{enc}} = 0.01$) for NashDreamer[RNaD], and under the vanilla hyperparameters for NashDreamer[MMD].

Two observations are worth drawing from these plots. First, the model-based runs consume roughly an order of magnitude more total transitions than model-free RNaD to reach a comparable NashConv. This is expected, and it is precisely why the main text reports real interaction: imagined transitions are cheap in the regime we target, so counting them equally with real ones measures a different quantity than the one we claim to improve. Second, the comparison that is meaningful on this axis is against RNaD with replay, which was deliberately configured to consume more total transitions than NashDreamer generates by imagination (Figure~\ref{fig:seensteps}). In Goofspiel 5 and in Phantom Tic-Tac-Toe alike, NashDreamer[RNaD] reaches a lower NashConv than the replay baseline at comparable total step counts, so the improvement is not attributable to step volume alone. In Leduc the NashDreamer[RNaD] curve additionally exhibits the late-training instability reported in Section~\ref{sec:sample_efficiency} and localized in Appendix~\ref{app:diagnostics}.

\textbf{How to read these plots.} Every run is evaluated at $30$ checkpoints spaced $1{,}000$ gradient steps apart, and each evaluated checkpoint is marked on its curve. Because the horizontal axis counts transitions consumed rather than gradient steps, an equal number of gradient steps maps to a different horizontal position for each method, so the curves neither begin nor end together. This is most visible for the replay-augmented baseline, whose first checkpoint lies well to the right of the others because it already consumes replayed transitions during the warm-up period; the model-based runs separate from their model-free counterparts in the same manner once imagination begins. At the first checkpoint, by contrast, NashDreamer[RNaD] and model-free RNaD coincide, as they must: up to the end of the warm-up the two consume exactly the same real transitions, and the same holds for the MMD pair.

\begin{figure}[h]
    \centering
    \includegraphics[width=\linewidth]{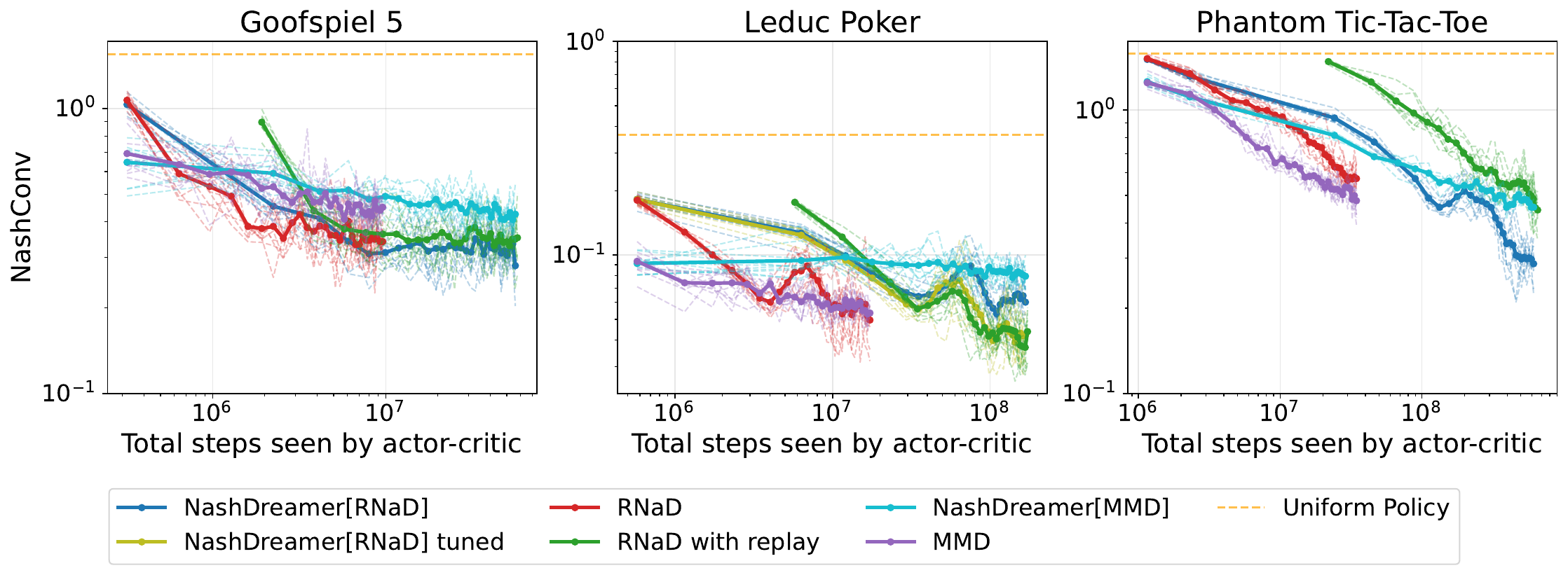}
    \caption{NashConv against the total number of transitions consumed by the actor-critic, counting real, imagined and replayed transitions identically, in Goofspiel 5 (left), Leduc Poker (middle) and Phantom Tic-Tac-Toe (right). Markers indicate evaluated checkpoints. Bold curves are means over the 10 training seeds, faint curves the individual seeds. The Leduc runs use the factored latent stochastic state, under the vanilla hyperparameters unless labelled tuned ($\beta^{\text{dyn}} = 0.1$, $\beta^{\text{enc}} = 0.01$).}
    \label{fig:seensteps}
\end{figure}

\section{Head-to-Head Evaluation}
\label{app:headtohead}

We also report statistics of head-to-head play on larger games, using the same checkpoints as reported in Section~\ref{sec:approx_expl}, which were obtained by playing 1M games per seed pairing, in each seat (so 20M games per ego-opponent pair in general). The results are averaged across the seats.

\textbf{Battleship $5\times5$.} NashDreamer[RNaD] wins against every other method at all three checkpoints (Table~\ref{tab:hth_battleship}), but its margin shrinks monotonically as interactions accumulate: against model-free RNaD it falls from $66.7\%$ at 10k to $62.6\%$ at 20k and $56.7\%$ at 30k, and against RNaD with replay from $71.9\%$ to $66.6\%$ to $61.5\%$. NashDreamer[MMD] shows the same pattern against model-free MMD, $61.5\%$ falling to $56.5\%$. This is the early-to-mid-training scope we claim throughout: the world model buys an advantage that the model-free baseline gradually closes given enough environment interaction. Between the two RNaD baselines, the one with replay is the weaker at every checkpoint, losing to plain RNaD in all three matrices, which is consistent with the variance argument of Appendix~\ref{app:replay}.

\begin{table}[h]
\centering
\small
\setlength{\tabcolsep}{4pt}
\caption{Head-to-head play of 1M games for each seed pairing and seat in Battleship $5\times5$: win rate (\%) of the row method against the column method, mean $\pm 2\sigma$ over training seeds. ND denotes NashDreamer with the given actor-critic inside imagination.}
\label{tab:hth_battleship}
\begin{tabular}{lccccc}
\toprule
 & \multicolumn{5}{c}{Win rate (\%) vs.} \\
\cmidrule(lr){2-6}
Method & MMD & ND[MMD] & ND[RNaD] & RNaD & RNaD w/ replay \\
\midrule
\multicolumn{6}{c}{\textit{10k minibatches}} \\
\midrule
MMD & -- & $38.5 \pm 2.6$ & $25.6 \pm 2.5$ & $41.4 \pm 2.7$ & $48.3 \pm 2.3$ \\
ND[MMD] & $61.5 \pm 2.6$ & -- & $35.7 \pm 2.2$ & $52.7 \pm 2.2$ & $58.8 \pm 1.1$ \\
ND[RNaD] & $74.4 \pm 2.5$ & $64.3 \pm 2.2$ & -- & $66.7 \pm 2.1$ & $71.9 \pm 1.6$ \\
RNaD & $58.6 \pm 2.7$ & $47.3 \pm 2.2$ & $33.3 \pm 2.1$ & -- & $56.9 \pm 1.7$ \\
RNaD w/ replay & $51.7 \pm 2.3$ & $41.2 \pm 1.1$ & $28.1 \pm 1.6$ & $43.1 \pm 1.7$ & -- \\
\midrule
\multicolumn{6}{c}{\textit{20k minibatches}} \\
\midrule
MMD & -- & $40.7 \pm 1.8$ & $29.0 \pm 1.6$ & $42.0 \pm 2.0$ & $45.1 \pm 1.8$ \\
ND[MMD] & $59.3 \pm 1.8$ & -- & $37.3 \pm 1.5$ & $51.0 \pm 2.2$ & $53.4 \pm 1.4$ \\
ND[RNaD] & $71.0 \pm 1.6$ & $62.7 \pm 1.5$ & -- & $62.6 \pm 1.2$ & $66.6 \pm 0.9$ \\
RNaD & $58.0 \pm 2.0$ & $49.0 \pm 2.2$ & $37.4 \pm 1.2$ & -- & $55.0 \pm 1.2$ \\
RNaD w/ replay & $54.9 \pm 1.8$ & $46.6 \pm 1.4$ & $33.4 \pm 0.9$ & $45.0 \pm 1.2$ & -- \\
\midrule
\multicolumn{6}{c}{\textit{30k minibatches}} \\
\midrule
MMD & -- & $43.5 \pm 2.0$ & $32.2 \pm 1.7$ & $39.3 \pm 1.9$ & $43.2 \pm 1.7$ \\
ND[MMD] & $56.5 \pm 2.0$ & -- & $38.5 \pm 1.8$ & $46.3 \pm 2.5$ & $49.9 \pm 1.9$ \\
ND[RNaD] & $67.8 \pm 1.7$ & $61.5 \pm 1.8$ & -- & $56.7 \pm 2.3$ & $61.5 \pm 0.8$ \\
RNaD & $60.7 \pm 1.9$ & $53.7 \pm 2.5$ & $43.3 \pm 2.3$ & -- & $54.9 \pm 1.9$ \\
RNaD w/ replay & $56.8 \pm 1.7$ & $50.1 \pm 1.9$ & $38.5 \pm 0.8$ & $45.1 \pm 1.9$ & -- \\
\bottomrule
\end{tabular}
\end{table}

\textbf{Goofspiel 13.} Here head-to-head play does not separate a NashDreamer variant from its model-free counterpart at all (Table~\ref{tab:hth_goofspiel}). NashDreamer[RNaD] and RNaD are at parity across all three checkpoints, the model-based side winning $50.3\%$, $49.9\%$ and $49.7\%$ of matches, and NashDreamer[MMD] and MMD likewise ($49.0\%$, $48.6\%$, $48.9\%$). RNaD consistently beats MMD across the tested checkpoints, and RNaD with replay is slightly weaker than without it.

\begin{table}[h]
\centering
\small
\setlength{\tabcolsep}{4pt}
\caption{Head-to-head play of 1M games for each seed pairing and seat in Goofspiel 13: win rate (\%) of the row method against the column method, mean $\pm 2\sigma$ over training seeds. ND denotes NashDreamer with the given actor-critic inside imagination.}
\label{tab:hth_goofspiel}
\begin{tabular}{lccccc}
\toprule
 & \multicolumn{5}{c}{Win rate (\%) vs.} \\
\cmidrule(lr){2-6}
Method & MMD & ND[MMD] & ND[RNaD] & RNaD & RNaD w/ replay \\
\midrule
\multicolumn{6}{c}{\textit{10k minibatches}} \\
\midrule
MMD & -- & $51.0 \pm 1.8$ & $37.4 \pm 1.8$ & $36.6 \pm 1.9$ & $42.5 \pm 1.6$ \\
ND[MMD] & $49.0 \pm 1.8$ & -- & $36.5 \pm 2.0$ & $35.2 \pm 1.5$ & $41.3 \pm 1.7$ \\
ND[RNaD] & $62.6 \pm 1.8$ & $63.5 \pm 2.0$ & -- & $50.3 \pm 1.6$ & $56.4 \pm 1.0$ \\
RNaD & $63.4 \pm 1.9$ & $64.8 \pm 1.5$ & $49.7 \pm 1.6$ & -- & $56.4 \pm 2.2$ \\
RNaD w/ replay & $57.5 \pm 1.6$ & $58.7 \pm 1.7$ & $43.6 \pm 1.0$ & $43.6 \pm 2.2$ & -- \\
\midrule
\multicolumn{6}{c}{\textit{20k minibatches}} \\
\midrule
MMD & -- & $51.4 \pm 3.9$ & $38.4 \pm 1.7$ & $37.2 \pm 2.4$ & $41.6 \pm 1.6$ \\
ND[MMD] & $48.6 \pm 3.9$ & -- & $37.4 \pm 2.5$ & $36.0 \pm 1.9$ & $40.4 \pm 2.4$ \\
ND[RNaD] & $61.6 \pm 1.7$ & $62.6 \pm 2.5$ & -- & $49.9 \pm 1.4$ & $54.9 \pm 1.3$ \\
RNaD & $62.8 \pm 2.4$ & $64.0 \pm 1.9$ & $50.1 \pm 1.4$ & -- & $55.7 \pm 2.4$ \\
RNaD w/ replay & $58.4 \pm 1.6$ & $59.6 \pm 2.4$ & $45.1 \pm 1.3$ & $44.3 \pm 2.4$ & -- \\
\midrule
\multicolumn{6}{c}{\textit{30k minibatches}} \\
\midrule
MMD & -- & $51.1 \pm 1.4$ & $37.9 \pm 2.3$ & $37.1 \pm 1.6$ & $41.5 \pm 1.5$ \\
ND[MMD] & $48.9 \pm 1.4$ & -- & $37.1 \pm 1.8$ & $36.1 \pm 1.9$ & $40.6 \pm 1.7$ \\
ND[RNaD] & $62.1 \pm 2.3$ & $62.9 \pm 1.8$ & -- & $49.7 \pm 1.3$ & $54.8 \pm 1.4$ \\
RNaD & $62.9 \pm 1.6$ & $63.9 \pm 1.9$ & $50.3 \pm 1.3$ & -- & $55.8 \pm 1.4$ \\
RNaD w/ replay & $58.5 \pm 1.5$ & $59.4 \pm 1.7$ & $45.2 \pm 1.4$ & $44.2 \pm 1.4$ & -- \\
\bottomrule
\end{tabular}
\end{table}

In Goofspiel 13, head-to-head play attributes no benefit at all to the world model: NashDreamer[RNaD] and RNaD are indistinguishable in direct play, and each fares the same against the rest of the pool. Under best-response search the two are far apart, $0.168$ against $0.356$ at 30k.

In Battleship the disagreement runs the other way. NashDreamer[RNaD] wins the cross-play table outright, yet at 30k it is the second most exploitable of the five methods ($1.169$), whereas MMD, which loses to every other method in direct play, is harder to exploit ($1.022$), and NashDreamer[MMD] is the hardest of all ($0.817$). 

Neither observation contradicts the results of Section~\ref{sec:approx_expl}, since head-to-head measures performance across a particular opponent pool, rather than how hard the policy is to exploit overall. We therefore report exact NashConv wherever it is tractable (Section~\ref{sec:sample_efficiency}), budgeted approximate exploitability where it is not (Section~\ref{sec:approx_expl}), and treat the tables above as evidence about pool performance alone.

\clearpage
\section{World-Model Warm-Up Ablation}
\label{app:warmup}

To check the sensitivity of our model towards the warm-up parameter, we sweep its length over $\{0, 250, 500, 1000, 2000\}$ world-model gradient steps in Leduc Poker, $1000$ being the value used in every other Leduc experiment, running both actor-critics inside imagination at each setting with model-free RNaD and MMD for reference (Figure~\ref{fig:warmup}). For this experiment, both NashDreamer variants use the tuned hyperparameters ($\beta^{\text{enc}} = 0.01$, $\beta^{\text{dyn}} = 0.1$). Runs are shown for $6000$ minibatches as this parameter has the greatest impact early in the training.

Figure~\ref{fig:warmup} reports the mean NashConv across 10 seeds with $95\%$ confidence intervals. The lower values of world model warm-up are less stable for NashDreamer[MMD], where warm-up values of $0$ and $250$ spike in NashConv, while NashDreamer[RNaD] is largely unaffected, with only a slight early spike at $500$. After the first 1000 minibatches, the differences are marginal and do not separate with further training, confirming that the parameter stabilizes early training, but the model is ultimately robust towards the choice.

\begin{figure}[h]
    \centering
    \includegraphics[width=0.9\linewidth]{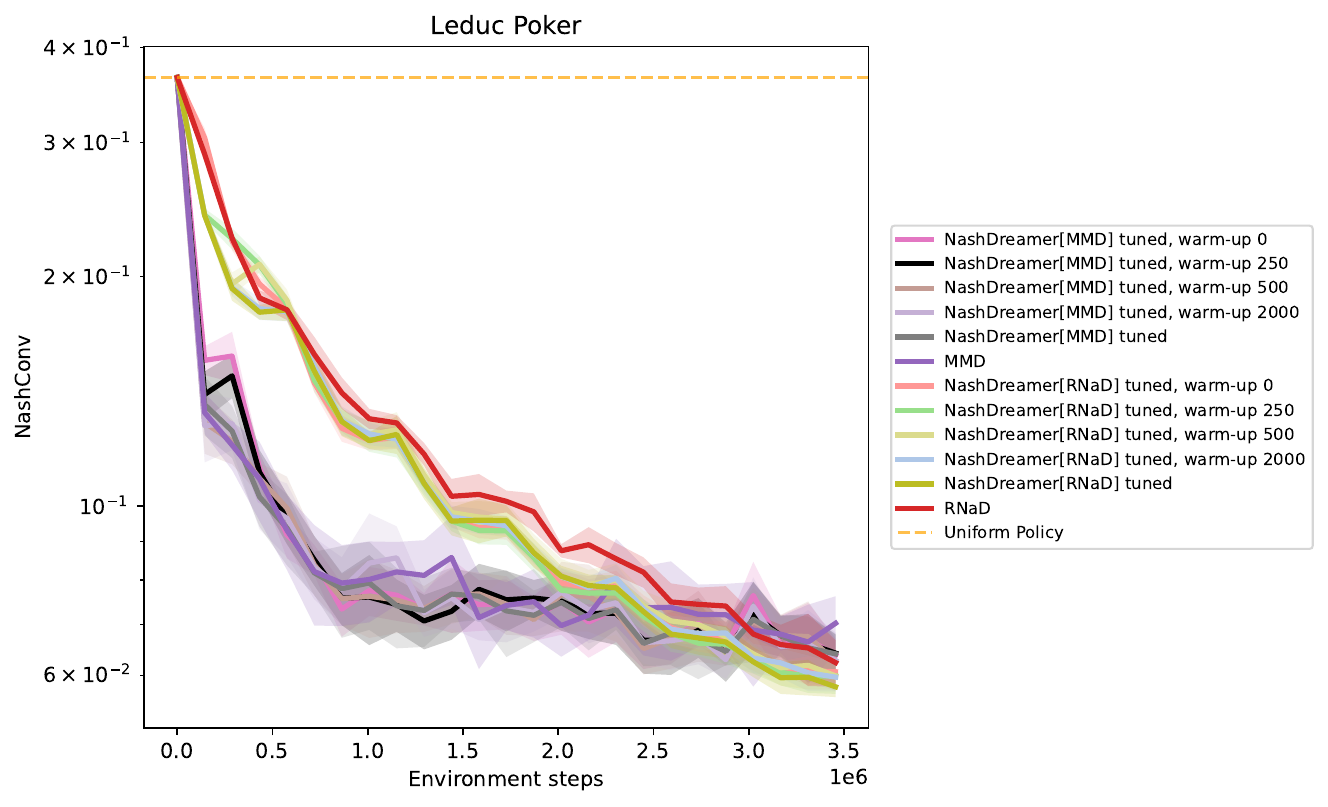}
    \caption{World-model warm-up sweep in Leduc Poker over $\{0, 250, 500, 1000, 2000\}$ gradient steps, applied to both actor-critics inside imagination with the tuned hyperparameters, with model-free RNaD and MMD for reference. Legend entries are labelled by the warm-up length in gradient steps; entries without a label use the default of $1000$. Bold curves are means over the $10$ training seeds, shaded bands $95\%$ confidence intervals. NashDreamer[RNaD] is largely unaffected by the setting, whereas NashDreamer[MMD] separates from model-free MMD in the early phase only at the two shortest warm-ups.}
    \label{fig:warmup}
\end{figure}